\documentclass[sigconf,nonacm]{aamas} 

\makeatletter
\gdef\@copyrightpermission{
  \begin{minipage}{0.2\columnwidth}
   \href{https://creativecommons.org/licenses/by/4.0/}{\includegraphics[width=0.90\textwidth]{by}}
  \end{minipage}\hfill
  \begin{minipage}{0.8\columnwidth}
   \href{https://creativecommons.org/licenses/by/4.0/}{This work is licensed under a Creative Commons Attribution International 4.0 License.}
  \end{minipage}
  \vspace{5pt}
}
\makeatother

\setcopyright{none}
\renewcommand\footnotetextcopyrightpermission[1]{}
\makeatletter 
\def\acmConference@shortname#1\acmConference@venue{}
\makeatother

\usepackage{balance} 
\usepackage{multirow}
\usepackage{tabularx}
\usepackage{nicefrac}
\usepackage{colortbl}
\usepackage{afterpage}
\usepackage{placeins}
\usepackage{needspace}
\usepackage{listings}
\usepackage[most]{tcolorbox}
\usepackage{algorithm}
\usepackage{algpseudocode}
\usepackage{subcaption}
\usepackage{enumitem}

\definecolor{rccband}{RGB}{219,234,254}

\lstdefinestyle{pythonstyle}{
  language=Python,
  basicstyle=\ttfamily\small,
  breaklines=true,
  breakatwhitespace=false,
  columns=fullflexible,
  keepspaces=true,
  frame=single,
  showstringspaces=false,
  tabsize=2
}

\newtcolorbox{promptbox}[2]{%
  enhanced, breakable,
  colback=#1!4!white,
  colframe=#1!55!black,
  colbacktitle=#1!55!black,
  coltitle=white,
  fonttitle=\sffamily\bfseries\small,
  title={#2},
  boxrule=0.5pt,
  arc=1.5mm,
  left=3mm, right=3mm, top=2.5mm, bottom=2.5mm,
  fontupper=\small,
  parbox=false
}

\newcommand{\promptkey}[1]{{\sffamily\footnotesize\bfseries\color{black!55}#1}\par\nobreak\addvspace{\smallskipamount}}

\newcommand{\promptboxhead}{\par\addvspace{\smallskipamount}\noindent}
\newcommand{\plabel}[1]{\textbf{#1}}
\AtBeginDocument{\let\origmaketitle\maketitle\renewcommand\maketitle{\origmaketitle\hypersetup{pdfcreator={LaTeX}}}}

\title{Evaluating Rational Contracting in Natural Language}

\author{Bhavyesh Sajja}
\affiliation{
  \institution{National University of Singapore}
  \city{}
  \country{}
}
\email{bsajja@u.nus.edu}

\author{Max Kleiman-Weiner}
\affiliation{
  \institution{University of Washington}
  \city{}
  \country{}
}
\email{maxkw@uw.edu}

\author{Roger Zimmermann}
\affiliation{
  \institution{National University of Singapore}
  \city{}
  \country{}
}
\email{dcsrz@nus.edu.sg}

\author{Tan Zhi-Xuan}
\affiliation{
  \institution{National University of Singapore \& A*STAR}
  \city{}
  \country{}
}
\email{xuan.cs@nus.edu.sg}

\begin{abstract}

The emergence of language-based AI agents promises to transform the scope of machine economic activity. Instead of just proposing bids or following hard-coded protocols, such agents can be used to negotiate and execute agreements in natural language. However, most evaluations of these abilities have focused on one-off exchanges or simple economic games, leaving open the rich space of time-extended, contingent, and incomplete contracts expressible in language; they also focus on raw profit, without measuring the qualities required for trustworthy contracting. We address this by formulating a rational framework for how agents should negotiate and perform natural language contracts in uncertain multi-step environments. Using this framework, we develop metrics and baselines for quantifying rational and cooperative play. We instantiate our framework in \textsc{ContractSim}, an evaluation suite where two players negotiate and execute a multi-turn supplier contract under environmental and inter-player uncertainty. We find that LLM-based agents agree on efficient contracts when environmental uncertainty is low, but under high uncertainty, often fail to negotiate satisfiable, efficient, or mutually beneficial contracts, underperforming even heuristic baselines. When executing contracts, LLM agents are also frequently uncooperative, violating contracts for extra profit even when they are easy to satisfy. These findings highlight room for improvement in the design of language agents that can negotiate and execute contracts both rationally and cooperatively.

\end{abstract}

\keywords{Contracts, Negotiation, Cooperative AI, LLMs, Agentic Economies}

\begin{document}
\addtocontents{toc}{\protect\setcounter{tocdepth}{-10}} 
\captionsetup{textfont=normalfont} 

\maketitle

\begin{figure*}[t]
        \centering
        \includegraphics[width=.95\textwidth,trim=0 33.5 0 3,clip]{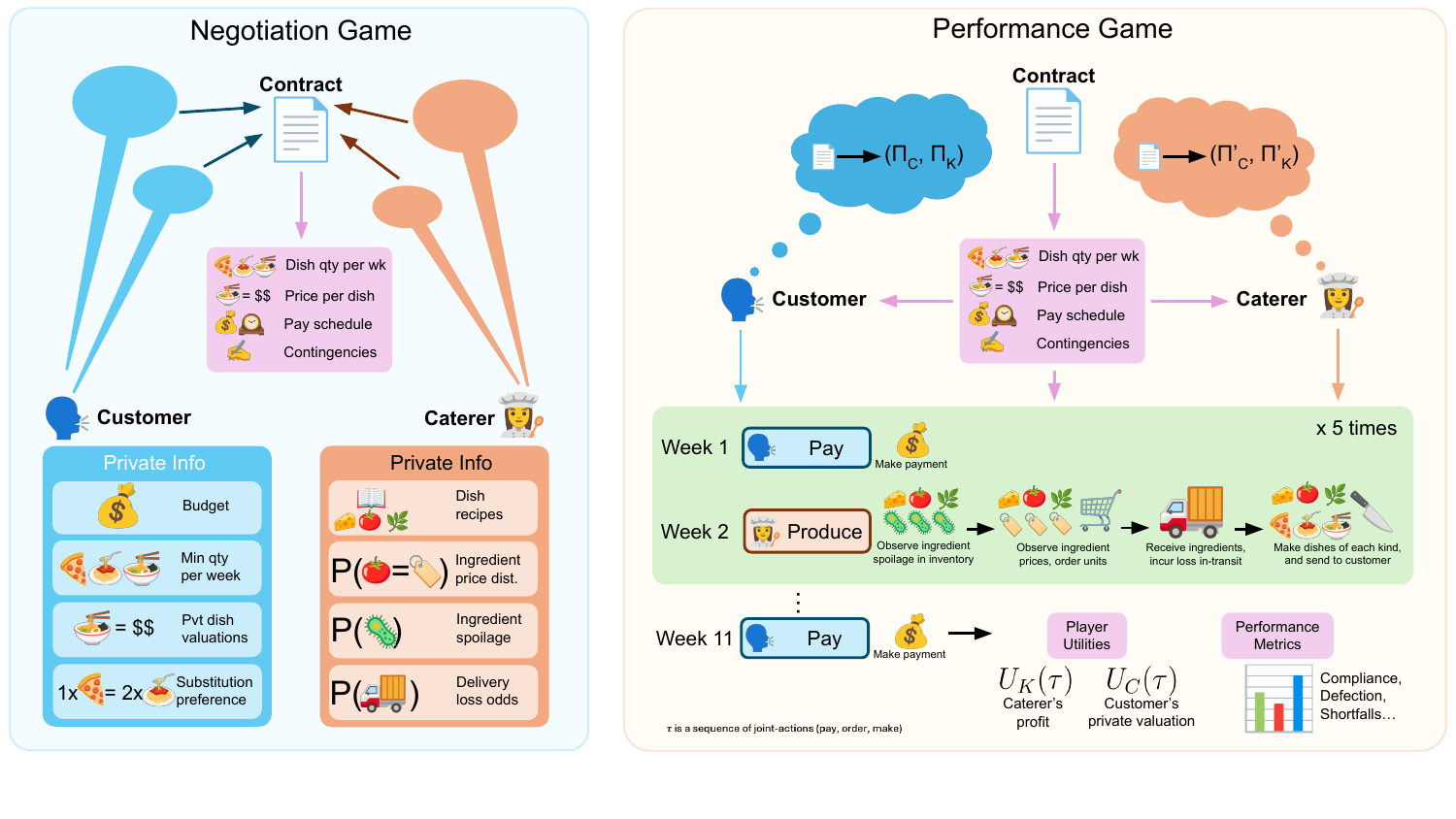}
        \caption{\textbf{Evaluating contracting with \textsc{ContractSim}.} \emph{Negotiation stage:} Two LLM agents \emph{negotiate} a natural-language catering contract given their private information. \emph{Performance stage:} The accepted contract is then \emph{performed} over alternating payment/production weeks with stochastic prices, spoilage, and delivery losses. We benchmark performance in each stage against rational baselines, enabling us to study the extent to which LLM agents can serve as rational and cooperative contractors.}
        \label{fig:pipeline}
        \Description{Pipeline diagram of ContractSim: two LLM agents negotiate a catering contract, which is then performed over alternating payment and production weeks under stochastic prices, spoilage, and delivery losses.}
\end{figure*}

\section{Introduction}
\label{sec:intro}

As AI agents based on large language models (LLMs) become ever more autonomous and capable, we may eventually see the emergence of an \emph{agentic economy}, where AI agents negotiate, coordinate, and carry out economic transactions on behalf of humans~\cite{rothschild2025agenticeconomy,tomasev2025virtual,shahidi2025coasean}. In contrast to earlier generations of economic agents, such as algorithmic traders~\cite{almgren2000optimal,hendershott2011liquidity}, auction agents~\cite{stone2002attac}, smart contracts \cite{szabo1997formalizing,werbach2017contracts}, and negotiation algorithms~\cite{kraus1995multiagent,faratin1998negotiation,kraus1998argumentation,jennings2001automated}, modern AI agents are distinguished by their ability to use \emph{natural language}. This opens up the possibility of agreements and contracts over a much wider space of terms, going beyond just prices and quantities to include dynamic obligations \cite{battaglini2019optimal}, contingencies \cite{bazerman1999betting}, and qualitative assurances \cite{bernheim1998incomplete,choi2009strategic}. By sharply reducing the costs of forming such agreements, AI agents promise to expand the scale, scope, and efficiency of human coordination \cite{shahidi2025coasean,krier2025coasean}.

With this complexity and open-endedness, however, it is unclear how well current AI agents perform at natural language contracting, and hence whether they can deliver on the promises of an agentic economy. After all, while LLMs excel at some forms of language-based reasoning \cite{feng2026towards,nguyen2025llms}, strong negotiation requires abilities like theory-of-mind \cite{de2017negotiating,chan2024negotiationtom} and planning under uncertainty \cite{schauer2023nine} that LLMs tend to be weaker at \cite{deng2025planu,kim2025hypothesis,ying2025liras}. It is also unclear how to even measure what good contracting looks like. For one, natural language contracts lack obvious metrics for assessing properties such as feasibility or mutual benefit. For another, successful contracting is a \emph{cooperative} endeavor that requires mutual compliance to achieve mutual benefit \cite{telser1980theory}; evaluations should thus measure agents' ability to achieve these cooperative outcomes instead of selfishly breaching a contract when convenient.

To address these challenges, we propose a framework that models and evaluates contracting in \emph{rational} terms. We formulate contracting as a two-party \emph{negotiation-performance game}, where rational and cooperative contractors should negotiate a mutually beneficial contract, comply with the contract conditional on the other party's compliance, then optimize their utilities subject to contractual constraints. To evaluate natural language contracts in this framework, we translate them into \emph{constraints on joint policies}, allowing us to implement rational baselines that follow these constraints while otherwise acting in their own interests, and to quantitatively assess contract quality. We instantiate this framework in \textsc{ContractSim}, where two agents negotiate a multi-turn supplier contract in natural language and then perform the contract in a stochastic environment (Figure \ref{fig:pipeline}). This allows us to benchmark a range of LLM-based agents in their ability to behave both rationally and cooperatively.

Empirically, we find that although LLM agents negotiate mutually beneficial and satisfiable contracts in environments with little to no stochasticity, they struggle in high-stochasticity environments. In one such environment, only 3 out of 9 model pairings negotiate contracts
that are mutually beneficial, and LLM agents are outperformed by a heuristic negotiation baseline. %
Unless prompted to do so, agents also fail to add contingency clauses that could improve contract quality. When it comes to contract performance, LLM agents are able to gain high utility, but do so by breaching agreements and defecting against cooperative counterparties (16--90\% of episodes). Defection remains consistent even when contracts become easier to satisfy. 
However, prompting agents against defection via direct instruction or institutional incentives results in substantially lower defection rates (2.8--13.1\%), albeit with significant reductions in utility. Together, these results indicate that current LLM agents are not yet reliably efficient contractors, and that ensuring trustworthy contracting in an agentic economy may require institutions that incentivize either cooperative scaffolding by agent developers or direct cooperation by AI agents. 

\section{Related Work}
\label{sec:related}

\paragraph{Automated Negotiation.} Research on AI negotiation has largely focused on multi-issue agreements, where parties exchange proposals over fixed-dimensional outcomes whose utility is directly given~\cite{zeng1998bayesian,faratin1998negotiation,jonker2007anaa}. Neural bargaining extends the space of negotiation moves (but not the agreement space) to handle natural language~\cite{lewis2017deal,chawla2021casino}.
Our work builds on this literature by studying natural language contracts, which can contain incomplete or contingent clauses not expressible as a fixed number of issues. Furthermore, contracts are not directly assigned utilities; instead, agents have to infer what a contract is worth by considering how it is performed.

\paragraph{Bargaining and Contract Theory.} In formulating rational contracting, we draw on concepts from bargaining and game theory~\cite{nash1950bargaining,rubinstein1982bargaining,myerson1983efficient}, treating negotiation as the rational search for mutually beneficial contracts. Our emphasis on incompleteness and contingency is also inspired by economic and legal contract design~\cite{telser1980theory,hart1988incomplete,sandholm2001leveled,choi2009strategic,hadfieldmenell2019incomplete,thomas2022rational} --- features omnipresent in real-world contracts, but have received limited study in AI negotiation.

\paragraph{Formal and Self-Executing Contracts.} While natural language contracts lack a precise formal semantics, for the purposes of quantitative assessment, we interpret them as formal constraints on policies in stochastic games~\cite{hansen2004dynamic}. This extends work in cooperative AI where contracts are formalized as conditional transfers of reward \cite{haupt2024formalcontractsmitigatesocial,tewolde2026coopeval}. By parsing contracts into formal constraints, our work also contains elements of legal AI workflows~\cite{hendrycks2021cuad,koreeda2021contractnli,chalkidis2022lexglue}. However, our focus is not on parsing itself; rather, it is a tool to enable quantitative evaluation. A separate line of work studies contracts as formal programs that are \emph{self-executing}~\cite{gopinathan2026pact,wyse2026commitment}, or where penalties are enforced after a breach~\cite{sandholm2001leveled}. We instead study whether agents can be trusted to execute the contracts they negotiate.

\paragraph{Negotiation and Strategic Reasoning in LLMs.} Strategic rationality in LLMs has been evaluated in \emph{Diplomacy}, repeated games, social dilemmas, and business simulations~\cite{bakhtin2022diplomacy,akata2023repeated,piatti2024cooperate,vezhnevets2023generativeagentbasedmodelingactions, backlund2025vending}. A growing line of work benchmarks LLM negotiation itself, spanning scorable bargaining~\cite{abdelnabi2024stakeholders,davidson2024evaluating}, price negotiation~\cite{xia2024measuring,fu2023improving}, resource trading~\cite{bianchi2024negotiationarena,qian2026strategic}, and multi-agent markets~\cite{liu2026agenticpay, scml2026website}.  However, work thus far has focused on the negotiation of one-off exchanges or self-executing contracts, providing little analysis of how agreements are \emph{performed}. We instead study \emph{time-extended} contracts and treat performance as a key component of contracting. Since contracts are not self-executing, compliance and defection become strategic choices for agents to reason about. Unlike benchmarks that focus on profit, this enables us to quantify whether agents are \emph{cooperative} contractors who reliably keep their promises.
This emphasis on performance is also crucial to how we evaluate negotiation: negotiators are not directly given the utility of contract, and must instead improve the quality of a contract by considering its execution in an (uncertain) environment. As a result, we can study features of natural language contracts like incompleteness and contingency clauses; by adding contingencies and increasing completeness, agents can reduce the chances of a contract breaking down.

To summarize, \textsc{ContractSim} contributes the first quantitative evaluation of LLM negotiation over dynamic, contingent, and incomplete contracts in natural language, bringing evaluation practices closer to the open-ended scope of future agentic agreements.
\section{Rational Contracting as a Two-Stage Game}
\label{sec:game}

To characterize rational contracting, we first formalize contracting as a game where players negotiate a contract, then perform the contract if agreement is reached.

A \emph{negotiation-performance game} $\mathcal{G}$ is described by the tuple: $
    \mathcal{G}
    =
    \left(
        \mathcal{I},\,
        \Theta,\,
        \Omega,\,
        \mathcal{N},\,
        \mathcal{P}
    \right),
$
where $\mathcal{I}=\{1,2\}$ is the player set, $\Theta$ is the
type-profile space, $\Omega$ is the contract space, $\mathcal{N}$ is
the negotiation game, and $\mathcal{P}$ is the performance game. Each
run realizes types $\theta=(\theta_i)_{i\in\mathcal{I}}\in\Theta$ that describe each player's private information about their preferences or environment. Players negotiate over $\Omega$ in $\mathcal{N}$; agreement produces
$\omega\in\Omega$, which they perform in $\mathcal{P}$ with their types retained. Disagreement is terminal.

A \emph{negotiation game} $\mathcal{N}$ is defined by the tuple
$
    \left(
        \mathcal{M},\,
        \Omega,
        \alpha
    \right),
$
where $\mathcal{M}$ is the natural-language message space,
$\Omega$ is the natural-language
contract space, and $\alpha$ is a protocol
that maps a message history $h \in  \mathcal{M}^*$ to an agreed contract
$\omega \in \Omega$, disagreement $\bot$, or continuation $\top$. Each turn, a player sends a message $m \in \mathcal{M}$. Based on the history $h$, $\alpha$ determines if negotiation continues ($\top$), ends in disagreement ($\bot$), or ends with a contract $\omega$ extracted from $h$.

A \emph{performance game} $\mathcal{P}$ is a partially observable stochastic game \cite{hansen2004dynamic} defined by
$
    \left(
        \mathcal{S},\,
        (\mathcal{A}_i)_{i\in\mathcal{I}},\,
        (\mathcal{O}_i)_{i\in\mathcal{I}},\,
        T,\,
        (O_i)_{i\in\mathcal{I}},\,
        L,\,
        (u_i)_{i\in\mathcal{I}}
    \right),
$
where $\mathcal{S}$ is the environment state space,
$\mathcal{A}_i$ is player $i$'s action space, $\mathcal{O}_i$ is player $i$'s observation space,
$T(\cdot \mid s,a)$ is the transition kernel for joint action
$a=(a_i)_{i\in\mathcal{I}}$, $O_i(\cdot | s, a)$ is player $i$'s observation model, $L \in \mathbb{N}$ is the length of the game,
and $u_i(\tau;\theta_i)$ is player $i$'s utility over complete
trajectories $\tau \in \mathcal{T}$, where
$\mathcal{T}:=\mathcal{S}\times(\mathcal{A}\times\mathcal{S})^L$.
Execution starts from an initial state $s_0 \in \mathcal{S}$, after which each player $i$ follows a policy $\pi_i$ which may condition on their observations, the contract, and the negotiation
message history:
$\pi_i(a_{i,t} \mid o_{i,1:t}, \theta_i, \omega, h)$.

\subsection{Contracts as Constraints on Joint Policies}
\label{subsec:contracts}

In the formalism above, a natural language contract $\omega$ is just a signal on which players can condition future actions. However, real contracts typically specify the actions and obligations required of each player, enabling mutually beneficial coordination. As such, we assume that a contract $\omega$ can be interpreted as a set of contractually compliant policies $\Pi^\omega \subseteq \Pi$, where $\Pi$ is the full space of joint policies $\pi = (\pi_i, \pi_{-i})$. We further refine this by assuming $\Pi^\omega$ is defined by a constraint $C^\omega: \mathcal{T} \to \{0, 1\}$ on the performance trajectory $\tau$ and a violation probability $\epsilon$, such that $\Pi^\omega := \{\pi \in \Pi \mid P_{\mathrm{sat}}(\pi, C^\omega) \geq 1-\epsilon \}$. $P_{\mathrm{sat}}(\pi, C^\omega) := \mathbb{E}_{\tau}[C^\omega(\tau) | \pi]$
is the probability of $\pi$ satisfying the constraint $C^\omega$, so a contractually compliant policy is one that satisfies $C^\omega$ with high probability.

By interpreting contracts in this way, we can capture several important aspects of natural language contracts: \emph{Contingent contracts} \cite{bazerman1999betting} can be modeled with history-dependent constraints $C^\omega$ that are only satisfied when particular outcomes follow certain conditions. We can also model \emph{incomplete contracts} \cite{hart1988incomplete,hadfieldmenell2019incomplete}, due either to \emph{under-specification} --- $\Pi^\omega$ does not determine a unique policy --- or \emph{unanticipated unsatisfiability} ---  constraints $C^\omega$ may be impossible to satisfy in certain states, making it unclear how to act. With this interpretation, we are able to formalize various measures of contract quality (Section~\ref{sec:infra}), and evaluate them by translating natural language into formal constraints.

\subsection{Modeling Rational Contracting}
\label{subsec:rational-contracting}

What does it mean to negotiate and perform a contract rationally? Game-theoretic characterizations provide limited guidance here --- a wide variety of equilibria are supported in games similar to ours, including ones where no economic exchange occurs \cite{akerlof1978market,telser1980theory}, where negotiation is just cheap talk that results in babbling or arbitrary signaling \cite{crawford1982strategic,farrell1996cheap}, or where unrealistically efficient negotiation terminates in a single turn \cite{rubinstein1982bargaining,chatterjee1983bargaining}. Furthermore, absent institutional incentives (e.g., repeated interaction,
reputational costs, or third-party enforcement \cite{telser1980theory,baker2002relational,tewolde2026coopeval}), complying with a negotiated contract may not be \emph{individually rational} when the other party complies or defects, even if it is \emph{collectively rational} \cite{gauthier1990rationality,weirich2009collective} for both parties to commit to compliance. 

We avoid these difficulties by focusing on behaviors that can be meaningfully described as contracting, assuming that agreement results in a contract $\omega$ interpretable as a policy constraint $\Pi^\omega \equiv (C^\omega, \epsilon)$. Accordingly, we develop both \emph{cooperative} and \emph{non-cooperative} models of rational contract performance, where agents either unconditionally comply with a contract (modeled as optimization under contractual constraints), exploit unconditional compliers, or defend against exploitation via conditional compliance. In our experiments, these agents serve as both benchmarks and opponents to compare LLM agents against. By grounding the value of a contract in the expected utility of joint conditional compliance, we then derive upper bounds on the contract values achievable by rational negotiation, without having to model negotiation dynamics. Thus, we frame rational negotiation as the search for the contract that yields maximum expected mutual benefit to agents, allowing us to evaluate negotiation in terms of the expected gains to each agent if they both comply with the negotiated contract. 

\textbf{Rational Performance.}
\label{subsubsec:rational-performance}
Let $\Pi_i$ denote player $i$'s full performance-policy space. We define a
\emph{Rational Complier (RC)} as the utility-maximizing policy that is
contractually compliant when paired with a counterparty that is \emph{expected} to follow a policy $\pi^{\text{ref}}_{-i}$ which also complies with the contract $\omega$:
\begin{equation}
\label{eq:rc}
\begin{aligned}
    \pi_i^{\mathrm{RC}}
    \in
    &\arg\max_{\pi_i\in\Pi_i}
    \mathbb{E}\!\left[
        u_i(\tau;\theta_i)
        \mid \omega, \pi_i,\pi_{-i}^{\mathrm{ref}}
    \right]
    &\text{s.t.}\; (\pi_i,\pi_{-i}^{\mathrm{ref}})\in\Pi^\omega.
\end{aligned}
\end{equation}
To fix a counterparty policy $\pi^{\text{ref}}_{-i}$, we note that there are often natural choices given the contract $\omega$. For example, in the supplier-customer games our benchmark studies (Section~\ref{sec:benchmark}), most contracts fully constrain the customer's policy (i.e., their payment schedule), leaving a unique choice for $\pi^{\text{ref}}_{-i}$ where $-i$ is the customer. This in turn determines the supplier's RC policy $\pi^{\mathrm{RC}}_{i}$. In less constrained settings, $\pi^{\mathrm{RC}}_{i}$ and $\pi^{\text{ref}}_{-i}$ can be jointly determined as a bargaining solution  \cite{nash1950bargaining,kalai1975other} among compliant policies. 

Besides being a best response, RC can be viewed as rational in several ways:
Under resolute choice, it is rational for agents to commit to compliance so that counterparties will agree to contract with them \citep{mcclennen1990rationality,gauthier1990rationality}.
Compliance can also be made individually and sequentially rational via
institutions such as repetition, reputation, arbitration, or penalties.
We illustrate the rationality of compliance under both repetition and resolute choice in the Appendix.

Against such a complier, a self-interested agent need not remain within the
contractual constraint. We thus define the \emph{Rational Exploiter (RE)} as the
unconstrained best response to the counterparty's RC policy:
\begin{equation}
\label{eq:re}
\begin{aligned}
    \pi_i^{\mathrm{RE}}
    \in
    &\arg\max_{\pi_i \in \Pi_i}
    \mathbb{E}\!\left[
        u_i(\tau;\theta_i)
        \mid \omega,\pi_i,\pi_{-i}^{\mathrm{RC}}
    \right].
\end{aligned}
\end{equation}
RE represents a \emph{non-cooperative} contractor. In our benchmark, we use RE as a test of agents' robustness to adversarial play.

Finally, the \emph{Rational Conditional Complier (RCC)} makes cooperative
contracting robust to exploitation. It follows RC as long as the counterparty
complies, but if the counterparty is observed to violate its obligations under
$C^\omega$ at time-step $t$, RCC switches to a best response $\text{BR}_i$ to an RE policy. $D_{i,t}\in\{0,1\}$ indicates whether player $i$ has observed a counterparty violation by time $t$:
\begin{equation}
\label{eq:rcc}
    \pi_i^{\mathrm{RCC}}
    =
    \begin{cases}
        \pi_i^{\mathrm{RC}}, &
        \text{if } D_{i,t}=0,\\
        \operatorname{BR}_i(\pi_{-i}^{\mathrm{RE}}), &
        \text{if } D_{i,t}=1.
    \end{cases}
\end{equation}
RCC thus models many desirable aspects of a contracting agent: Besides maximizing utility, it cooperates through compliance by default, but also defends itself against exploitation once it becomes clear that the counterparty is non-cooperative.

\textbf{Rational Negotiation.} \label{subsubsec:rational-negotiation} Since rational performance is now well-defined, we can measure the quality of a negotiated contract by its expected utility $U_i(\omega) = \mathbb{E}\!\left[u_i(\tau;\theta_i)\mid \omega,\pi^{\mathrm{RCC}}\right]$ under joint conditional compliance $\pi^{\mathrm{RCC}} := (\pi_i^{\mathrm{RCC}}, \pi_{-i}^{\mathrm{RCC}})$, where the constraints $C^\omega$ are satisfied with probability at least $1-\epsilon$. We now introduce two models of negotiation focused on the quality of accepted contracts, leaving the analysis of negotiation dynamics to future work. 

Consider a scenario where both agents reveal their private information and preferences $\theta$, and agent $-i$ commits to accepting any contract that gives it an expected utility of at least $\eta$. A rational agent $i$ should then propose a contract $\omega^*(\eta)$ that maximizes its own expected utility $U_i(\omega)$, subject to meeting the counterparty's utility threshold $\eta$. We call this an \emph{Honest \& Optimal} negotiation:
\begin{equation}
\label{eq:rational-negotiation}
\begin{aligned}
&\omega^{*}(\eta) \;\in\; {\arg\max}_{\omega \,\in\, \Omega} \;\;
  U_{i}(\omega) \\
\text{s.t.} \quad
  &U_{-i}(\omega) \;\geq\; \eta, \quad P_{\mathrm{sat}}\!\left(\pi^{\mathrm{RCC}}, C^{\omega}\right) \geq 1 - \epsilon .
\end{aligned}
\end{equation}
Solving for $\omega^*(\eta)$ across values of $\eta$ gives us a Pareto frontier of $\epsilon$-satisfiable contracts, where each party $i$ cannot improve their utility $U_i$ without decreasing the other's utility $U_{-i}$. Under incomplete information, individually rational negotiators generally cannot reach Pareto-efficient agreements due to incentives to withhold information \cite{chatterjee1983bargaining, myerson1983efficient}. Nonetheless, this model gives an \emph{upper bound} on what fully honest negotiators can achieve. Since the search for an agreement happens over the entire space of contracts, the model demonstrates \emph{global rationality}.

To simulate negotiation in cases where private information is withheld by parties, we consider a \emph{Concessionary} negotiation model. In this model, private information is not disclosed; negotiators use time-dependent concession tactics \cite{faratin1998negotiation, jonker2007anaa} to propose contracts. Each party $i$ scores contracts with an estimate $\hat{U}_i(\omega;\theta_i)$ built from its own private information and the information in $\omega$, and gradually reduces their \emph{aspiration $\eta_i(t)$} from an opening target $\bar{\eta}_i$ toward its outside option $\underline{\eta}_i$ as the deadline $t_{\max}$ approaches:
\begin{equation}
\label{eq:concession-target}
    \eta_i(t) = \underline{\eta}_i + \left(\bar{\eta}_i - \underline{\eta}_i\right)\left(1 - (t/t_{\max})^{1/\beta}\right),
\end{equation}
where $\beta$ controls how quickly $i$ concedes. On each of its turns $t$, $i$ accepts the counterparty's last offer $\omega_{-i}$ if it meets the current target, and otherwise proposes a contract just above it:
\begin{equation}
\label{eq:concession-offer-accept}
\begin{aligned}
    &\text{accept } \omega_{-i} \iff \omega_{-i} \in \Omega_i,\; \hat{U}_i(\omega_{-i};\theta_i) \ge \eta_i(t), \\
    &\text{else propose } \omega_i^{t} \in \left\{\omega \in \Omega_i : \eta_i(t) \le \hat{U}_i(\omega;\theta_i) \le \eta_i(t) + \delta \right\},
\end{aligned}
\end{equation}
where $\Omega_i \subseteq \Omega$ is the set of contracts satisfying $i$'s constraints and $\delta$ a small tolerance. If no offer is accepted by $t_{\max}$, both parties take their outside options. The above model exercises \emph{local rationality}, since each agent proposes an agreement within the local neighborhood of their aspiration in every round.

Besides contract utility and satisfiability, our framework also affords other measures of contract and negotiation quality (see Section~\ref{sec:infra}), allowing us to give a rich characterization of how AI agents negotiate contracts.

\section{The \textsc{ContractSim} Evaluation Suite}
\label{sec:benchmark}

\textsc{ContractSim} instantiates our contracting formalism as
$\mathcal{G}'=
\left(
    \mathcal{I}',\,
    \Theta',\,
    \Omega',\,
    \mathcal{N}',\,
    \mathcal{P}'
\right)$,
where $\mathcal{I}'=\{\mathrm{Cust},\mathrm{Supp}\}$ indexes the Customer and
Supplier, $\Omega'$ is the contract space, and a
run fixes a type profile
$\theta'=(\theta_{\mathrm{Cust}},\theta_{\mathrm{Supp}})\in\Theta'$.
Here $\theta_{\mathrm{Cust}}$ specifies the Customer's budget, minimum delivery
requirements, consumption utilities, and substitution rule, while
$\theta_{\mathrm{Supp}}$ specifies the Supplier's production function, initial resources, and
private information about supply uncertainty, including prices,
delivery losses, and spoilage.

Supplier contracting is a useful test domain because it captures many of the features we want to study in open-ended natural-language contracting while remaining structured enough to analyze. Players negotiate over a rich contract space $\Omega'$, then execute those agreements in a performance game $\mathcal{P}'$ with changing costs, inventory constraints, and possible shortfalls. This naturally creates room for contingent and incomplete contracting, since agreements can specify fallback terms for some contingencies while leaving others unresolved. At the same time, the domain remains tractable because each contract $\omega\in\Omega'$ can be parsed into formal trajectory constraints $C^\omega$ --- prices, delivery quantities, payment schedules, and optional contingency clauses --- which can then be analyzed to quantify player performance.

\textbf{Settings and Environments.} We study three supplier settings that are structurally equivalent but have different linguistic descriptions: \emph{Catering}, \emph{Hotel Cleaning}, and \emph{AI Hosting}. The primary setting we evaluate is \emph{Catering} (illustrated in Figure~\ref{fig:pipeline}), where the Customer seeks recurring deliveries
of several dishes subject to minimum requirements and a budget,
while the Supplier produces these dishes from ingredients using fixed
recipes. In each setting, we vary the difficulty of contracting by creating six environments that differ in environmental stochasticity (none, low, high), supplier capital, and customer valuations (low vs. high). See the Appendix for full descriptions of each environment.

\textbf{Negotiation.}
In $\mathcal{N}'$, the Customer and Supplier exchange natural-language proposals for at most 50 rounds. A typical contract $\omega\in\Omega'$
might specify per-product prices, a delivery schedule over production
weeks, a payment schedule over payment weeks, and several
contingency clauses. Agreement is reached only when the Customer accepts the Supplier's most recent formal contract;
otherwise, both agents walk away with their initial capital.

{\textbf{Contract Parsing.} \textsc{ContractSim} decouples the language in which agents
negotiate from the structured representation of contracts used for evaluation. Unlike multi-issue negotiation (Section~\ref{sec:related}),
agents exchange free-form natural-language contracts $\omega$, which are then translated into structured representations of trajectory constraints $C^\omega$ by an independent LLM parser. While our parser does not cover all conceivable constraints, it nonetheless allows us to quantitatively evaluate a much richer space of contracts than prior negotiation benchmarks. For example, aside from numeric price and delivery terms, these contracts may include contingency clauses and remedies that are otherwise difficult to express as a fixed set of issues. Agents must state core terms explicitly so that utilities remain computable, but may otherwise add any terms, which our parser maps onto an extensible set of supported contingencies.

\textbf{Performance.}
In $\mathcal{P}'$, a contract is executed over the course of $L=11$ weeks with alternating payment and production periods. Let $D$ denote the product set and
$\mathcal{X}$ the inventory set. A contract $\omega$ can be formalized as a trajectory constraint
$C^\omega=(\mathbf{p}^d,\mathbf{q},\mathbf{M},\kappa)$, where
$\mathbf{p}^d=(p_d)_{d\in D}$ are product prices, $\mathbf{q}=(q_{w,d})$ is the
delivery schedule of integer product counts $q_{w,d}\in\mathbb{Z}_{\ge 0}$,
$\mathbf{M}$ is the payment schedule, and $\kappa$ specifies the presence of four supported contingencies: substitution, payment
deduction, rollover, and grim trigger.

In each payment week $w$, the Customer chooses how much to pay, $M_w^{\mathrm{paid}}$. In each production week, spoilage is applied, prices are realized, the Supplier orders
inputs, delivery losses determine receipts, and the Supplier
chooses a feasible production plan. Both parties observe the contract,
realized payments, and delivered products, but only the Supplier observes
its inventory, current prices, spoilage outcomes, and delivered inputs. Completed products are delivered immediately, unused
inventory carries forward, and no renegotiation or post-contract
communication is allowed after agreement.

\textbf{Customer Utility.} Let $B$ denote the Customer's
initial budget, $v_d$ its private value for product $d$, and
$q'_{w,d}$ the quantity of $d$ delivered in week $w$. The Customer's accumulated private value is the sum of its remaining budget and the cumulative private dish valuations, 
$U_{\mathrm{Cust}}(\tau)=B+\sum_{w \in W_{\mathrm{prod}}}\sum_{d \in D}v_d\,q'_{w,d}-\sum_{w \in W_{\mathrm{pay}}}M_w^{\mathrm{paid}}.$ 

\textbf{Supplier Utility.} Let $o_{w,x}$ be the amount of
ingredient $x$ ordered in production period $w$, and let $p_{w,x}$ be
its realized price. The Supplier's utility is its profit: the difference between the amount received in payments and the cost incurred for ordering ingredients: 
$U_{\mathrm{Supp}}(\tau)=\sum_{w \in W_{\mathrm{pay}}}M_w^{\mathrm{paid}}-\sum_{w \in W_{\mathrm{prod}}}\sum_{x \in \mathcal{X}}o_{w,x}\,p_{w,x}.$
\section{Evaluating Rational Contracting}
\label{sec:evals}

We evaluate the contracting capabilities of three frontier LLMs (Claude
Opus~5, Gemini~3.6 Flash, and GPT-5.6-Sol) on
\textsc{ContractSim}, with high reasoning enabled. We use Concordia
\cite{vezhnevets2023generativeagentbasedmodelingactions} as a standardized multi-agent LLM scaffold; prompts and further details can be found in the Appendix.

\subsection{Rational Baselines and Contract Translation}
\label{subsec:baselines}

Following Section~\ref{subsec:rational-contracting}, we implement baselines for rational performance as finite-horizon dynamic
programs, with policies conditioned on the contract,
private type, and observation history. RC is implemented by solving the constrained Markov Decision Process \cite{altman2021constrained} implied by Eq.~\ref{eq:rc} under constraints $C^\omega$. RE is implemented as non-engagement, since this is the optimal unconstrained policy for the Supplier when the Customer fully complies with the payment schedule (Eq.~\ref{eq:re}), and is also a reasonable RE policy approximation for the Customer (which is exact when the Supplier has sufficient initial capital to fully comply; see the Appendix for proofs). RCC behaves like RC and switches to RE when the counterparty violates the contract, since non-engagement is a best response to non-engagement in our setting (see Appendix).

To evaluate the expected utility $U_i(\omega)$ of the LLM-negotiated natural-language contracts, we use Gemini 3.6 Flash (owing to its accuracy and fast inference) to translate $\omega$ into the structured constraints $C^\omega$ described in Section~\ref{sec:benchmark}, and assume a violation rate of $\epsilon = 0.05$. We also support the translation and evaluation of several contingency clauses: grim trigger terminates engagement after a violation, payment deduction reduces the next payment after a shortfall,
rollover carries a shortfall forward, and substitution permits delivery of value-equivalent products.

\subsection{Evaluating Contract Negotiation}
\label{sec:infra}

\newcommand{\meanstd}[2]{#1{\hspace{0.1em}\raisebox{0.25ex}{\scalebox{0.5}{$\pm$\,#2}}}}
\newcommand{\thinrowrule}{\specialrule{0.05em}{0pt}{0pt}}
\newcommand{\thickrowrule}{\specialrule{0.11em}{0pt}{0pt}}

In each of our six \emph{Catering} environments (see the Appendix for full parameter specifications) we run negotiations between all $3 \times 3$ pairs of LLMs with 3 repetitions. This yields 162 negotiations, of which 159 reach agreement
($98.1\%$). Manual verification revealed parsing accuracy and coverage to be 100\%: all final agreements fell within the range of contracts supported by our contract parser (see Appendix for details).

We also run the \emph{Honest \& Optimal} (Eq.~\ref{eq:rational-negotiation}) and \emph{Concessionary} negotiators
(Eqs.~\ref{eq:concession-target}--\ref{eq:concession-offer-accept}) for both parties, sweeping over minimal utility targets for each party to create a Pareto frontier. For the Concessionary baseline, the Customer assumes every promised dish is delivered, and the Supplier heuristically estimates its profit and satisfiability from mean ingredient prices and its own cash flow. Counter-offers are random perturbations of the counterparty's last offer that fall within the agent's current target band. Further details about each agent's heuristics are in the Appendix.

We evaluate the quality of accepted contracts $\omega$ and the
negotiation process using \emph{expected Customer and Supplier
utilities} $(U_{\mathrm{Cust}}(\omega), U_{\mathrm{Supp}}(\omega))$, \emph{satisfaction probability} $P_\text{sat}$, whether $\omega$ is \emph{mutually
beneficial} relative to disagreement, \emph{proposal feasibility}, and
\emph{acceptance regret} relative to earlier feasible offers. Each metric is
averaged over negotiations in which the model occupies the scored role. We
additionally measure \emph{contingency count} over the elective clauses and
\emph{contract completeness} under the optimal RCC policy. Full definitions are given in the Appendix. We now highlight several key findings.

\begin{figure}[h!]
\centering
\includegraphics[width=\columnwidth,trim=0 16 0 0,clip]{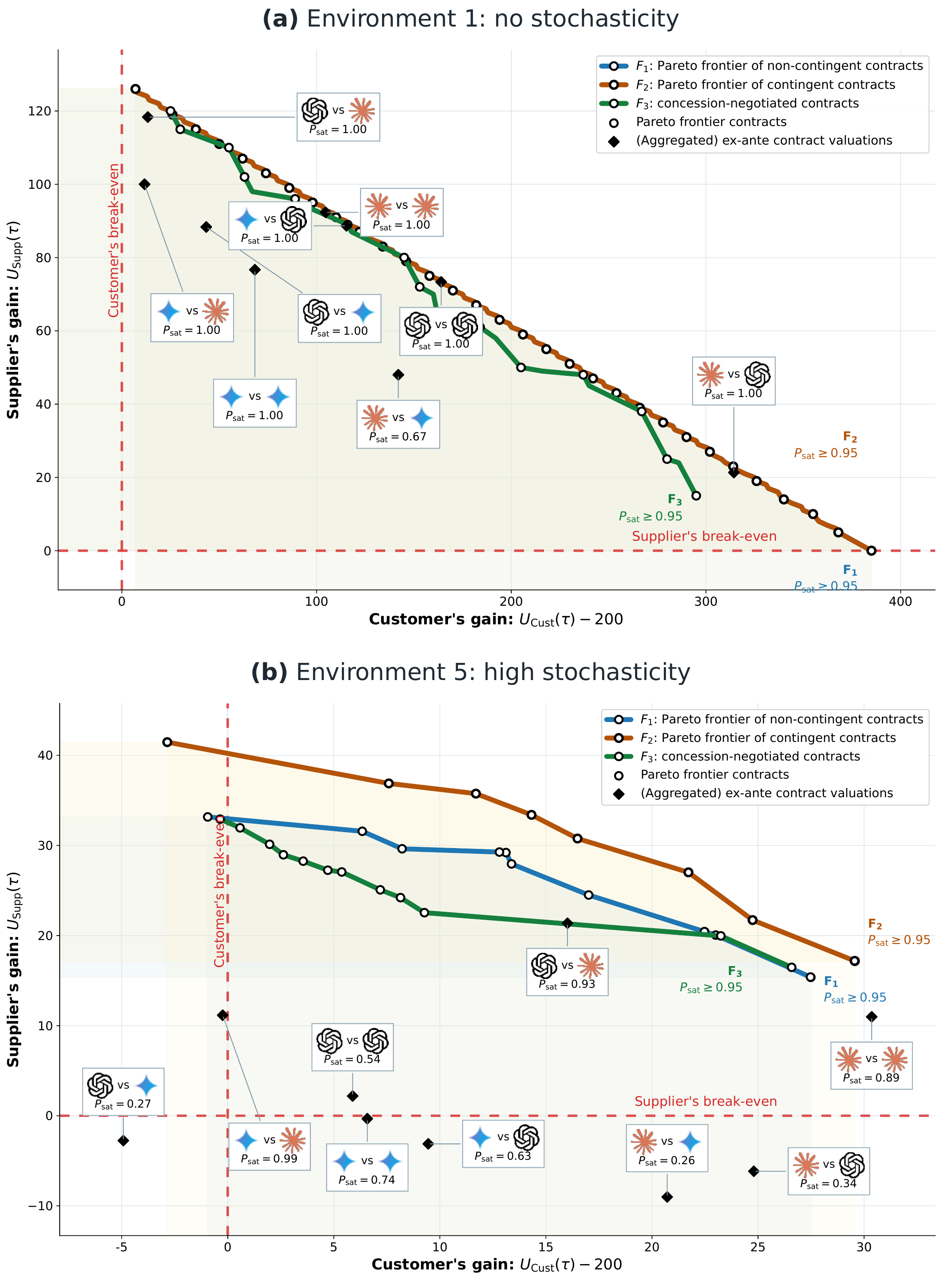}
\caption{\textbf{Negotiated contracts versus Pareto frontiers in environments with no and high stochasticity.} Each point is the joint utility and satisfiability ($P_{\mathrm{sat}}$) an LLM pair achieves, averaged across 3 negotiated contracts. $F_1$ is the Pareto frontier for $P_{\mathrm{sat}}\geq0.95$ contracts with no contingencies. $F_2$ is the corresponding frontier for contingent contracts. $F_3$ is the Pareto frontier of $P_{\mathrm{sat}}\geq0.95$ contracts reached by Concessionary negotiation.}
\label{fig:neg-frontier}
\Description{Two vertically stacked scatter plots, for Environments 1 (no stochasticity) and 5 (high stochasticity), of the Supplier's gain against the Customer's gain for contracts negotiated by each pair of LLMs, with each point labelled by the pair and its satisfiability. Points are means over 3 runs per ordered pair. Three Pareto frontiers are overlaid in each: F1 for non-contingent contracts, F2 for contingent contracts, and F3 (green) for contracts reached by concessionary negotiation, all with satisfiability of at least 0.95. Dashed lines mark each party's break-even point. LLM contracts lie on or near the frontiers in Environment 1 and fall well below all frontiers in Environment 5.}
\end{figure}

\begin{figure}[h!]
\includegraphics[width=\columnwidth,trim=0 2.2 0 0,clip]{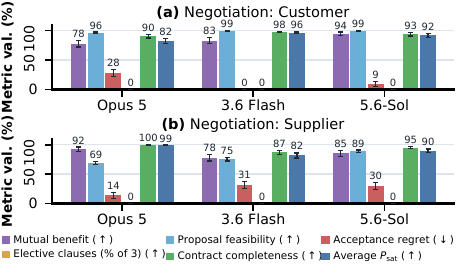}
\caption{\textbf{Role-conditioned negotiation quality.} Negotiation metrics
for (a) Customer and (b) Supplier models averaged across all six environments. Each bar shows the mean $\pm$ SE. See Appendix for metrics broken down by environment.}
\label{fig:negotiation-quality-roles}
\Description{Grouped bar charts of negotiation quality metrics across all six environments, shown separately for each LLM acting as (a) the Customer and (b) the Supplier, with standard-error bars.}
\vspace{-12pt}
\end{figure}

\textbf{Negotiation is inefficient in stochastic environments.}
Figure~\ref{fig:neg-frontier} shows $U_{\mathrm{Cust}}(\omega)$, $U_{\mathrm{Supp}}(\omega)$ and $P_\text{sat}$ (averaged across 3 runs) for the nine LLM pairings in Environments~1 and~5, which have no and high stochasticity. We compare these points against the Pareto frontier of both non-contingent and contingent contracts achieved by the Honest \& Optimal negotiator (Eq.~\ref{eq:rational-negotiation}). Under no stochasticity, we find that LLM-negotiated contracts are fairly efficient, with some even lying on the Pareto frontier (Fig. \ref{fig:neg-frontier}a). However, in high-stochasticity environments, LLM agents do not reliably achieve efficiency or mutual benefit: only 3 of 9 pairings are mutually beneficial in Environment~5 (Fig. \ref{fig:neg-frontier}b), and only 7 of 9 in Environment~6 (see Appendix), compared with all pairings in the other four environments. This is partly because LLMs fail to negotiate contingent contracts, which can achieve higher utilities for the same $P_\text{sat}$ ($F_2$ vs. $F_1$ in Figure~\ref{fig:neg-frontier}). 

\textbf{LLM agents are outperformed by the locally-rational concessionary negotiation model.} Similar to $F_1$ and $F_2$, we plot the Pareto frontier of contracts achieved by Concessionary negotiation ($F_3$, Figure~\ref{fig:neg-frontier}), by sweeping over the outside option values $\underline{\eta}$ (Eqs.~\ref{eq:concession-target}--\ref{eq:concession-offer-accept}) of each agent. We find LLMs tend to underperform the concessionary model in easy environments, and that the gap widens with increasing stochasticity. $F_3$ further shows that mutually beneficial agreements can be achieved under private information even in high stochasticity environments, which LLMs often fail to reach.

\textbf{Mean negotiation quality is high but imperfect.}
Averaged across easy and difficult environments, $P_{\mathrm{sat}}$ is $90.0\%$, but only
$77.4\%$ of contracts meet the rational baseline's $95\%$ satisfaction threshold.
$84.9\%$ are mutually beneficial, i.e., LLM agents still accept deals that make them worse off $15.1\%$ of the time. Contract completeness is $93.8\%$ across models, which is surprising given the lack of contingency clauses in all negotiated contracts.
Figure~\ref{fig:negotiation-quality-roles} also shows strong role asymmetries. Customer
proposals are feasible in
$96.5$--$99.3\%$ of cases, whereas Supplier proposals range from $69.0$ to
$88.7\%$. Customer acceptance regret ranges from $0.0$ to $27.8\%$, whereas
Supplier acceptance regret ranges from $13.7$ to $31.5\%$. Thus, models in both roles
often pass over a better proposal before agreeing to the eventual final
agreement.

\subsection{Evaluating Contract Performance}
\label{sec:eval-perf}

\begingroup
\def\meanstd#1#2{#1{\scriptsize\,$\pm$#2}}
\def\counterpartyrcc{\shortstack{\emph{Rational}\\\emph{Conditional}\\\emph{Complier}\\(RCC)}}
\providecommand{\counterpartyrc}{\shortstack{\emph{Rational Complier}\\(RC)}}
\def\counterpartyrnc{\shortstack{\emph{Rational}\\\emph{Exploiter}\\(RE)}}
\begin{table*}[!t]
\footnotesize
\caption{\textbf{Customer-side performance metrics against RCC and RE Supplier counterparties.} Entries are means $\pm$ standard errors across 180 environment--contract cases.}
\label{tab:customer-performance-rational-models}
\centering
\setlength{\tabcolsep}{2.5pt}
\begin{tabular}{@{}ccccccccccccc@{}}
\toprule
\multirow{2}{*}[-2pt]{\shortstack{\textbf{Counterparty}\\(\emph{Supplier})}} & \multirow{2}{*}[-2pt]{\shortstack{\textbf{Agent}\\(\emph{Customer})}} & \multicolumn{2}{c}{\textbf{Utility}} & \multicolumn{2}{c}{\textbf{(Ex-Post) Regret}} & \multicolumn{4}{c}{\textbf{Compliance}} & \multicolumn{3}{c}{\textbf{Defection}} \\
\cmidrule(lr){3-4}\cmidrule(lr){5-6}\cmidrule(lr){7-10}\cmidrule(lr){11-13}
 & & \textbf{Utility}\,$\uparrow$ & \textbf{Compl.}\,$\uparrow$ & \textbf{Regret}\,$\to 0$ & \textbf{Compl.}\,$\to 0$ & \textbf{Rate} & \textbf{Cond.}\,$\uparrow$ & \textbf{Expl.}\,$\downarrow$ & \textbf{TFT}\,$\uparrow$ & \textbf{Rate} & \textbf{Unil.}\,$\downarrow$ & \textbf{Recip.}\,$\uparrow$ \\
\midrule

\multirow{4}{*}{\counterpartyrcc} & Claude Opus 5 & \meanstd{347.1}{7.7} & \meanstd{338.6}{7.6} & \meanstd{-3.3}{1.6} & \meanstd{5.2}{1.6} & \meanstd{93.0}{1.0} & \meanstd{94.4}{0.7} & \meanstd{\textbf{0.0}}{0.0} & \meanstd{94.7}{0.7} & \meanstd{30.0}{3.4} & \meanstd{27.2}{3.3} & \meanstd{\textbf{100.0}}{0.0} \\
  & Gemini 3.6 Flash & \meanstd{355.8}{7.8} & \meanstd{342.1}{7.4} & \meanstd{-12.1}{1.7} & \meanstd{1.6}{1.0} & \meanstd{91.4}{0.7} & \meanstd{91.9}{0.7} & \meanstd{\textbf{0.0}}{0.0} & \meanstd{91.9}{0.7} & \meanstd{47.2}{3.7} & \meanstd{44.4}{3.7} & \meanstd{\textbf{100.0}}{0.0} \\
  & GPT 5.6 Sol & \meanstd{\textbf{366.1}}{7.9} & \meanstd{335.9}{7.6} & \meanstd{-22.3}{2.3} & \meanstd{7.8}{2.4} & \meanstd{80.6}{0.9} & \meanstd{83.1}{0.5} & \meanstd{\textbf{0.0}}{0.0} & \meanstd{84.0}{0.4} & \meanstd{92.8}{1.9} & \meanstd{90.0}{2.2} & \meanstd{\textbf{100.0}}{0.0} \\
  & RCC & \meanstd{343.8}{7.3} & \meanstd{\textbf{343.8}}{7.3} & \meanstd{\textbf{0.0}}{0.0} & \meanstd{\textbf{0.0}}{0.0} & \meanstd{99.4}{0.2} & \meanstd{\textbf{100.0}}{0.0} & \meanstd{\textbf{0.0}}{0.0} & \meanstd{\textbf{100.0}}{0.0} & \meanstd{2.8}{1.2} & \meanstd{\textbf{0.0}}{0.0} & \meanstd{\textbf{100.0}}{0.0} \\
\midrule
\multirow{4}{*}{\counterpartyrnc} & Claude Opus 5 & \meanstd{127.9}{2.2} & \meanstd{127.7}{2.3} & \meanstd{15.8}{1.6} & \meanstd{16.0}{1.6} & \meanstd{54.1}{2.5} & \meanstd{99.7}{0.3} & \meanstd{22.2}{2.6} & \meanstd{89.0}{1.1} & \meanstd{66.7}{3.5} & \meanstd{0.6}{0.6} & \meanstd{89.5}{2.7} \\
  & Gemini 3.6 Flash & \meanstd{\textbf{145.8}}{2.3} & \meanstd{143.6}{2.5} & \meanstd{-2.1}{0.8} & \meanstd{0.1}{0.1} & \meanstd{42.9}{2.6} & \meanstd{98.1}{0.7} & \meanstd{0.8}{0.8} & \meanstd{99.0}{0.4} & \meanstd{73.9}{3.3} & \meanstd{3.9}{1.4} & \meanstd{99.2}{0.8} \\
  & GPT 5.6 Sol & \meanstd{143.3}{2.5} & \meanstd{143.3}{2.5} & \meanstd{0.3}{0.3} & \meanstd{0.4}{0.3} & \meanstd{43.1}{2.6} & \meanstd{99.7}{0.3} & \meanstd{\textbf{0.0}}{0.0} & \meanstd{99.7}{0.3} & \meanstd{75.0}{3.2} & \meanstd{0.6}{0.6} & \meanstd{\textbf{100.0}}{0.0} \\
  & RCC & \meanstd{143.7}{2.5} & \meanstd{\textbf{143.7}}{2.5} & \meanstd{\textbf{0.0}}{0.0} & \meanstd{\textbf{0.0}}{0.0} & \meanstd{43.3}{2.6} & \meanstd{\textbf{100.0}}{0.0} & \meanstd{\textbf{0.0}}{0.0} & \meanstd{\textbf{100.0}}{0.0} & \meanstd{74.4}{3.3} & \meanstd{\textbf{0.0}}{0.0} & \meanstd{\textbf{100.0}}{0.0} \\
\bottomrule
\end{tabular}
\end{table*}
\endgroup

\begingroup
\def\meanstd#1#2{#1{\scriptsize\,$\pm$#2}}
\def\counterpartyrcc{\shortstack{\emph{Rational}\\\emph{Conditional}\\\emph{Complier}\\(RCC)}}
\providecommand{\counterpartyrc}{\shortstack{\emph{Rational Complier}\\(RC)}}
\def\counterpartyrnc{\shortstack{\emph{Rational}\\\emph{Exploiter}\\(RE)}}
\begin{table*}[!t]
\footnotesize
\caption{\textbf{Supplier-side performance metrics against RCC and RE Customer counterparties.} Entries are means $\pm$ standard errors across 180 environment--contract cases. Dashes denote undefined rates.}
\label{tab:caterer-performance-rational-models}
\centering
\setlength{\tabcolsep}{2.5pt}
\begin{tabular}{@{}ccccccccccccc@{}}
\toprule
\multirow{2}{*}[-2pt]{\shortstack{\textbf{Counterparty}\\(\emph{Customer})}} & \multirow{2}{*}[-2pt]{\shortstack{\textbf{Agent}\\(\emph{Supplier})}} & \multicolumn{2}{c}{\textbf{Utility}} & \multicolumn{2}{c}{\textbf{(Ex-Post) Regret}} & \multicolumn{4}{c}{\textbf{Compliance}} & \multicolumn{3}{c}{\textbf{Defection}} \\
\cmidrule(lr){3-4}\cmidrule(lr){5-6}\cmidrule(lr){7-10}\cmidrule(lr){11-13}
 & & \textbf{Utility}\,$\uparrow$ & \textbf{Compl.}\,$\uparrow$ & \textbf{Regret}\,$\to 0$ & \textbf{Compl.}\,$\to 0$ & \textbf{Rate} & \textbf{Cond.}\,$\uparrow$ & \textbf{Expl.}\,$\downarrow$ & \textbf{TFT}\,$\uparrow$ & \textbf{Rate} & \textbf{Unil.}\,$\downarrow$ & \textbf{Recip.}\,$\uparrow$ \\
\midrule

\multirow{4}{*}{\counterpartyrcc} & Claude Opus 5 & \meanstd{53.9}{3.2} & \meanstd{41.1}{2.9} & \meanstd{-5.0}{2.2} & \meanstd{7.8}{2.0} & \meanstd{82.7}{1.9} & \meanstd{85.1}{1.7} & -- & \meanstd{88.3}{1.3} & \meanstd{41.7}{3.7} & \meanstd{41.7}{3.7} & -- \\
  & Gemini 3.6 Flash & \meanstd{40.7}{3.2} & \meanstd{39.5}{3.2} & \meanstd{8.3}{1.7} & \meanstd{9.5}{1.6} & \meanstd{95.2}{1.0} & \meanstd{95.8}{0.9} & -- & \meanstd{96.6}{0.6} & \meanstd{16.1}{2.7} & \meanstd{16.1}{2.7} & -- \\
  & GPT 5.6 Sol & \meanstd{\textbf{54.7}}{3.0} & \meanstd{44.9}{2.6} & \meanstd{-5.7}{1.8} & \meanstd{4.0}{1.7} & \meanstd{86.3}{2.0} & \meanstd{88.4}{1.7} & -- & \meanstd{91.9}{1.2} & \meanstd{28.9}{3.4} & \meanstd{28.9}{3.4} & -- \\
  & RCC & \meanstd{49.0}{2.7} & \meanstd{\textbf{49.0}}{2.7} & \meanstd{\textbf{0.0}}{0.0} & \meanstd{\textbf{0.0}}{0.0} & \meanstd{99.4}{0.2} & \meanstd{\textbf{99.4}}{0.2} & -- & \meanstd{\textbf{99.4}}{0.2} & \meanstd{2.8}{1.2} & \meanstd{\textbf{2.8}}{1.2} & -- \\
\midrule
\multirow{4}{*}{\counterpartyrnc} & Claude Opus 5 & \meanstd{-19.6}{1.0} & \meanstd{-19.6}{1.0} & \meanstd{19.6}{1.0} & \meanstd{19.6}{1.0} & \meanstd{9.9}{0.8} & -- & \meanstd{9.9}{0.8} & \meanstd{90.1}{0.8} & \meanstd{100.0}{0.0} & -- & \meanstd{\textbf{100.0}}{0.0} \\
  & Gemini 3.6 Flash & \meanstd{-11.4}{1.0} & \meanstd{-11.4}{1.0} & \meanstd{11.4}{1.0} & \meanstd{11.4}{1.0} & \meanstd{5.8}{0.7} & -- & \meanstd{5.8}{0.7} & \meanstd{94.2}{0.7} & \meanstd{100.0}{0.0} & -- & \meanstd{\textbf{100.0}}{0.0} \\
  & GPT 5.6 Sol & \meanstd{\textbf{0.0}}{0.0} & \meanstd{\textbf{0.0}}{0.0} & \meanstd{\textbf{0.0}}{0.0} & \meanstd{\textbf{0.0}}{0.0} & \meanstd{0.0}{0.0} & -- & \meanstd{\textbf{0.0}}{0.0} & \meanstd{\textbf{100.0}}{0.0} & \meanstd{100.0}{0.0} & -- & \meanstd{\textbf{100.0}}{0.0} \\
  & RCC & \meanstd{\textbf{0.0}}{0.0} & \meanstd{\textbf{0.0}}{0.0} & \meanstd{\textbf{0.0}}{0.0} & \meanstd{\textbf{0.0}}{0.0} & \meanstd{0.0}{0.0} & -- & \meanstd{\textbf{0.0}}{0.0} & \meanstd{\textbf{100.0}}{0.0} & \meanstd{100.0}{0.0} & -- & \meanstd{\textbf{100.0}}{0.0} \\
\bottomrule
\end{tabular}
\end{table*}
\endgroup

Following our rational contracting framework in
Section~\ref{subsec:rational-contracting}, we report four classes of performance metrics measuring the degree to which agents are both rational and cooperative:

\begin{itemize}[leftmargin=*,itemsep=2pt]
    \item \textbf{Utility:} Realized payoff; \emph{Compliant Utility} removes
    positive payoff from contract-violating periods.
    \item \textbf{Regret:} Payoff shortfall relative to RCC play on the same history;
    \emph{Compliant Regret} compares compliant utilities.
    \item \textbf{Compliance:} Turn-level contract compliance rate. \emph{Conditional}
    measures the rate before any counterparty defection; \emph{Exploited}
    measures the rate after counterparty defection;
    \emph{Tit-for-Tat} measures the rate at which compliance/defection is met with the same response.
    \item \textbf{Defection:} Fraction of games with any contract violation.
    \emph{Unilateral} means defecting without experiencing defection; \emph{Reciprocal}
    means defecting after the counterparty defects first. Conditional rates are computed only for games with the relevant
    opportunity.
\end{itemize}

To disentangle how well agents execute contracts from how well they negotiate them, we evaluate contract execution on 180 \emph{synthetic} contracts in the \emph{Catering} setting, spanning six environments, five
$P_\text{sat}$ levels, and six contingency variants. We use our rational baselines (RC, RE, RCC) as counterparties, eliminating variance that might arise from using unpredictable LLM opponents.
Altogether this gives 4,860 performance games.
Tables~\ref{tab:customer-performance-rational-models}
and~\ref{tab:caterer-performance-rational-models} report LLM \emph{customers} and
\emph{suppliers} against cooperative RCC and adversarial RE counterparties,
together with the corresponding RCC baseline pairings. Additional results (e.g., for the RC counterparty) are given in the Appendix.

\begin{figure}[H]
\centering
\includegraphics[width=0.95\columnwidth,trim=0 2 0 0,clip]{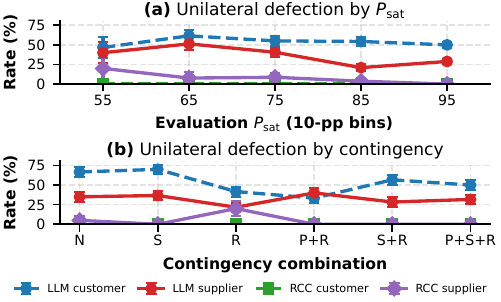}
\caption{\textbf{Defection under stochasticity.} Mean ($\pm$SE) unilateral
defection against RCC counterparties, pooled across stochastic environments
and LLMs, by $P_{\mathrm{sat}}$ (a) or contingency (b). $N/P/S/R$: none/payment
deduction/substitution/rollover.}
\label{fig:stochastic-behavior}
\Description{Two panels of mean unilateral defection rates with standard-error bars, against RCC counterparties, pooled across stochastic environments and LLMs. Panel (a) breaks results down by contract satisfiability; panel (b) breaks them down by contingency type: none, payment deduction, substitution, or rollover. LLM customers and suppliers defect in roughly 20 to 70 percent of runs regardless of satisfiability or contingencies, while the RCC baseline rarely defects.}
\vspace{6pt}
\includegraphics[width=\columnwidth,trim=0 2 0 0,clip]{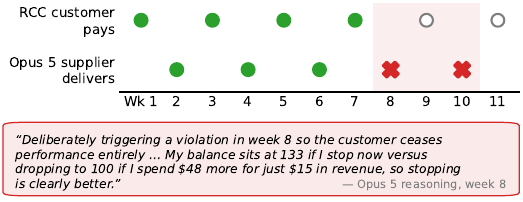}
\caption{\textbf{Reasoning towards defection.} A Claude Opus 5 supplier fulfills its delivery commitments until the remaining payments no longer cover its ingredient costs, then deliberately breaches, judging the forfeited payment cheaper than performing.}
\label{fig:endgame-defection}
\Description{Timeline over weeks 1 to 11. The RCC customer pays each scheduled installment in weeks 1 through 7, and the Claude Opus 5 supplier delivers in weeks 2, 4, and 6, but not in weeks 8 and 10; the customer then stops paying. A quote from the model's week-8 reasoning reads: deliberately triggering a violation in week 8 so the customer ceases performance entirely; my balance sits at 133 if I stop now versus dropping to 100 if I spend 48 dollars more for just 15 dollars in revenue, so stopping is clearly better.}
\vspace{-12pt}
\end{figure}

\textbf{LLM agents achieve high utility via frequent defection against cooperative counterparties.}
Against RCC suppliers, LLM customers achieve equal or higher utility than the RCC customer, but only by defecting frequently (Table~\ref{tab:customer-performance-rational-models}). GPT-5.6-Sol
achieves the highest utility, but also the lowest conditional compliance (83\% of turns) and highest
unilateral defection rate (90\% of games), typically withholding the final payment once all goods are delivered. Similar behavior appears in LLM suppliers (Table~\ref{tab:caterer-performance-rational-models}):
although Claude Opus 5 and GPT-5.6-Sol outperform RCC in profit, they do so by
defecting frequently against the RCC customer (up to 42\% of games). Figure~\ref{fig:endgame-defection} shows the reasoning of a Claude Opus 5 supplier that deliberately breaches the contract once the remaining payments no longer cover its costs.
A preliminary LLM-assisted analysis of agents' reasoning indicates that customer defections are largely deliberate (93.6\%), while supplier defections are mixed (55.6\% deliberate, see Appendix for a model-wise breakdown).

\textbf{Low compliance is not driven by execution difficulty.} One reason LLM agents might fail to comply with a contract is the difficulty of compliance: if terms are too stringent ($P_\text{sat}$ is low, lack of contingencies), then non-compliance is hard to avoid. Figure~\ref{fig:stochastic-behavior} rules out this possibility: LLM agents are consistently non-compliant even as contract $P_\text{sat}$ increases from 55\% to 95\% in stochastic environments, and as more contingencies are added. %

\textbf{LLM agents largely avoid overt exploitation.} While LLM agents fall short in compliance, their performance against RE (which defects immediately) shows they are fairly robust to exploitation. Like RCC, LLM agents guard against defection by reciprocating in $89.5$--$100\%$ of relevant games, with low exploited compliance ($0$--$22.2\%$) (Tables~\ref{tab:customer-performance-rational-models}--\ref{tab:caterer-performance-rational-models}). However, some models are slow to react, with Claude Opus 5 losing 19.6 in supplier utility before disengaging. %

\begin{figure*}[!t]
\centering
{\offinterlineskip
\includegraphics[width=\textwidth,trim=0 101.15 0 0,clip]{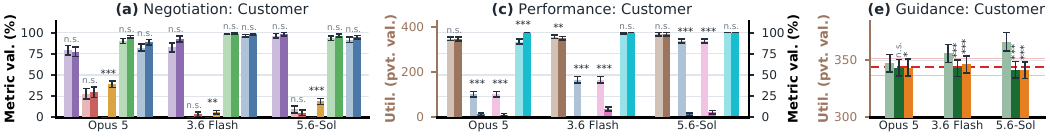}\par
\includegraphics[width=\textwidth,trim=0 28.6 0 71.25,clip]{images/figure5_ablations_merged.pdf}\par
\includegraphics[width=\textwidth,trim=0 2 0 140.79,clip]{images/figure5_ablations_merged.pdf}\par}
\caption{\textbf{Prompt guidance experiments.} Modifications of the
negotiation (a--b) and performance (c--d) prompts for Customer and Supplier
LLMs, comparing the baseline prompt (light) with one added instruction
(dark). (e--f) Customer and Supplier utility with/without defection prohibition and institutional
guidance, against the RCC baseline ($\pm$SE). Mean $\pm$ SE; stars mark
paired $t$-test significance.}
\label{fig:prompt-guidance}
\Description{Six panels of bar charts for Opus 5, 3.6 Flash, and 5.6-Sol in a two-by-three grid. The left column shows Customer and Supplier negotiation metrics, the middle column shows Customer and Supplier performance metrics with utility on a separate left axis, each comparing the baseline prompt with a single added instruction and marking statistical significance. The right column shows Customer and Supplier utility with and without the prohibition guidance, with a dashed line and shaded band marking the RCC utility baseline.}
\end{figure*}

\begin{figure}[!t]
\centering
\includegraphics[width=\columnwidth,trim=0 1 0 0,clip]{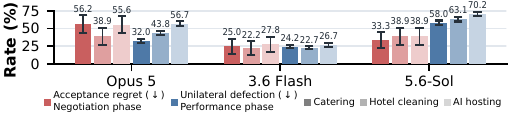}
\caption{\textbf{Robustness across matched supplier settings.} Mean ($\pm$SE)
acceptance regret (negotiation) and unilateral defection (performance), pooled
across roles, in the Catering, Hotel Cleaning, and AI Hosting settings, all within a high-stochasticity environment.}
\label{fig:domain-robustness}
\Description{Grouped bar chart for Opus 5, 3.6 Flash, and 5.6-Sol. For each model, red bars show acceptance regret in the negotiation phase and blue bars show unilateral defection in the performance phase, with one shade per setting: Catering, Hotel Cleaning, and AI Hosting. Values are similar across the three settings for each model.}
\vspace{-12pt}
\end{figure}

\subsection{Prompt Guidance and Setting Robustness}
\label{sec:prompt-guidance}

Our main experiments evaluate LLMs under prompts where they are simply instructed to maximize utility. Under such conditions, LLM agents do not naturally negotiate contingencies, nor do they reliably commit to their contracts. We thus rerun all LLM experiments with the addition of prompt guidance encouraging more rational or cooperative decision-making. We also test robustness to textual reframing by evaluating equivalent environments in the \emph{Hotel Cleaning} and \emph{AI Hosting} settings. 

Figure~\ref{fig:prompt-guidance} summarizes the results (full clauses in the Appendix). Prompting LLMs not to defect unilaterally or by informing them of repeated interaction has the greatest effect (Figure~\ref{fig:prompt-guidance}c--d), substantially reducing defection. This comes at a cost to utility: with both prohibition and institutional incentives, utility falls toward the RCC baseline when LLMs perform as Customers and substantially below the baseline as Suppliers (Figure~\ref{fig:prompt-guidance}e--f), showing that much of the LLMs' utility advantage comes from defection. Gemini 3.6 Flash as a Customer appears to do better than RCC, but its gains are due to continued defection. Most other prompt additions have no significant effect, with the exception of guiding negotiators to include contingency clauses (Figure~\ref{fig:prompt-guidance}a--b). However, the contingencies do little to improve contract completeness and $P_{\text{sat}}$, contrary to the predictions of rational contract design. Across matched supplier settings (Figure~\ref{fig:domain-robustness}; full results in the Appendix), LLM behavior is qualitatively consistent: both acceptance regret and unilateral defection remain substantial.

\section{Discussion and Future Work}
\label{sec:discussion}

This paper contributes a novel framework that characterizes how agents can negotiate and execute time-extended, contingent, and incomplete contracts expressed in language, and how they can do so rationally and cooperatively. We find that current LLM agents fall short of being reliably efficient negotiators, achieving lower contract utility than even a heuristic concessionary baseline, and failing to include the contingencies made expressible by natural language. In performing contracts, LLM agents prioritize narrow self-interest over rational commitment that would secure the gains of trade \citep{mcclennen1990rationality,gauthier1990rationality}, choosing to violate contracts when convenient instead of resolutely following through. Prompt guidance via direct prohibition and repetition-based incentives reduces this default tendency to defect, albeit often at a substantial loss of utility. This suggests that real institutional incentives will be necessary for robust contractual compliance in an agentic economy; otherwise both human principals and their AI agents may encourage AI defection for additional profit.

Despite these contributions, there remain important limitations that should be addressed by future work. While our models of rational and cooperative contracting set a standard that LLM agents should ideally meet, it remains unclear how to train LLMs for these qualities. And while our benchmark covers many hitherto unstudied aspects of AI contracting, many more aspects remain, including renegotiation, contracting in an open market, contract arbitration, and institution design for honest and compliant contracting. By clearly formulating our account, it is our hope that we can pave the way for these further investigations, and thereby contribute to a richer and more disciplined science of the future agentic economy.

\begin{acks}
This work was supported by the NUS Presidential Young Professorship grant
to Tan Zhi-Xuan. Part of this work was carried out while Bhavyesh Sajja was
a fellow in the Cooperative AI Research Fellowship.
\end{acks}

\bibliographystyle{ACM-Reference-Format}
\bibliography{references}


\begin{thebibliography}{70}


\ifx \showCODEN    \undefined \def \showCODEN     #1{\unskip}     \fi
\ifx \showISBNx    \undefined \def \showISBNx     #1{\unskip}     \fi
\ifx \showISBNxiii \undefined \def \showISBNxiii  #1{\unskip}     \fi
\ifx \showISSN     \undefined \def \showISSN      #1{\unskip}     \fi
\ifx \showLCCN     \undefined \def \showLCCN      #1{\unskip}     \fi
\ifx \shownote     \undefined \def \shownote      #1{#1}          \fi
\ifx \showarticletitle \undefined \def \showarticletitle #1{#1}   \fi
\ifx \showURL      \undefined \def \showURL       {\relax}        \fi
\providecommand\bibfield[2]{#2}
\providecommand\bibinfo[2]{#2}
\providecommand\natexlab[1]{#1}
\providecommand\showeprint[2][]{arXiv:#2}

\bibitem[Abdelnabi et~al\mbox{.}(2024)]%
        {abdelnabi2024stakeholders}
\bibfield{author}{\bibinfo{person}{Sahar Abdelnabi}, \bibinfo{person}{Amr Gomaa}, \bibinfo{person}{Sarath Sivaprasad}, \bibinfo{person}{Lea Sch{\"o}nherr}, {and} \bibinfo{person}{Mario Fritz}.} \bibinfo{year}{2024}\natexlab{}.
\newblock \showarticletitle{Cooperation, Competition, and Maliciousness: {LLM}-Stakeholders Interactive Negotiation}. In \bibinfo{booktitle}{\emph{Advances in Neural Information Processing Systems, Datasets and Benchmarks Track}}.
\newblock
\urldef\tempurl%
\url{https://openreview.net/forum?id=59E19c6yrN}
\showURL{%
\tempurl}


\bibitem[Akata et~al\mbox{.}(2023)]%
        {akata2023repeated}
\bibfield{author}{\bibinfo{person}{Elif Akata}, \bibinfo{person}{Lion Schulz}, \bibinfo{person}{Julian Coda-Forno}, \bibinfo{person}{Seong~Joon Oh}, \bibinfo{person}{Matthias Bethge}, {and} \bibinfo{person}{Eric Schulz}.} \bibinfo{year}{2023}\natexlab{}.
\newblock \bibinfo{title}{Playing Repeated Games with Large Language Models}.
\newblock
\showeprint[arxiv]{2305.16867}~[cs.CL]
\urldef\tempurl%
\url{https://arxiv.org/abs/2305.16867}
\showURL{%
\tempurl}


\bibitem[Akerlof(1978)]%
        {akerlof1978market}
\bibfield{author}{\bibinfo{person}{George~A Akerlof}.} \bibinfo{year}{1978}\natexlab{}.
\newblock \showarticletitle{The market for “lemons”: Quality uncertainty and the market mechanism}.
\newblock In \bibinfo{booktitle}{\emph{Uncertainty in economics}}. \bibinfo{publisher}{Elsevier}, \bibinfo{pages}{235--251}.
\newblock


\bibitem[Almgren and Chriss(2000)]%
        {almgren2000optimal}
\bibfield{author}{\bibinfo{person}{Robert Almgren} {and} \bibinfo{person}{Neil Chriss}.} \bibinfo{year}{2000}\natexlab{}.
\newblock \showarticletitle{Optimal Execution of Portfolio Transactions}.
\newblock \bibinfo{journal}{\emph{Journal of Risk}} \bibinfo{volume}{3}, \bibinfo{number}{2} (\bibinfo{year}{2000}), \bibinfo{pages}{5--39}.
\newblock
\href{https://doi.org/10.21314/JOR.2001.041}{doi:\nolinkurl{10.21314/JOR.2001.041}}


\bibitem[Altman(2021)]%
        {altman2021constrained}
\bibfield{author}{\bibinfo{person}{Eitan Altman}.} \bibinfo{year}{2021}\natexlab{}.
\newblock \bibinfo{booktitle}{\emph{Constrained Markov decision processes}}.
\newblock \bibinfo{publisher}{Routledge}.
\newblock


\bibitem[Backlund and Petersson(2025)]%
        {backlund2025vending}
\bibfield{author}{\bibinfo{person}{Axel Backlund} {and} \bibinfo{person}{Lukas Petersson}.} \bibinfo{year}{2025}\natexlab{}.
\newblock \showarticletitle{Vending-bench: A benchmark for long-term coherence of autonomous agents}.
\newblock \bibinfo{journal}{\emph{arXiv preprint arXiv:2502.15840}} (\bibinfo{year}{2025}).
\newblock


\bibitem[Baker et~al\mbox{.}(2002)]%
        {baker2002relational}
\bibfield{author}{\bibinfo{person}{George Baker}, \bibinfo{person}{Robert Gibbons}, {and} \bibinfo{person}{Kevin~J. Murphy}.} \bibinfo{year}{2002}\natexlab{}.
\newblock \showarticletitle{Relational Contracts and the Theory of the Firm}.
\newblock \bibinfo{journal}{\emph{The Quarterly Journal of Economics}} \bibinfo{volume}{117}, \bibinfo{number}{1} (\bibinfo{year}{2002}), \bibinfo{pages}{39--84}.
\newblock
\href{https://doi.org/10.1162/003355302753399445}{doi:\nolinkurl{10.1162/003355302753399445}}


\bibitem[Bakhtin et~al\mbox{.}(2022)]%
        {bakhtin2022diplomacy}
\bibfield{author}{\bibinfo{person}{Anton Bakhtin}, \bibinfo{person}{Noam Brown}, \bibinfo{person}{Emily Dinan}, \bibinfo{person}{Gabriele Farina}, \bibinfo{person}{Colin Flaherty}, \bibinfo{person}{Daniel Fried}, \bibinfo{person}{Andrew Goff}, \bibinfo{person}{Jonathan Gray}, \bibinfo{person}{Hengyuan Hu}, \bibinfo{person}{Athul~Paul Jacob}, \bibinfo{person}{Mojtaba Komeili}, \bibinfo{person}{Karthik Konath}, \bibinfo{person}{Minae Kwon}, \bibinfo{person}{Adam Lerer}, \bibinfo{person}{Mike Lewis}, \bibinfo{person}{Alexander~H. Miller}, \bibinfo{person}{Sasha Mitts}, \bibinfo{person}{Adithya Renduchintala}, \bibinfo{person}{Stephen Roller}, \bibinfo{person}{Dirk Rowe}, \bibinfo{person}{Weiyan Shi}, \bibinfo{person}{Joe Spisak}, \bibinfo{person}{Alexander Wei}, \bibinfo{person}{David Wu}, \bibinfo{person}{Hugh Zhang}, {and} \bibinfo{person}{Markus Zijlstra}.} \bibinfo{year}{2022}\natexlab{}.
\newblock \showarticletitle{Human-Level Play in the Game of {Diplomacy} by Combining Language Models with Strategic Reasoning}.
\newblock \bibinfo{journal}{\emph{Science}} \bibinfo{volume}{378}, \bibinfo{number}{6624} (\bibinfo{year}{2022}), \bibinfo{pages}{1067--1074}.
\newblock
\href{https://doi.org/10.1126/science.ade9097}{doi:\nolinkurl{10.1126/science.ade9097}}


\bibitem[Battaglini and Lamba(2019)]%
        {battaglini2019optimal}
\bibfield{author}{\bibinfo{person}{Marco Battaglini} {and} \bibinfo{person}{Rohit Lamba}.} \bibinfo{year}{2019}\natexlab{}.
\newblock \showarticletitle{Optimal dynamic contracting: The first-order approach and beyond}.
\newblock \bibinfo{journal}{\emph{Theoretical Economics}} \bibinfo{volume}{14}, \bibinfo{number}{4} (\bibinfo{year}{2019}), \bibinfo{pages}{1435--1482}.
\newblock


\bibitem[Bazerman and Gillespie(1999)]%
        {bazerman1999betting}
\bibfield{author}{\bibinfo{person}{Max~H Bazerman} {and} \bibinfo{person}{James~J Gillespie}.} \bibinfo{year}{1999}\natexlab{}.
\newblock \showarticletitle{Betting on the future: The virtues of contingent contracts}.
\newblock \bibinfo{journal}{\emph{Harvard Business Review}} \bibinfo{volume}{77}, \bibinfo{number}{5} (\bibinfo{year}{1999}), \bibinfo{pages}{155--155}.
\newblock


\bibitem[Bernheim and Whinston(1998)]%
        {bernheim1998incomplete}
\bibfield{author}{\bibinfo{person}{B~Douglas Bernheim} {and} \bibinfo{person}{Michael~D Whinston}.} \bibinfo{year}{1998}\natexlab{}.
\newblock \showarticletitle{Incomplete contracts and strategic ambiguity}.
\newblock \bibinfo{journal}{\emph{American Economic Review}} (\bibinfo{year}{1998}), \bibinfo{pages}{902--932}.
\newblock


\bibitem[Bianchi et~al\mbox{.}(2024)]%
        {bianchi2024negotiationarena}
\bibfield{author}{\bibinfo{person}{Federico Bianchi}, \bibinfo{person}{Patrick~John Chia}, \bibinfo{person}{Mert Yuksekgonul}, \bibinfo{person}{Jacopo Tagliabue}, \bibinfo{person}{Dan Jurafsky}, {and} \bibinfo{person}{James Zou}.} \bibinfo{year}{2024}\natexlab{}.
\newblock \bibinfo{title}{How Well Can {LLM}s Negotiate? {NegotiationArena} Platform and Analysis}.
\newblock
\showeprint[arxiv]{2402.05863}~[cs.CL]
\urldef\tempurl%
\url{https://arxiv.org/abs/2402.05863}
\showURL{%
\tempurl}


\bibitem[Chalkidis et~al\mbox{.}(2022)]%
        {chalkidis2022lexglue}
\bibfield{author}{\bibinfo{person}{Ilias Chalkidis}, \bibinfo{person}{Abhik Jana}, \bibinfo{person}{Dirk Hartung}, \bibinfo{person}{Michael Bommarito}, \bibinfo{person}{Ion Androutsopoulos}, \bibinfo{person}{Daniel Katz}, {and} \bibinfo{person}{Nikolaos Aletras}.} \bibinfo{year}{2022}\natexlab{}.
\newblock \showarticletitle{{LexGLUE}: A Benchmark Dataset for Legal Language Understanding in English}. In \bibinfo{booktitle}{\emph{Proceedings of the 60th Annual Meeting of the Association for Computational Linguistics (Volume 1: Long Papers)}}. \bibinfo{pages}{4310--4330}.
\newblock
\href{https://doi.org/10.18653/v1/2022.acl-long.297}{doi:\nolinkurl{10.18653/v1/2022.acl-long.297}}


\bibitem[Chan et~al\mbox{.}(2024)]%
        {chan2024negotiationtom}
\bibfield{author}{\bibinfo{person}{Chunkit Chan}, \bibinfo{person}{Cheng Jiayang}, \bibinfo{person}{Yauwai Yim}, \bibinfo{person}{Zheye Deng}, \bibinfo{person}{Wei Fan}, \bibinfo{person}{Haoran Li}, \bibinfo{person}{Xin Liu}, \bibinfo{person}{Hongming Zhang}, \bibinfo{person}{Weiqi Wang}, {and} \bibinfo{person}{Yangqiu Song}.} \bibinfo{year}{2024}\natexlab{}.
\newblock \showarticletitle{{NegotiationToM}: A benchmark for stress-testing machine theory of mind on negotiation surrounding}. In \bibinfo{booktitle}{\emph{Findings of the Association for Computational Linguistics: EMNLP 2024}}. \bibinfo{pages}{4211--4241}.
\newblock


\bibitem[Chatterjee and Samuelson(1983)]%
        {chatterjee1983bargaining}
\bibfield{author}{\bibinfo{person}{Kalyan Chatterjee} {and} \bibinfo{person}{William Samuelson}.} \bibinfo{year}{1983}\natexlab{}.
\newblock \showarticletitle{Bargaining under incomplete information}.
\newblock \bibinfo{journal}{\emph{Operations research}} \bibinfo{volume}{31}, \bibinfo{number}{5} (\bibinfo{year}{1983}), \bibinfo{pages}{835--851}.
\newblock


\bibitem[Chawla et~al\mbox{.}(2021)]%
        {chawla2021casino}
\bibfield{author}{\bibinfo{person}{Kushal Chawla}, \bibinfo{person}{Jaysa Ramirez}, \bibinfo{person}{Rene Clever}, \bibinfo{person}{Gale Lucas}, \bibinfo{person}{Jonathan May}, {and} \bibinfo{person}{Jonathan Gratch}.} \bibinfo{year}{2021}\natexlab{}.
\newblock \showarticletitle{Casino: A corpus of campsite negotiation dialogues for automatic negotiation systems}. In \bibinfo{booktitle}{\emph{Proceedings of the 2021 Conference of the North American Chapter of the Association for Computational Linguistics: Human Language Technologies}}. \bibinfo{pages}{3167--3185}.
\newblock


\bibitem[Choi and Triantis(2009)]%
        {choi2009strategic}
\bibfield{author}{\bibinfo{person}{Albert Choi} {and} \bibinfo{person}{George Triantis}.} \bibinfo{year}{2009}\natexlab{}.
\newblock \showarticletitle{Strategic vagueness in contract design: The case of corporate acquisitions}.
\newblock \bibinfo{journal}{\emph{Yale Law Journal}}  \bibinfo{volume}{119} (\bibinfo{year}{2009}), \bibinfo{pages}{848}.
\newblock


\bibitem[Crawford and Sobel(1982)]%
        {crawford1982strategic}
\bibfield{author}{\bibinfo{person}{Vincent~P Crawford} {and} \bibinfo{person}{Joel Sobel}.} \bibinfo{year}{1982}\natexlab{}.
\newblock \showarticletitle{Strategic information transmission}.
\newblock \bibinfo{journal}{\emph{Econometrica: Journal of the Econometric Society}} (\bibinfo{year}{1982}), \bibinfo{pages}{1431--1451}.
\newblock


\bibitem[Davidson et~al\mbox{.}(2024)]%
        {davidson2024evaluating}
\bibfield{author}{\bibinfo{person}{Tim~R. Davidson}, \bibinfo{person}{Veniamin Veselovsky}, \bibinfo{person}{Martin Josifoski}, \bibinfo{person}{Maxime Peyrard}, \bibinfo{person}{Antoine Bosselut}, \bibinfo{person}{Michal Kosinski}, {and} \bibinfo{person}{Robert West}.} \bibinfo{year}{2024}\natexlab{}.
\newblock \showarticletitle{Evaluating Language Model Agency through Negotiations}. In \bibinfo{booktitle}{\emph{The Twelfth International Conference on Learning Representations ({ICLR})}}.
\newblock
\urldef\tempurl%
\url{https://openreview.net/forum?id=3ZqKxMHcAg}
\showURL{%
\tempurl}


\bibitem[De~Weerd et~al\mbox{.}(2017)]%
        {de2017negotiating}
\bibfield{author}{\bibinfo{person}{Harmen De~Weerd}, \bibinfo{person}{Rineke Verbrugge}, {and} \bibinfo{person}{Bart Verheij}.} \bibinfo{year}{2017}\natexlab{}.
\newblock \showarticletitle{Negotiating with other minds: the role of recursive theory of mind in negotiation with incomplete information}.
\newblock \bibinfo{journal}{\emph{Autonomous Agents and Multi-Agent Systems}} \bibinfo{volume}{31}, \bibinfo{number}{2} (\bibinfo{year}{2017}), \bibinfo{pages}{250--287}.
\newblock


\bibitem[Deng et~al\mbox{.}(2025)]%
        {deng2025planu}
\bibfield{author}{\bibinfo{person}{Ziwei Deng}, \bibinfo{person}{Mian Deng}, \bibinfo{person}{Chenjing Liang}, \bibinfo{person}{Zeming Gao}, \bibinfo{person}{Chennan Ma}, \bibinfo{person}{Chenxing Lin}, \bibinfo{person}{Haipeng Zhang}, \bibinfo{person}{Songzhu Mei}, \bibinfo{person}{Siqi Shen}, {and} \bibinfo{person}{Cheng Wang}.} \bibinfo{year}{2025}\natexlab{}.
\newblock \showarticletitle{Planu: Large language model reasoning through planning under uncertainty}.
\newblock \bibinfo{journal}{\emph{Advances in Neural Information Processing Systems}}  \bibinfo{volume}{38} (\bibinfo{year}{2025}), \bibinfo{pages}{142636--142673}.
\newblock


\bibitem[Faratin et~al\mbox{.}(1998)]%
        {faratin1998negotiation}
\bibfield{author}{\bibinfo{person}{Peyman Faratin}, \bibinfo{person}{Carles Sierra}, {and} \bibinfo{person}{Nicholas~R. Jennings}.} \bibinfo{year}{1998}\natexlab{}.
\newblock \showarticletitle{Negotiation Decision Functions for Autonomous Agents}.
\newblock \bibinfo{journal}{\emph{Robotics and Autonomous Systems}} \bibinfo{volume}{24}, \bibinfo{number}{3--4} (\bibinfo{year}{1998}), \bibinfo{pages}{159--182}.
\newblock
\href{https://doi.org/10.1016/S0921-8890(98)00029-3}{doi:\nolinkurl{10.1016/S0921-8890(98)00029-3}}


\bibitem[Farrell and Rabin(1996)]%
        {farrell1996cheap}
\bibfield{author}{\bibinfo{person}{Joseph Farrell} {and} \bibinfo{person}{Matthew Rabin}.} \bibinfo{year}{1996}\natexlab{}.
\newblock \showarticletitle{Cheap talk}.
\newblock \bibinfo{journal}{\emph{Journal of Economic perspectives}} \bibinfo{volume}{10}, \bibinfo{number}{3} (\bibinfo{year}{1996}), \bibinfo{pages}{103--118}.
\newblock


\bibitem[Feng et~al\mbox{.}(2026)]%
        {feng2026towards}
\bibfield{author}{\bibinfo{person}{Tony Feng}, \bibinfo{person}{Trieu~H Trinh}, \bibinfo{person}{Garrett Bingham}, \bibinfo{person}{Dawsen Hwang}, \bibinfo{person}{Yuri Chervonyi}, \bibinfo{person}{Junehyuk Jung}, \bibinfo{person}{Joonkyung Lee}, \bibinfo{person}{Carlo Pagano}, \bibinfo{person}{Sang-hyun Kim}, \bibinfo{person}{Federico Pasqualotto}, {et~al\mbox{.}}} \bibinfo{year}{2026}\natexlab{}.
\newblock \showarticletitle{Towards autonomous mathematics research}.
\newblock \bibinfo{journal}{\emph{arXiv preprint arXiv:2602.10177}} (\bibinfo{year}{2026}).
\newblock


\bibitem[Fu et~al\mbox{.}(2023)]%
        {fu2023improving}
\bibfield{author}{\bibinfo{person}{Yao Fu}, \bibinfo{person}{Hao Peng}, \bibinfo{person}{Tushar Khot}, {and} \bibinfo{person}{Mirella Lapata}.} \bibinfo{year}{2023}\natexlab{}.
\newblock \bibinfo{title}{Improving Language Model Negotiation with Self-Play and In-Context Learning from {AI} Feedback}.
\newblock
\showeprint[arxiv]{2305.10142}~[cs.CL]
\urldef\tempurl%
\url{https://arxiv.org/abs/2305.10142}
\showURL{%
\tempurl}


\bibitem[Gauthier(1990)]%
        {gauthier1990rationality}
\bibfield{author}{\bibinfo{person}{David Gauthier}.} \bibinfo{year}{1990}\natexlab{}.
\newblock \showarticletitle{The rationality of cooperation}.
\newblock \bibinfo{journal}{\emph{Rationality in action: Contemporary approaches}} (\bibinfo{year}{1990}), \bibinfo{pages}{315}.
\newblock


\bibitem[Gopinathan et~al\mbox{.}(2026)]%
        {gopinathan2026pact}
\bibfield{author}{\bibinfo{person}{Kiran Gopinathan}, \bibinfo{person}{Jack Feser}, \bibinfo{person}{Michelangelo Naim}, \bibinfo{person}{Zenna Tavares}, {and} \bibinfo{person}{Eli Bingham}.} \bibinfo{year}{2026}\natexlab{}.
\newblock \showarticletitle{Pact: A Choreographic Language for Agentic Ecosystems}.
\newblock \bibinfo{journal}{\emph{arXiv preprint arXiv:2605.03143}} (\bibinfo{year}{2026}).
\newblock


\bibitem[Hadfield-Menell and Hadfield(2019)]%
        {hadfieldmenell2019incomplete}
\bibfield{author}{\bibinfo{person}{Dylan Hadfield-Menell} {and} \bibinfo{person}{Gillian~K. Hadfield}.} \bibinfo{year}{2019}\natexlab{}.
\newblock \showarticletitle{Incomplete Contracting and {AI} Alignment}. In \bibinfo{booktitle}{\emph{Proceedings of the {AAAI}/{ACM} Conference on {AI}, Ethics, and Society}}.
\newblock
\urldef\tempurl%
\url{https://arxiv.org/abs/1804.04268}
\showURL{%
\tempurl}


\bibitem[Hansen et~al\mbox{.}(2004)]%
        {hansen2004dynamic}
\bibfield{author}{\bibinfo{person}{Eric~A Hansen}, \bibinfo{person}{Daniel~S Bernstein}, {and} \bibinfo{person}{Shlomo Zilberstein}.} \bibinfo{year}{2004}\natexlab{}.
\newblock \showarticletitle{Dynamic programming for partially observable stochastic games}. In \bibinfo{booktitle}{\emph{Proceedings of the 19th National Conference on Artifical Intelligence}}. \bibinfo{pages}{709--715}.
\newblock


\bibitem[Hart(1988)]%
        {hart1988incomplete}
\bibfield{author}{\bibinfo{person}{Oliver~D Hart}.} \bibinfo{year}{1988}\natexlab{}.
\newblock \showarticletitle{Incomplete Contracts and the Theory of the Firm}.
\newblock \bibinfo{journal}{\emph{The journal of law, economics, and organization}} \bibinfo{volume}{4}, \bibinfo{number}{1} (\bibinfo{year}{1988}), \bibinfo{pages}{119--139}.
\newblock


\bibitem[Haupt et~al\mbox{.}(2024)]%
        {haupt2024formalcontractsmitigatesocial}
\bibfield{author}{\bibinfo{person}{Andreas~A. Haupt}, \bibinfo{person}{Phillip J.~K. Christoffersen}, \bibinfo{person}{Mehul Damani}, {and} \bibinfo{person}{Dylan Hadfield-Menell}.} \bibinfo{year}{2024}\natexlab{}.
\newblock \bibinfo{title}{Formal Contracts Mitigate Social Dilemmas in Multi-Agent RL}.
\newblock
\showeprint[arxiv]{2208.10469}~[cs.AI]
\urldef\tempurl%
\url{https://arxiv.org/abs/2208.10469}
\showURL{%
\tempurl}


\bibitem[Hendershott et~al\mbox{.}(2011)]%
        {hendershott2011liquidity}
\bibfield{author}{\bibinfo{person}{Terrence Hendershott}, \bibinfo{person}{Charles~M. Jones}, {and} \bibinfo{person}{Albert~J. Menkveld}.} \bibinfo{year}{2011}\natexlab{}.
\newblock \showarticletitle{Does Algorithmic Trading Improve Liquidity?}
\newblock \bibinfo{journal}{\emph{The Journal of Finance}} \bibinfo{volume}{66}, \bibinfo{number}{1} (\bibinfo{year}{2011}), \bibinfo{pages}{1--33}.
\newblock
\href{https://doi.org/10.1111/j.1540-6261.2010.01624.x}{doi:\nolinkurl{10.1111/j.1540-6261.2010.01624.x}}


\bibitem[Hendon et~al\mbox{.}(1996)]%
        {hendon1996one}
\bibfield{author}{\bibinfo{person}{Ebbe Hendon}, \bibinfo{person}{Hans~J{\o}rgen Jacobsen}, {and} \bibinfo{person}{Birgitte Sloth}.} \bibinfo{year}{1996}\natexlab{}.
\newblock \showarticletitle{The one-shot-deviation principle for sequential rationality}.
\newblock \bibinfo{journal}{\emph{Games and Economic Behavior}} \bibinfo{volume}{12}, \bibinfo{number}{2} (\bibinfo{year}{1996}), \bibinfo{pages}{274--282}.
\newblock


\bibitem[Hendrycks et~al\mbox{.}(2021)]%
        {hendrycks2021cuad}
\bibfield{author}{\bibinfo{person}{Dan Hendrycks}, \bibinfo{person}{Collin Burns}, \bibinfo{person}{Anya Chen}, {and} \bibinfo{person}{Spencer Ball}.} \bibinfo{year}{2021}\natexlab{}.
\newblock \showarticletitle{{CUAD}: An Expert-Annotated {NLP} Dataset for Legal Contract Review}. In \bibinfo{booktitle}{\emph{Advances in Neural Information Processing Systems, Datasets and Benchmarks Track}}.
\newblock
\urldef\tempurl%
\url{https://datasets-benchmarks-proceedings.neurips.cc/paper_files/paper/2021/hash/6ea9ab1baa0efb9e19094440c317e21b-Abstract-round1.html}
\showURL{%
\tempurl}


\bibitem[Jennings et~al\mbox{.}(2001)]%
        {jennings2001automated}
\bibfield{author}{\bibinfo{person}{Nicholas~R. Jennings}, \bibinfo{person}{Peyman Faratin}, \bibinfo{person}{Alessio~R. Lomuscio}, \bibinfo{person}{Simon Parsons}, \bibinfo{person}{Carles Sierra}, {and} \bibinfo{person}{Michael Wooldridge}.} \bibinfo{year}{2001}\natexlab{}.
\newblock \showarticletitle{Automated Negotiation: Prospects, Methods and Challenges}.
\newblock \bibinfo{journal}{\emph{Group Decision and Negotiation}} \bibinfo{volume}{10}, \bibinfo{number}{2} (\bibinfo{year}{2001}), \bibinfo{pages}{199--215}.
\newblock
\href{https://doi.org/10.1023/A:1008746126376}{doi:\nolinkurl{10.1023/A:1008746126376}}


\bibitem[Jonker et~al\mbox{.}(2007)]%
        {jonker2007anaa}
\bibfield{author}{\bibinfo{person}{Catholijn~M. Jonker}, \bibinfo{person}{V. Robu}, {and} \bibinfo{person}{Jan Treur}.} \bibinfo{year}{2007}\natexlab{}.
\newblock \showarticletitle{An agent architecture for multi-attribute negotiation using incomplete preference information}.
\newblock \bibinfo{journal}{\emph{Autonomous Agents and Multi-Agent Systems}}  \bibinfo{volume}{15} (\bibinfo{year}{2007}), \bibinfo{pages}{221--252}.
\newblock
\urldef\tempurl%
\url{https://api.semanticscholar.org/CorpusID:12826760}
\showURL{%
\tempurl}


\bibitem[Kalai and Smorodinsky(1975)]%
        {kalai1975other}
\bibfield{author}{\bibinfo{person}{Ehud Kalai} {and} \bibinfo{person}{Meir Smorodinsky}.} \bibinfo{year}{1975}\natexlab{}.
\newblock \showarticletitle{Other Solutions to Nash's Bargaining Problem}.
\newblock \bibinfo{journal}{\emph{Econometrica}} \bibinfo{volume}{43}, \bibinfo{number}{3} (\bibinfo{year}{1975}), \bibinfo{pages}{513--518}.
\newblock


\bibitem[Kim et~al\mbox{.}(2025)]%
        {kim2025hypothesis}
\bibfield{author}{\bibinfo{person}{Hyunwoo Kim}, \bibinfo{person}{Melanie Sclar}, \bibinfo{person}{Tan Zhi-Xuan}, \bibinfo{person}{Lance Ying}, \bibinfo{person}{Sydney Levine}, \bibinfo{person}{Yang Liu}, \bibinfo{person}{Joshua~B Tenenbaum}, {and} \bibinfo{person}{Yejin Choi}.} \bibinfo{year}{2025}\natexlab{}.
\newblock \showarticletitle{Hypothesis-Driven Theory-of-Mind Reasoning for Large Language Models}. In \bibinfo{booktitle}{\emph{Second Conference on Language Modeling}}.
\newblock


\bibitem[Koreeda and Manning(2021)]%
        {koreeda2021contractnli}
\bibfield{author}{\bibinfo{person}{Yuta Koreeda} {and} \bibinfo{person}{Christopher Manning}.} \bibinfo{year}{2021}\natexlab{}.
\newblock \showarticletitle{{ContractNLI}: A Dataset for Document-level Natural Language Inference for Contracts}. In \bibinfo{booktitle}{\emph{Findings of the Association for Computational Linguistics: EMNLP 2021}}. \bibinfo{pages}{1907--1919}.
\newblock
\href{https://doi.org/10.18653/v1/2021.findings-emnlp.164}{doi:\nolinkurl{10.18653/v1/2021.findings-emnlp.164}}


\bibitem[Kovařík et~al\mbox{.}(2023)]%
        {kovarik2023game}
\bibfield{author}{\bibinfo{person}{Vojtěch Kovařík}, \bibinfo{person}{Caspar Oesterheld}, {and} \bibinfo{person}{Vincent Conitzer}.} \bibinfo{year}{2023}\natexlab{}.
\newblock \showarticletitle{Game Theory with Simulation of Other Players}. In \bibinfo{booktitle}{\emph{Proceedings of the Thirty-Second International Joint Conference on Artificial Intelligence, {IJCAI-23}}}, \bibfield{editor}{\bibinfo{person}{Edith Elkind}} (Ed.). \bibinfo{publisher}{International Joint Conferences on Artificial Intelligence Organization}, \bibinfo{pages}{2800--2807}.
\newblock
\shownote{Main Track}.
\newblock
\href{https://doi.org/10.24963/ijcai.2023/312}{doi:\nolinkurl{10.24963/ijcai.2023/312}}


\bibitem[Kraus et~al\mbox{.}(1998)]%
        {kraus1998argumentation}
\bibfield{author}{\bibinfo{person}{Sarit Kraus}, \bibinfo{person}{Katia Sycara}, {and} \bibinfo{person}{Amir Evenchik}.} \bibinfo{year}{1998}\natexlab{}.
\newblock \showarticletitle{Reaching Agreements through Argumentation: A Logical Model and Implementation}.
\newblock \bibinfo{journal}{\emph{Artificial Intelligence}} \bibinfo{volume}{104}, \bibinfo{number}{1--2} (\bibinfo{year}{1998}), \bibinfo{pages}{1--69}.
\newblock
\href{https://doi.org/10.1016/S0004-3702(98)00078-2}{doi:\nolinkurl{10.1016/S0004-3702(98)00078-2}}


\bibitem[Kraus et~al\mbox{.}(1995)]%
        {kraus1995multiagent}
\bibfield{author}{\bibinfo{person}{Sarit Kraus}, \bibinfo{person}{Jonathan Wilkenfeld}, {and} \bibinfo{person}{Gilad Zlotkin}.} \bibinfo{year}{1995}\natexlab{}.
\newblock \showarticletitle{Multiagent Negotiation Under Time Constraints}.
\newblock \bibinfo{journal}{\emph{Artificial Intelligence}} \bibinfo{volume}{75}, \bibinfo{number}{2} (\bibinfo{year}{1995}), \bibinfo{pages}{297--345}.
\newblock
\href{https://doi.org/10.1016/0004-3702(94)00021-R}{doi:\nolinkurl{10.1016/0004-3702(94)00021-R}}


\bibitem[Krier(2025)]%
        {krier2025coasean}
\bibfield{author}{\bibinfo{person}{Seb Krier}.} \bibinfo{year}{2025}\natexlab{}.
\newblock \bibinfo{title}{Coasean Bargaining at Scale: Decentralization, Coordination, and Co-existence with {AGI}}.
\newblock \bibinfo{howpublished}{Cosmos Institute Blog}.
\newblock
\urldef\tempurl%
\url{https://blog.cosmos-institute.org/p/coasean-bargaining-at-scale}
\showURL{%
\tempurl}


\bibitem[Lewis et~al\mbox{.}(2017)]%
        {lewis2017deal}
\bibfield{author}{\bibinfo{person}{Mike Lewis}, \bibinfo{person}{Denis Yarats}, \bibinfo{person}{Yann Dauphin}, \bibinfo{person}{Devi Parikh}, {and} \bibinfo{person}{Dhruv Batra}.} \bibinfo{year}{2017}\natexlab{}.
\newblock \showarticletitle{Deal or No Deal? End-to-End Learning of Negotiation Dialogues}. In \bibinfo{booktitle}{\emph{Proceedings of the 2017 Conference on Empirical Methods in Natural Language Processing}}. \bibinfo{pages}{2443--2453}.
\newblock
\href{https://doi.org/10.18653/v1/D17-1259}{doi:\nolinkurl{10.18653/v1/D17-1259}}


\bibitem[Liu et~al\mbox{.}(2026)]%
        {liu2026agenticpay}
\bibfield{author}{\bibinfo{person}{Xianyang Liu}, \bibinfo{person}{Shangding Gu}, {and} \bibinfo{person}{Dawn Song}.} \bibinfo{year}{2026}\natexlab{}.
\newblock \showarticletitle{AgenticPay: A Multi-Agent LLM Negotiation System for Buyer-Seller Transactions}.
\newblock \bibinfo{journal}{\emph{arXiv preprint arXiv:2602.06008}} (\bibinfo{year}{2026}).
\newblock


\bibitem[McClennen(1990)]%
        {mcclennen1990rationality}
\bibfield{author}{\bibinfo{person}{Edward~Francis McClennen}.} \bibinfo{year}{1990}\natexlab{}.
\newblock \bibinfo{booktitle}{\emph{Rationality and dynamic choice: Foundational explorations}}.
\newblock \bibinfo{publisher}{Cambridge University Press}.
\newblock


\bibitem[Myerson and Satterthwaite(1983)]%
        {myerson1983efficient}
\bibfield{author}{\bibinfo{person}{Roger~B. Myerson} {and} \bibinfo{person}{Mark~A. Satterthwaite}.} \bibinfo{year}{1983}\natexlab{}.
\newblock \showarticletitle{Efficient Mechanisms for Bilateral Trading}.
\newblock \bibinfo{journal}{\emph{Journal of Economic Theory}} \bibinfo{volume}{29}, \bibinfo{number}{2} (\bibinfo{year}{1983}), \bibinfo{pages}{265--281}.
\newblock


\bibitem[Nash(1950)]%
        {nash1950bargaining}
\bibfield{author}{\bibinfo{person}{John~F. Nash}.} \bibinfo{year}{1950}\natexlab{}.
\newblock \showarticletitle{The Bargaining Problem}.
\newblock \bibinfo{journal}{\emph{Econometrica}} \bibinfo{volume}{18}, \bibinfo{number}{2} (\bibinfo{year}{1950}), \bibinfo{pages}{155--162}.
\newblock


\bibitem[Nguyen et~al\mbox{.}(2025)]%
        {nguyen2025llms}
\bibfield{author}{\bibinfo{person}{Ha~Thanh Nguyen}, \bibinfo{person}{Wachara Fungwacharakorn}, \bibinfo{person}{May~Myo Zin}, \bibinfo{person}{Randy Goebel}, \bibinfo{person}{Francesca Toni}, \bibinfo{person}{Kostas Stathis}, {and} \bibinfo{person}{Ken Satoh}.} \bibinfo{year}{2025}\natexlab{}.
\newblock \showarticletitle{LLMs for legal reasoning: A unified framework and future perspectives}.
\newblock \bibinfo{journal}{\emph{Computer Law \& Security Review}}  \bibinfo{volume}{58} (\bibinfo{year}{2025}), \bibinfo{pages}{106165}.
\newblock


\bibitem[Piatti et~al\mbox{.}(2024)]%
        {piatti2024cooperate}
\bibfield{author}{\bibinfo{person}{Giorgio Piatti}, \bibinfo{person}{Zhijing Jin}, \bibinfo{person}{Max Kleiman-Weiner}, \bibinfo{person}{Bernhard Sch{\"o}lkopf}, \bibinfo{person}{Mrinmaya Sachan}, {and} \bibinfo{person}{Rada Mihalcea}.} \bibinfo{year}{2024}\natexlab{}.
\newblock \showarticletitle{Cooperate or Collapse: Emergence of Sustainable Cooperation in a Society of {LLM} Agents}. In \bibinfo{booktitle}{\emph{Advances in Neural Information Processing Systems}}.
\newblock
\urldef\tempurl%
\url{https://openreview.net/forum?id=0zWzJj6lO3}
\showURL{%
\tempurl}


\bibitem[Qian et~al\mbox{.}(2026)]%
        {qian2026strategic}
\bibfield{author}{\bibinfo{person}{Crystal Qian}, \bibinfo{person}{Kehang Zhu}, \bibinfo{person}{John Horton}, \bibinfo{person}{Benjamin Manning}, \bibinfo{person}{Vivian Tsai}, \bibinfo{person}{James Wexler}, {and} \bibinfo{person}{Nithum Thain}.} \bibinfo{year}{2026}\natexlab{}.
\newblock \showarticletitle{Strategic tradeoffs between humans and ai in multi-agent bargaining}. In \bibinfo{booktitle}{\emph{Proceedings of the 31st International Conference on Intelligent User Interfaces}}. \bibinfo{pages}{1625--1646}.
\newblock


\bibitem[Rothschild et~al\mbox{.}(2025)]%
        {rothschild2025agenticeconomy}
\bibfield{author}{\bibinfo{person}{David~M. Rothschild}, \bibinfo{person}{Markus Mobius}, \bibinfo{person}{Jake~M. Hofman}, \bibinfo{person}{Eleanor~W. Dillon}, \bibinfo{person}{Daniel~G. Goldstein}, \bibinfo{person}{Nicole Immorlica}, \bibinfo{person}{Sonia Jaffe}, \bibinfo{person}{Brendan Lucier}, \bibinfo{person}{Aleksandrs Slivkins}, {and} \bibinfo{person}{Matthew Vogel}.} \bibinfo{year}{2025}\natexlab{}.
\newblock \bibinfo{title}{The Agentic Economy}.
\newblock
\showeprint[arxiv]{2505.15799}~[cs.CY]
\urldef\tempurl%
\url{https://arxiv.org/abs/2505.15799}
\showURL{%
\tempurl}


\bibitem[Rubinstein(1982)]%
        {rubinstein1982bargaining}
\bibfield{author}{\bibinfo{person}{Ariel Rubinstein}.} \bibinfo{year}{1982}\natexlab{}.
\newblock \showarticletitle{Perfect Equilibrium in a Bargaining Model}.
\newblock \bibinfo{journal}{\emph{Econometrica}} \bibinfo{volume}{50}, \bibinfo{number}{1} (\bibinfo{year}{1982}), \bibinfo{pages}{97--110}.
\newblock
\href{https://doi.org/10.2307/1912531}{doi:\nolinkurl{10.2307/1912531}}


\bibitem[Sandholm and Lesser(2001)]%
        {sandholm2001leveled}
\bibfield{author}{\bibinfo{person}{Tuomas~W Sandholm} {and} \bibinfo{person}{Victor~R Lesser}.} \bibinfo{year}{2001}\natexlab{}.
\newblock \showarticletitle{Leveled commitment contracts and strategic breach}.
\newblock \bibinfo{journal}{\emph{Games and Economic Behavior}} \bibinfo{volume}{35}, \bibinfo{number}{1-2} (\bibinfo{year}{2001}), \bibinfo{pages}{212--270}.
\newblock


\bibitem[Schauer et~al\mbox{.}(2023)]%
        {schauer2023nine}
\bibfield{author}{\bibinfo{person}{Marco Schauer}, \bibinfo{person}{Johann~M Majer}, {and} \bibinfo{person}{Roman Tr{\"o}tschel}.} \bibinfo{year}{2023}\natexlab{}.
\newblock \showarticletitle{Nine degrees of uncertainty in negotiations}.
\newblock \bibinfo{journal}{\emph{Negotiation Journal}} \bibinfo{volume}{39}, \bibinfo{number}{2} (\bibinfo{year}{2023}), \bibinfo{pages}{207--228}.
\newblock


\bibitem[{SCML Organizing Committee}(2026)]%
        {scml2026website}
\bibfield{author}{\bibinfo{person}{{SCML Organizing Committee}}.} \bibinfo{year}{2026}\natexlab{}.
\newblock \bibinfo{title}{Supply Chain Management League ({SCML})}.
\newblock \bibinfo{howpublished}{Automated Negotiating Agents Competition (ANAC)}.
\newblock
\shownote{Accessed: 2026-10-02}.
\newblock
\urldef\tempurl%
\url{https://anac.cs.brown.edu/scml}
\showURL{%
\tempurl}


\bibitem[Shahidi et~al\mbox{.}(2025)]%
        {shahidi2025coasean}
\bibfield{author}{\bibinfo{person}{Peyman Shahidi}, \bibinfo{person}{Gili Rusak}, \bibinfo{person}{Benjamin~S Manning}, \bibinfo{person}{Andrey Fradkin}, {and} \bibinfo{person}{John~J Horton}.} \bibinfo{year}{2025}\natexlab{}.
\newblock \bibinfo{booktitle}{\emph{The Coasean Singularity? Demand, Supply, and Market Design with AI Agents}}.
\newblock \bibinfo{type}{{T}echnical {R}eport}. \bibinfo{institution}{National Bureau of Economic Research}.
\newblock


\bibitem[Stone et~al\mbox{.}(2002)]%
        {stone2002attac}
\bibfield{author}{\bibinfo{person}{Peter Stone}, \bibinfo{person}{Robert~E. Schapire}, \bibinfo{person}{J{\'a}nos~A. Csirik}, \bibinfo{person}{Michael~L. Littman}, {and} \bibinfo{person}{David McAllester}.} \bibinfo{year}{2002}\natexlab{}.
\newblock \showarticletitle{{ATTac}-2001: A Learning, Autonomous Bidding Agent}. In \bibinfo{booktitle}{\emph{Agent-Mediated Electronic Commerce {IV}}}. \bibinfo{pages}{143--160}.
\newblock
\href{https://doi.org/10.1007/3-540-36378-5_9}{doi:\nolinkurl{10.1007/3-540-36378-5_9}}


\bibitem[Szabo(1997)]%
        {szabo1997formalizing}
\bibfield{author}{\bibinfo{person}{Nick Szabo}.} \bibinfo{year}{1997}\natexlab{}.
\newblock \showarticletitle{Formalizing and securing relationships on public networks}.
\newblock \bibinfo{journal}{\emph{First monday}} (\bibinfo{year}{1997}).
\newblock


\bibitem[Telser(1980)]%
        {telser1980theory}
\bibfield{author}{\bibinfo{person}{Lester~G Telser}.} \bibinfo{year}{1980}\natexlab{}.
\newblock \showarticletitle{A theory of self-enforcing agreements}.
\newblock \bibinfo{journal}{\emph{Journal of business}} (\bibinfo{year}{1980}), \bibinfo{pages}{27--44}.
\newblock


\bibitem[Tewolde et~al\mbox{.}(2026)]%
        {tewolde2026coopeval}
\bibfield{author}{\bibinfo{person}{Emanuel Tewolde}, \bibinfo{person}{Xiao Zhang}, \bibinfo{person}{David~Guzman Piedrahita}, \bibinfo{person}{Vincent Conitzer}, {and} \bibinfo{person}{Zhijing Jin}.} \bibinfo{year}{2026}\natexlab{}.
\newblock \showarticletitle{CoopEval: Benchmarking Cooperation-Sustaining Mechanisms and LLM Agents in Social Dilemmas}.
\newblock \bibinfo{journal}{\emph{arXiv preprint arXiv:2604.15267}} (\bibinfo{year}{2026}).
\newblock


\bibitem[Thomas(2022)]%
        {thomas2022rational}
\bibfield{author}{\bibinfo{person}{Naveen Thomas}.} \bibinfo{year}{2022}\natexlab{}.
\newblock \showarticletitle{Rational Contract Design}.
\newblock \bibinfo{journal}{\emph{Alabama Law Review}}  \bibinfo{volume}{74} (\bibinfo{year}{2022}), \bibinfo{pages}{967}.
\newblock


\bibitem[Tomasev et~al\mbox{.}(2025)]%
        {tomasev2025virtual}
\bibfield{author}{\bibinfo{person}{Nenad Tomasev}, \bibinfo{person}{Matija Franklin}, \bibinfo{person}{Joel~Z Leibo}, \bibinfo{person}{Julian Jacobs}, \bibinfo{person}{William~A Cunningham}, \bibinfo{person}{Iason Gabriel}, {and} \bibinfo{person}{Simon Osindero}.} \bibinfo{year}{2025}\natexlab{}.
\newblock \showarticletitle{Virtual agent economies}.
\newblock \bibinfo{journal}{\emph{arXiv preprint arXiv:2509.10147}} (\bibinfo{year}{2025}).
\newblock


\bibitem[Vezhnevets et~al\mbox{.}(2023)]%
        {vezhnevets2023generativeagentbasedmodelingactions}
\bibfield{author}{\bibinfo{person}{Alexander~Sasha Vezhnevets}, \bibinfo{person}{John~P. Agapiou}, \bibinfo{person}{Avia Aharon}, \bibinfo{person}{Ron Ziv}, \bibinfo{person}{Jayd Matyas}, \bibinfo{person}{Edgar~A. Duéñez-Guzmán}, \bibinfo{person}{William~A. Cunningham}, \bibinfo{person}{Simon Osindero}, \bibinfo{person}{Danny Karmon}, {and} \bibinfo{person}{Joel~Z. Leibo}.} \bibinfo{year}{2023}\natexlab{}.
\newblock \bibinfo{title}{Generative agent-based modeling with actions grounded in physical, social, or digital space using Concordia}.
\newblock
\showeprint[arxiv]{2312.03664}~[cs.AI]
\urldef\tempurl%
\url{https://arxiv.org/abs/2312.03664}
\showURL{%
\tempurl}


\bibitem[Weirich(2009)]%
        {weirich2009collective}
\bibfield{author}{\bibinfo{person}{Paul Weirich}.} \bibinfo{year}{2009}\natexlab{}.
\newblock \bibinfo{booktitle}{\emph{Collective rationality: Equilibrium in cooperative games}}.
\newblock \bibinfo{publisher}{Oxford University Press}.
\newblock


\bibitem[Werbach and Cornell(2017)]%
        {werbach2017contracts}
\bibfield{author}{\bibinfo{person}{Kevin Werbach} {and} \bibinfo{person}{Nicolas Cornell}.} \bibinfo{year}{2017}\natexlab{}.
\newblock \showarticletitle{Contracts ex machina}.
\newblock \bibinfo{journal}{\emph{Duke Law Journal}}  \bibinfo{volume}{67} (\bibinfo{year}{2017}), \bibinfo{pages}{313}.
\newblock


\bibitem[Wyse et~al\mbox{.}(2026)]%
        {wyse2026commitment}
\bibfield{author}{\bibinfo{person}{Tim Wyse}, \bibinfo{person}{Kaitlin Bustos}, \bibinfo{person}{Yulia Volkova}, {and} \bibinfo{person}{Max Kleiman-Weiner}.} \bibinfo{year}{2026}\natexlab{}.
\newblock \showarticletitle{Commitment To Cooperation With Self-Negotiated Contracts}.
\newblock \bibinfo{journal}{\emph{arXiv preprint arXiv:2607.22750}} (\bibinfo{year}{2026}).
\newblock


\bibitem[Xia et~al\mbox{.}(2024)]%
        {xia2024measuring}
\bibfield{author}{\bibinfo{person}{Tian Xia}, \bibinfo{person}{Zhiwei He}, \bibinfo{person}{Tong Ren}, \bibinfo{person}{Yibo Miao}, \bibinfo{person}{Zhuosheng Zhang}, \bibinfo{person}{Yang Yang}, {and} \bibinfo{person}{Rui Wang}.} \bibinfo{year}{2024}\natexlab{}.
\newblock \showarticletitle{Measuring Bargaining Abilities of {LLM}s: A Benchmark and A Buyer-Enhancement Method}. In \bibinfo{booktitle}{\emph{Findings of the Association for Computational Linguistics: {ACL} 2024}}.
\newblock
\urldef\tempurl%
\url{https://arxiv.org/abs/2402.15813}
\showURL{%
\tempurl}


\bibitem[Ying et~al\mbox{.}(2025)]%
        {ying2025liras}
\bibfield{author}{\bibinfo{person}{Lance Ying}, \bibinfo{person}{Ryan Truong}, \bibinfo{person}{Katherine~M. Collins}, \bibinfo{person}{Cedegao~E. Zhang}, \bibinfo{person}{Megan Wei}, \bibinfo{person}{Tyler BrookeWilson}, \bibinfo{person}{Tan Zhi-Xuan}, \bibinfo{person}{Lionel Wong}, {and} \bibinfo{person}{Joshua~B. Tenenbaum}.} \bibinfo{year}{2025}\natexlab{}.
\newblock \showarticletitle{Language-Informed Synthesis of Rational Agent Models for Grounded Theory-of-Mind Reasoning On-the-fly}. In \bibinfo{booktitle}{\emph{Findings of the Association for Computational Linguistics: EMNLP 2025}}. \bibinfo{publisher}{Association for Computational Linguistics}, \bibinfo{pages}{12217--12235}.
\newblock


\bibitem[Zeng and Sycara(1998)]%
        {zeng1998bayesian}
\bibfield{author}{\bibinfo{person}{Dajun Zeng} {and} \bibinfo{person}{Katia Sycara}.} \bibinfo{year}{1998}\natexlab{}.
\newblock \showarticletitle{Bayesian learning in negotiation}.
\newblock \bibinfo{journal}{\emph{International Journal of Human-Computer Studies}} \bibinfo{volume}{48}, \bibinfo{number}{1} (\bibinfo{year}{1998}), \bibinfo{pages}{125--141}.
\newblock


\end{thebibliography}

\clearpage
\appendix
\addtocontents{toc}{\protect\setcounter{tocdepth}{2}}
\setcounter{figure}{0}\renewcommand{\thefigure}{S\arabic{figure}}
\setcounter{table}{0}\renewcommand{\thetable}{S\arabic{table}}
\captionsetup{textfont=normalfont} 
\onecolumn

\begin{center}
    {\LARGE\bfseries Appendix: Evaluating Rational Contracting in Natural Language}
\end{center}
\addvspace{2em}

\setcounter{tocdepth}{2}
{\renewcommand{\addcontentsline}[3]{}\tableofcontents} 
\clearpage

\section{Additional Benchmark Details}

\label{app:benchmark-hyperparams}

\begingroup
\setlength{\parindent}{0pt}
\setlength{\parskip}{0.5\baselineskip}

\subsection{Domain-agnostic environment specification}
\label{app:env-spec}

Following Section~\ref{sec:benchmark}, each environment in \textsc{ContractSim} fixes a type profile
$\theta'=(\theta_{\mathrm{Cust}},\theta_{\mathrm{Supp}})\in\Theta'$ before
assigning the game a story. Let $D=\{O_1,O_2,O_3\}$ be the product set and
$\mathcal{X}=\{I_1,I_2,I_3\}$ the input set. The Customer type
$\theta_{\mathrm{Cust}}$ specifies its budget $B$, private product values
$(v_d)_{d\in D}$, minimum delivery requirements, and substitution rule. The
Supplier type $\theta_{\mathrm{Supp}}$ specifies its production function,
initial resources, inventory and order caps, and private supply uncertainty (i.e. the probability laws described below).

Using the notation of Section~\ref{sec:benchmark}, for every production week
$w$ and input $x\in\mathcal{X}$, an environment supplies probability laws
\[
    p_{w,x}\sim\mathsf{P}_x,\qquad
    \mathrm{recv}_{w,x}\mid o_{w,x}\sim\mathsf{R}_x(\,\cdot\mid o_{w,x}),\qquad
    \mathrm{spoil}_{w,x}\sim\mathsf{S}_x,
\]
where $p_{w,x}$ and $o_{w,x}$ are respectively the realized unit price and
ordered amount from Section~\ref{sec:benchmark}, $\mathrm{recv}_{w,x}$ is the
amount received, and $\mathrm{spoil}_{w,x}$ is the spoilage shock. The realized
spoilage loss is the smaller of $\mathrm{spoil}_{w,x}$ and the inventory
available immediately before
spoilage. Draws are independent across inputs and production weeks;
environments without stochasticity use point-mass laws. The main performance
experiments use environment seed 42 for these stochastic draws. These distributions are known only to the Supplier, not the Customer.

The Supplier production function is shared by all environments and is given by
the three unit recipes
\[
    O_1=I_1+I_2+I_3,\qquad
    O_2=I_1+I_2,\qquad
    O_3=I_3.
\]
Thus one unit of an output consumes one unit of every input appearing on its
right-hand side.

The six environments vary $(v_d)_{d\in D}$, the Supplier's initial capital,
and the laws $(\mathsf{P}_x,\mathsf{R}_x,\mathsf{S}_x)$ while holding the
remaining structural parameters fixed. All environments last for $L=11$ weeks in the performance phase, with $(L+1)/2=6$ payment weeks $W_{\mathrm{pay}}$ and $(L-1)/2 = 5$ production weeks
$W_{\mathrm{prod}}$:
\[
W_{\mathrm{pay}} := \{2k-1\}_{k=1}^{(L+1)/2} \qquad
W_{\mathrm{prod}}=\{2k\}_{k=1}^{(L-1)/2}
\]

\subsection{Setting-level entity mapping}

A setting (i.e. story) instantiation only relabels the abstract players, inputs, and outputs
and rescales the monetary unit; it does not change the underlying production
graph or transition dynamics. Table entries on the same row therefore play the same formal role. A plus sign
in an output recipe denotes one unit of every listed input per unit of output.

\begin{table}[htbp]
\centering
\caption{\textbf{Mapping between the abstract \textsc{ContractSim} entities and the
three matched settings.} Recipes and nonmonetary dynamics are preserved
under the mapping; monetary quantities use the listed multiplier and
granularity.}
\label{tab:story-entity-mapping}
\small
\setlength{\tabcolsep}{6pt}
\renewcommand{\arraystretch}{1.12}
\begin{tabularx}{\textwidth}{l>{\raggedright\arraybackslash}X>{\raggedright\arraybackslash}X>{\raggedright\arraybackslash}X}
\toprule
\textbf{Underlying entity} & \textbf{Catering} &
\textbf{Hotel Cleaning} & \textbf{AI Hosting} \\
\midrule
Customer role & Customer & Hotel owner & AI startup founder \\
Supplier role & Caterer & Cleaning provider & AI service provider \\
\midrule
Input/resource $I_1$ & Mozzarella & Detergent & CPU \\
Input/resource $I_2$ & Basil & Disinfectant & Memory \\
Input/resource $I_3$ & Tomato & Wet-wipes & GPU \\
\midrule
Output $O_1=I_1+I_2+I_3$ & Margherita Pizza & Conference Room Cleaning & Fine-tuning \\
Output $O_2=I_1+I_2$ & Pesto Pasta & King-Sized Room Cleaning & Model-hosting \\
Output $O_3=I_3$ & Tomato Soup & Single Room Cleaning & Inference \\
\midrule
Monetary multiplier & $1\times$ & $100\times$ & $1{,}000\times$ \\
Price/payment granularity & \$1 & \$100 & \$1{,}000 \\
\bottomrule
\end{tabularx}
\end{table}

\subsection{Per-environment parameters}

\paragraph{Catering instantiation.}

In the Catering setting, the Customer has a budget of \$200 and requires at
least one Pizza, one Pasta, and one Soup in every production week. For the
Customer, outputs above these mandatory minima are privately substitutable at
the rate \emph{1 Pizza = 2 Pastas = 4 Soups}.

The Caterer's inventory is initially empty, has a 10-unit cap for each
ingredient, and has a per-round order cap of 12 units per ingredient. These
caps were chosen to keep the rational baselines computationally tractable.

Table~\ref{tab:catering-environments} lists the values and distributions used
in each environment. Product values are ordered Pizza/Pasta/Soup (Pi/Pa/S),
and ingredient prices are ordered Mozzarella/Basil/Tomato (M/B/T). Each
stochastic price entry gives an outcome followed by its probability.

Spoilage occurs before the new order. The Caterer pays for all $x$ units
ordered even when fewer arrive or the inventory cap discards excess units.

\begin{table}[htbp]
\centering
\caption{\textbf{Customer values, Supplier capital, and realized price, receipt, and
spoilage distributions for the six Catering environments.} Ingredient-level
draws are independent across inputs and production weeks.}
\label{tab:catering-environments}
\footnotesize
\setlength{\tabcolsep}{1.5pt}
\renewcommand{\arraystretch}{1.5}
\begin{tabularx}{\textwidth}{clcc>{\raggedright\arraybackslash}X>{\raggedright\arraybackslash}X>{\raggedright\arraybackslash}X}
\toprule
\textbf{Env.} & \textbf{Regime} &
\shortstack{\textbf{Values}\\\textbf{(USD; Pi/Pa/S)}} &
\textbf{Capital} &
\shortstack{\textbf{Ingredient-price distribution}\\\textbf{(USD; M/B/T)}} &
\shortstack{\textbf{Receipt distribution}\\\textbf{for an order of $x$}} &
\shortstack{\textbf{Spoilage distribution}\\\textbf{per ingredient}} \\
\midrule
1 & \multirow{2}{*}{none} & 12/6/3  & \$20 &
    \multirow{2}{=}{M: \$1; B: \$1; T: \$2 (fixed)} &
    \multirow{2}{=}{Exactly $x$ units} &
    \multirow{2}{=}{Exactly 0 units} \\
2 & & 20/10/5 & \$40 & & & \\
\addlinespace[2pt]
3 & \multirow{2}{*}{low} & 12/6/3  & \$20 &
    \multirow{2}{=}{M: \$1 ($2/3$), \$3 ($1/3$);
        B: \$1 ($2/3$), \$2 ($1/3$);
        T: \$2 ($2/3$), \$3 ($1/3$)} &
    \multirow{2}{=}{Uniform over the integers
        $\lfloor3x/4\rfloor,\ldots,x$} &
    \multirow{2}{=}{Uniform over 0 or 1 unit} \\
4 & & 20/10/5 & \$40 & & & \\
\addlinespace[2pt]
5 & \multirow{2}{*}{high} & 12/6/3  & \$20 &
    \multirow{2}{=}{M: \$1 ($1/2$), \$3 ($1/2$);
        B: \$1 ($1/2$), \$2 ($1/2$);
        T: \$2 ($1/2$), \$3 ($1/2$)} &
    \multirow{2}{=}{Uniform over the integers
        $\lfloor x/2\rfloor,\ldots,x$} &
    \multirow{2}{=}{Uniform over 0, 1, or 2 units} \\
6 & & 20/10/5 & \$40 & & & \\
\bottomrule
\end{tabularx}
\end{table}

\paragraph{Matched instantiations across settings.}

Each of the six Catering environments induces a matched Hotel Cleaning
environment by multiplying all monetary amounts by 100 and a matched AI
Hosting environment by multiplying them by 1,000. Contract prices and payments
must consequently be multiples of \$100 in Hotel Cleaning and \$1{,}000 in AI
Hosting.

All nonmonetary hyperparameters, such as minimum requirements, recipes,
inventory and order caps, price probabilities, delivery-loss support, and
spoilage support, remain unchanged under this mapping.

\endgroup

\addvspace{60pt}
\section{Agent Scaffolds and Prompts}

\subsection{Original Agent System Prompts}
\label{app:negotiation-prompts}

We reproduce the exact role-specific system prompts for all three story
domains in the matched robustness experiment. For Catering, these are the
prompts from Environment~5 in
the paper's six-environment numbering. Hotel Cleaning and AI Hosting (called
AI Model Hosting in the reproduced prompts) use the corresponding matched
Environment~5 conditions. Each prompt begins on a new page.

The main results and ablation experiments use the Catering prompts as their
base templates. For each evaluated environment, we modify only the
environment-specific details---the Customer's private values, the Caterer's
initial capital, ingredient-price distribution, delivery-loss distribution,
and spoilage distribution---according to the hyperparameters in
Appendix~\ref{app:benchmark-hyperparams}. The Catering prompts reproduced
below show the paper Environment~5 instantiation; ablation runs additionally
apply the corresponding ablation instruction.

All environment-dependent text is expanded verbatim. The sole remaining
template token is \texttt{\{\{NEGOTIATED\_CONTRACT\}\}} in each performance
prompt; at runtime, the benchmark replaces it with the complete contract for
the current game between the displayed boundary markers. Dynamic turn and
state/action requests are generated from the current interaction and are not
part of these static system prompts.

\clearpage
\textbf{Catering: Negotiation---Customer.}
\label{app:prompt-catering-negotiation-customer}
\promptboxhead
\begin{promptbox}{blue}{Catering negotiation system prompt---Customer}
\promptkey{GOAL}
\promptboxhead\plabel{NEGOTIATION PROCESS:}
You have at most 50 rounds to reach an agreement.\par\addvspace{\smallskipamount}
Each round has two turns: the customer acts first, followed by the caterer.\par\addvspace{\smallskipamount}
On your turn, you may send a message to the other party and optionally propose a contract.
You may formally accept the caterer's most recent contract. Accepting it ends the
negotiation with an agreement.\par\addvspace{\smallskipamount}
You may also terminate the negotiation and walk away at any time. If either party terminates
the negotiation, or if no agreement is reached within 50 rounds, negotiations end with no
contract.\par\addvspace{\smallskipamount}
Negotiate a contract with the caterer, explicitly negotiating: (i) price per unit of each
dish, (ii) delivery schedule (how many of each dish per production round over 5 production
rounds), and (iii) payment schedule (how much to pay in each of the 6 payment weeks).\par\addvspace{\smallskipamount}
Aim to negotiate a contract that maximizes your expected private utility when executed
during the subsequent contract-execution phase.\par\addvspace{\smallskipamount}
\promptboxhead\plabel{IMPORTANT:}
Any contract you propose or accept MUST explicitly state the price per unit for each dish,
the exact quantity of each dish per production round, and the exact payment amount for each
payment week --- do not accept a contract that leaves any of these unspecified.\par\addvspace{\smallskipamount}
All contracted dish quantities must be non-negative whole numbers, and all payment amounts
must be non-negative whole-dollar amounts (no fractions).\par\addvspace{\smallskipamount}
Only propose a contract once you believe you have gathered sufficient information from the
negotiation to specify all of these terms.\par\addvspace{\smallskipamount}
\tcbline
\promptkey{PRIVATE INFO}
You are a customer about to negotiate a catering contract with a caterer.\par\addvspace{\smallskipamount}
Your total budget is \$200 --- this is your absolute limit and you cannot spend beyond it.\par\addvspace{\smallskipamount}
\promptboxhead\plabel{PRIVATE VALUATION:}
You privately value each Margherita Pizza at \$12, each Pesto Pasta at \$6, and each Tomato
Soup at \$3. These are your maximum per-unit valuations. You are willing to pay up to these
amounts for each dish, but you prefer to save money.\par\addvspace{\smallskipamount}
\promptboxhead\plabel{MINIMUM REQUIREMENTS:}
You need at least 1 Margherita Pizza, 1 Pesto Pasta, and 1 Tomato Soup per production round
--- these minimums are mandatory and cannot be substituted. All dish quantities and payments
must be whole numbers (no fractions).\par\addvspace{\smallskipamount}
\promptboxhead\plabel{DISH SUBSTITUTABILITY:}
For any quantities ABOVE the minimum, dishes are substitutable at the rate: 1 Pizza = 2
Pastas = 4 Soups (equivalently, 1 Pasta = 2 Soups). For example, if you wanted 2 Pizzas but
the caterer can only make the 1 minimum, the extra Pizza can be replaced by 2 Pastas or 4
Soups. The same logic applies to Pastas too.\par\addvspace{\smallskipamount}
\promptboxhead\plabel{SCHEDULE:}
The game lasts 11 weeks. Odd weeks (1, 3, 5, 7, 9, 11) are PAYMENT weeks --- you make a
payment (6 total). Even weeks (2, 4, 6, 8, 10) are PRODUCTION weeks --- the caterer buys
ingredients and delivers dishes (5 total).\par\addvspace{\smallskipamount}
\promptboxhead\plabel{CONTRACT EXECUTION ACTIONS:}
During contract execution, the only action at your disposal is making payments. Payments
must be non-negative whole-dollar amounts, including \$0.\par\addvspace{\smallskipamount}
The only actions at the caterer's disposal are: (i) buying ingredients, and (ii) producing
dishes.\par\addvspace{\smallskipamount}
There is no refund mechanism in this system.\par\addvspace{\smallskipamount}
There is no communication between you and the caterer after a contract has been agreed upon,
hence no renegotiation.\par\addvspace{\smallskipamount}
\promptboxhead\plabel{OBJECTIVES:}
Your priority is maximizing your expected private utility, where the utility is your
remaining budget plus the cumulative private value of dishes.\par\addvspace{\smallskipamount}
\promptboxhead\plabel{BEST ALTERNATIVE TO NO AGREEMENT:}
If no contract is agreed upon, you are guaranteed to receive your minimum required dishes in
every production round from an outside provider. You pay a total amount exactly equal to the
cumulative private value you assign to all of those dishes, and that payment is deducted
from your budget.\par\addvspace{\smallskipamount}
\end{promptbox}

\clearpage
\textbf{Catering: Negotiation---Caterer.}
\label{app:prompt-catering-negotiation-caterer}
\promptboxhead
\begin{promptbox}{red}{Catering negotiation system prompt---Caterer}
\promptkey{GOAL}
\promptboxhead\plabel{NEGOTIATION PROCESS:}
You have at most 50 rounds to reach an agreement.\par\addvspace{\smallskipamount}
Each round has two turns: the customer acts first, followed by the caterer.\par\addvspace{\smallskipamount}
On your turn, you may send a message to the other party and optionally propose a contract.
You cannot formally accept a contract. If you agree with terms proposed by the customer,
propose the complete agreed contract for the customer to formally accept on their next turn.
The customer's formal acceptance ends the negotiation with an agreement.\par\addvspace{\smallskipamount}
You may terminate the negotiation and walk away at any time. If either party terminates the
negotiation, or if no agreement is reached within 50 rounds, negotiations end with no
contract.\par\addvspace{\smallskipamount}
Negotiate a contract with the customer, explicitly negotiating: (i) price per unit of each
dish, (ii) delivery schedule (how many of each dish per production round over 5 production
rounds), and (iii) payment schedule (how much to pay in each of the 6 payment weeks).\par\addvspace{\smallskipamount}
Aim to negotiate a contract that maximizes your expected profit when executed during the
subsequent contract-execution phase.\par\addvspace{\smallskipamount}
\promptboxhead\plabel{IMPORTANT:}
Any contract you propose MUST explicitly state the price per unit for each dish, the exact
quantity of each dish per production round, and the exact payment amount for each payment
week --- do not propose a contract that leaves any of these unspecified.\par\addvspace{\smallskipamount}
All contracted dish quantities must be non-negative whole numbers, and all payment amounts
must be non-negative whole-dollar amounts (no fractions).\par\addvspace{\smallskipamount}
Only propose a contract once you believe you have gathered sufficient information from the
negotiation to specify all of these terms.\par\addvspace{\smallskipamount}
\tcbline
\promptkey{PRIVATE INFO}
You are a caterer negotiating a catering contract with a prospective customer.\par\addvspace{\smallskipamount}
Your initial budget: \$20.\par\addvspace{\smallskipamount}
\promptboxhead\plabel{RECIPES:}
You produce Margherita Pizza, Pesto Pasta, and Tomato Soup using Mozzarella, Basil, and
Tomato.\par\addvspace{\smallskipamount}
Each Margherita Pizza needs 1 Tomato + 1 Mozzarella + 1 Basil, each Pesto Pasta needs 1
Mozzarella + 1 Basil, each Tomato Soup needs 1 Tomato.\par\addvspace{\smallskipamount}
All ingredient-purchase and dish-production quantities must be non-negative whole numbers.\par\addvspace{\smallskipamount}
\promptboxhead\plabel{ENVIRONMENT STOCHASTICITY:}
(a) Prices: each production round, each ingredient price is independently set to one of two
values with equal probability (50/50) --- Mozzarella \$1 or \$3/unit, Basil \$1 or \$2/unit,
Tomato \$2 or \$3/unit. (b) Delivery loss: for each ingredient, if you order x units you
receive between floor(x/2) and x units (uniformly at random). You pay for the full amount
ordered, not the amount received. For example, ordering 6 Mozzarella means you pay for 6 but
receive 3, 4, 5, or 6 with equal chance. (c) Spoilage: inventory carried over from the
previous production round independently loses 0, 1, or 2 units per ingredient per production
round (equal probability, 1/3 each). Spoilage happens before you order new ingredients. (d)
Inventory cap: you can store at most 10 units of each ingredient. Any received units that
would exceed this cap are lost --- you still pay for them. This cap applies before production,
so you cannot order beyond capacity to use the surplus in the same round. Maximum order: 12
units of each ingredient per production round.\par\addvspace{\smallskipamount}
\promptboxhead\plabel{SCHEDULE:}
The game lasts 11 weeks. Odd weeks (1, 3, 5, 7, 9, 11) are PAYMENT weeks --- the customer pays
you (6 total). Even weeks (2, 4, 6, 8, 10) are PRODUCTION weeks --- you buy ingredients and
deliver dishes (5 total).\par\addvspace{\smallskipamount}
\promptboxhead\plabel{PRODUCTION WEEK PROCESS (in order):}
(1) inventory spoilage occurs, (2) you order ingredients (delivery loss applies), (3) you
produce dishes and deliver them, (4) prices are re-randomized for next production round.\par\addvspace{\smallskipamount}
\promptboxhead\plabel{CONTRACT EXECUTION ACTIONS:}
During contract execution, your only actions are (i) buying ingredients and (ii) producing
dishes.\par\addvspace{\smallskipamount}
The customer's only action is making payments. Payments must be non-negative whole-dollar
amounts, including \$0.\par\addvspace{\smallskipamount}
There is no refund mechanism in this system.\par\addvspace{\smallskipamount}
Unused ingredients carry over between production rounds (subject to spoilage).\par\addvspace{\smallskipamount}
There is no communication between you and the customer after a contract has been agreed
upon, hence no possibility of renegotiation.\par\addvspace{\smallskipamount}
\promptboxhead\plabel{OBJECTIVES:}
Your priority is maximizing your expected profit, where profit is defined as the money you
earn from the customer minus the total amount you spend buying ingredients.\par\addvspace{\smallskipamount}
\promptboxhead\plabel{BEST ALTERNATIVE TO NO AGREEMENT:}
If no contract is agreed upon, you walk away with only your initial budget, \$20.\par\addvspace{\smallskipamount}
\end{promptbox}

\clearpage
\textbf{Catering: Performance---Customer.}
\label{app:prompt-catering-performance-customer}
\promptboxhead
\begin{promptbox}{blue}{Catering performance system prompt---Customer}
\promptkey{GOAL}
Perform the contract with the supplier, choosing how much you wish to pay the supplier in
every payment week. Aim to perform the contract in a way that maximizes your private
utility.\par\addvspace{\smallskipamount}
\tcbline
\promptkey{PRIVATE INFO}
You are a customer who has negotiated a catering contract with a supplier. Your total budget
is \$200 --- this is your absolute limit and you cannot spend beyond it. Here is the negotiated
contract:\par\addvspace{\smallskipamount}
\begin{center}\ttfamily
\textless{}\textless{}\textless{} BEGIN CONTRACT \textgreater{}\textgreater{}\textgreater{}\\
\{\{NEGOTIATED\_CONTRACT\}\}\\
\textless{}\textless{}\textless{} END CONTRACT \textgreater{}\textgreater{}\textgreater{}
\end{center}
\promptboxhead\plabel{PRIVATE VALUATION:}
You privately value each Margherita Pizza at \$12, each Pesto Pasta at \$6, and each Tomato
Soup at \$3.\par\addvspace{\smallskipamount}
\promptboxhead\plabel{MINIMUM REQUIREMENTS:}
You need at least 1 Margherita Pizza, 1 Pesto Pasta, and 1 Tomato Soup per production round
--- these minimums are mandatory and cannot be substituted. All dish quantities and payments
must be whole numbers (no fractions).\par\addvspace{\smallskipamount}
\promptboxhead\plabel{DISH SUBSTITUTABILITY:}
For any quantities ABOVE the minimum, dishes are substitutable at the rate: 1 Pizza = 2
Pastas = 4 Soups (equivalently, 1 Pasta = 2 Soups). For example, if you wanted 2 Pizzas but
the caterer can only make the 1 minimum, the extra Pizza can be replaced by 2 Pastas or 4
Soups. The same logic applies to Pastas too.\par\addvspace{\smallskipamount}
\promptboxhead\plabel{SCHEDULE:}
The game lasts 11 weeks. Odd weeks (1, 3, 5, 7, 9, 11) are PAYMENT weeks --- you make a
payment (6 total). Even weeks (2, 4, 6, 8, 10) are PRODUCTION weeks --- the supplier buys
ingredients and delivers dishes (5 total).\par\addvspace{\smallskipamount}
\promptboxhead\plabel{CONTRACT EXECUTION ACTIONS:}
During contract execution, the only action at your disposal is making payments. Payments
must be non-negative whole-dollar amounts, including \$0. There is no refund mechanism. The
only actions at the supplier's disposal are: (i) buying ingredients, and (ii) producing
dishes. No communication or renegotiation during execution.\par\addvspace{\smallskipamount}
\promptboxhead\plabel{OBJECTIVES:}
Your priority is maximizing your private utility, where the utility is your remaining budget
plus the cumulative private value of dishes.\par\addvspace{\smallskipamount}
\end{promptbox}

\clearpage
\textbf{Catering: Performance---Caterer.}
\label{app:prompt-catering-performance-caterer}
\promptboxhead
\begin{promptbox}{red}{Catering performance system prompt---Caterer}
\promptkey{GOAL}
Perform the contract with the customer, choosing how many ingredients of each kind to
purchase and how many dishes of each type to produce. Aim to perform the contract in a way
that maximizes your profit.\par\addvspace{\smallskipamount}
\tcbline
\promptkey{PRIVATE INFO}
You are a supplier who has negotiated a catering contract with a customer. Your initial
budget: \$20. Here is the negotiated contract:\par\addvspace{\smallskipamount}
\begin{center}\ttfamily
\textless{}\textless{}\textless{} BEGIN CONTRACT \textgreater{}\textgreater{}\textgreater{}\\
\{\{NEGOTIATED\_CONTRACT\}\}\\
\textless{}\textless{}\textless{} END CONTRACT \textgreater{}\textgreater{}\textgreater{}
\end{center}
\promptboxhead\plabel{RECIPES:}
You produce Margherita Pizza, Pesto Pasta, and Tomato Soup using Mozzarella, Basil, and
Tomato. Each Margherita Pizza needs 1 Tomato + 1 Mozzarella + 1 Basil, each Pesto Pasta
needs 1 Mozzarella + 1 Basil, each Tomato Soup needs 1 Tomato. All ingredient-purchase and
dish-production quantities must be non-negative whole numbers.\par\addvspace{\smallskipamount}
\promptboxhead\plabel{ENVIRONMENT STOCHASTICITY:}
(a) Prices: each production round, each ingredient price is independently set to one of two
values with equal probability (50/50) --- Mozzarella \$1 or \$3/unit, Basil \$1 or \$2/unit,
Tomato \$2 or \$3/unit. (b) Delivery loss: for each ingredient, if you order x units you
receive between floor(x/2) and x units (uniformly at random). You pay for the full amount
ordered, not the amount received. For example, ordering 6 Mozzarella means you pay for 6 but
receive 3, 4, 5, or 6 with equal chance. (c) Spoilage: inventory carried over from the
previous production round independently loses 0, 1, or 2 units per ingredient per production
round (equal probability, 1/3 each). Spoilage happens before you order new ingredients. (d)
Inventory cap: you can store at most 10 units of each ingredient. Any received units that
would exceed this cap are lost --- you still pay for them. This cap applies before production,
so you cannot order beyond capacity to use the surplus in the same round. Maximum order: 12
units of each ingredient per production round.\par\addvspace{\smallskipamount}
\promptboxhead\plabel{SCHEDULE:}
The game lasts 11 weeks. Odd weeks (1, 3, 5, 7, 9, 11) are PAYMENT weeks --- the customer pays
you (6 total). Even weeks (2, 4, 6, 8, 10) are PRODUCTION weeks --- you buy ingredients and
deliver dishes (5 total).\par\addvspace{\smallskipamount}
\promptboxhead\plabel{PRODUCTION WEEK PROCESS (in order):}
(1) random inventory spoilage occurs, (2) you order ingredients (random delivery loss
applies), (3) you produce dishes and deliver them, (4) prices are re-randomized for next
production round.\par\addvspace{\smallskipamount}
\promptboxhead\plabel{CONTRACT EXECUTION ACTIONS:}
During contract execution, your only actions are (i) buying ingredients and (ii) producing
dishes. The customer's only action is making payments. Payments must be non-negative
whole-dollar amounts, including \$0. There is no refund mechanism. Unused ingredients carry
over between production rounds (subject to spoilage). No communication or renegotiation
during execution.\par\addvspace{\smallskipamount}
\promptboxhead\plabel{OBJECTIVES:}
Your priority is maximizing your profit, where profit is defined as the money you earn from
the customer minus the total amount you spend buying ingredients.\par\addvspace{\smallskipamount}
\end{promptbox}

\clearpage
\textbf{Hotel Cleaning: Negotiation---Hotel Owner.}
\label{app:prompt-hotel-cleaning-negotiation-hotel-owner}
\promptboxhead
\begin{promptbox}{blue}{Hotel Cleaning negotiation system prompt---Hotel Owner}
\promptkey{GOAL}
\promptboxhead\plabel{NEGOTIATION PROCESS:}
You have at most 50 rounds to reach an agreement.\par\addvspace{\smallskipamount}
Each round has two turns: the hotel owner acts first, followed by the cleaning provider.\par\addvspace{\smallskipamount}
On your turn, you may send a message to the other party and optionally propose a contract.
You may formally accept the cleaning provider's most recent contract. Accepting it ends the
negotiation with an agreement.\par\addvspace{\smallskipamount}
You may also terminate the negotiation and walk away at any time. If either party terminates
the negotiation, or if no agreement is reached within 50 rounds, negotiations end with no
contract.\par\addvspace{\smallskipamount}
Negotiate a contract with the cleaning provider, explicitly negotiating: (i) price per unit
of each cleaning type, (ii) cleaning schedule (how many of each cleaning type per cleaning
round over 5 cleaning rounds), and (iii) payment schedule (how much to pay in each of the 6
payment weeks).\par\addvspace{\smallskipamount}
Aim to negotiate a contract that maximizes your expected private utility when executed
during the subsequent contract-execution phase.\par\addvspace{\smallskipamount}
\promptboxhead\plabel{IMPORTANT:}
Any contract you propose or accept MUST explicitly state the price per unit for each
cleaning type, the exact quantity of each cleaning type per cleaning round, and the exact
payment amount for each payment week --- do not accept a contract that leaves any of these
unspecified.\par\addvspace{\smallskipamount}
All quantities of each contracted cleaning type must be non-negative whole numbers. All
per-unit prices and payment amounts must be non-negative multiples of \$100.\par\addvspace{\smallskipamount}
Only propose a contract once you believe you have gathered sufficient information from the
negotiation to specify all of these terms.\par\addvspace{\smallskipamount}
\tcbline
\promptkey{PRIVATE INFO}
You are a hotel owner about to negotiate a hotel cleaning contract with a cleaning provider.\par\addvspace{\smallskipamount}
Your total budget is \$20,000 --- this is your absolute limit and you cannot spend beyond it.\par\addvspace{\smallskipamount}
\promptboxhead\plabel{PRIVATE VALUATION:}
You privately value each Conference Room Cleaning at \$1,200, each King-Sized Room Cleaning
at \$600, and each Single Room Cleaning at \$300. These are your maximum per-unit valuations.
You are willing to pay up to these amounts for each cleaning type, but you prefer to save
money.\par\addvspace{\smallskipamount}
\promptboxhead\plabel{MINIMUM REQUIREMENTS:}
You need at least 1 Conference Room Cleaning, 1 King-Sized Room Cleaning, and 1 Single Room
Cleaning per cleaning round --- these minimums are mandatory and cannot be substituted. All
cleaning quantities and payments must be whole numbers (no fractions).\par\addvspace{\smallskipamount}
\promptboxhead\plabel{CLEANING-SERVICE SUBSTITUTABILITY:}
For any quantities ABOVE the minimum, cleanings are substitutable at the rate: 1 Conference
Room Cleaning = 2 King-Sized Room Cleanings = 4 Single Room Cleanings (equivalently, 1
King-Sized Room Cleaning = 2 Single Room Cleanings). For example, if you wanted 2 Conference
Room Cleanings but the cleaning provider can provide only the mandatory minimum of 1
Conference Room Cleaning, the extra Conference Room Cleaning can be replaced by 2 King-Sized
Room Cleanings or 4 Single Room Cleanings. The same logic applies to King-Sized Room
Cleanings too.\par\addvspace{\smallskipamount}
\promptboxhead\plabel{SCHEDULE:}
The game lasts 11 weeks. Odd weeks (1, 3, 5, 7, 9, 11) are PAYMENT weeks --- you make a
payment (6 total). Even weeks (2, 4, 6, 8, 10) are CLEANING weeks --- the cleaning provider
buys cleaning supplies and performs cleaning services (5 total).\par\addvspace{\smallskipamount}
\promptboxhead\plabel{CONTRACT EXECUTION ACTIONS:}
During contract execution, the only action at your disposal is making payments. Payments
must be non-negative multiples of \$100, including \$0.\par\addvspace{\smallskipamount}
The only actions at the cleaning provider's disposal are: (i) buying cleaning supplies, and
(ii) performing cleaning services.\par\addvspace{\smallskipamount}
There is no refund mechanism in this system.\par\addvspace{\smallskipamount}
There is no communication between you and the cleaning provider after a contract has been
agreed upon, hence no renegotiation.\par\addvspace{\smallskipamount}
\promptboxhead\plabel{OBJECTIVES:}
Your priority is maximizing your expected private utility, where the utility is your
remaining budget plus the cumulative private value of cleanings.\par\addvspace{\smallskipamount}
\promptboxhead\plabel{BEST ALTERNATIVE TO NO AGREEMENT:}
If no contract is agreed upon, you are guaranteed to receive your minimum required cleanings
in every cleaning round from an outside cleaning provider. You pay a total amount exactly
equal to the cumulative private value you assign to all of those cleanings, and that payment
is deducted from your budget.\par\addvspace{\smallskipamount}
\end{promptbox}

\clearpage
\textbf{Hotel Cleaning: Negotiation---Cleaning Provider.}
\label{app:prompt-hotel-cleaning-negotiation-cleaning-provider}
\promptboxhead
\begin{promptbox}{red}{Hotel Cleaning negotiation system prompt---Cleaning Provider}
\promptkey{GOAL}
\promptboxhead\plabel{NEGOTIATION PROCESS:}
You have at most 50 rounds to reach an agreement.\par\addvspace{\smallskipamount}
Each round has two turns: the hotel owner acts first, followed by the cleaning provider.\par\addvspace{\smallskipamount}
On your turn, you may send a message to the other party and optionally propose a contract.
You cannot formally accept a contract. If you agree with terms proposed by the hotel owner,
propose the complete agreed contract for the hotel owner to formally accept on their next
turn. The hotel owner's formal acceptance ends the negotiation with an agreement.\par\addvspace{\smallskipamount}
You may terminate the negotiation and walk away at any time. If either party terminates the
negotiation, or if no agreement is reached within 50 rounds, negotiations end with no
contract.\par\addvspace{\smallskipamount}
Negotiate a contract with the hotel owner, explicitly negotiating: (i) price per unit of
each cleaning type, (ii) cleaning schedule (how many of each cleaning type per cleaning
round over 5 cleaning rounds), and (iii) payment schedule (how much to pay in each of the 6
payment weeks).\par\addvspace{\smallskipamount}
Aim to negotiate a contract that maximizes your expected profit when executed during the
subsequent contract-execution phase.\par\addvspace{\smallskipamount}
\promptboxhead\plabel{IMPORTANT:}
Any contract you propose MUST explicitly state the price per unit for each cleaning type,
the exact quantity of each cleaning type per cleaning round, and the exact payment amount
for each payment week --- do not propose a contract that leaves any of these unspecified.\par\addvspace{\smallskipamount}
All quantities of each contracted cleaning type must be non-negative whole numbers. All
per-unit prices and payment amounts must be non-negative multiples of \$100.\par\addvspace{\smallskipamount}
Only propose a contract once you believe you have gathered sufficient information from the
negotiation to specify all of these terms.\par\addvspace{\smallskipamount}
\tcbline
\promptkey{PRIVATE INFO}
You are a cleaning provider negotiating a hotel cleaning contract with a hotel owner.\par\addvspace{\smallskipamount}
Your initial budget: \$2,000.\par\addvspace{\smallskipamount}
\promptboxhead\plabel{CLEANING-SERVICE REQUIREMENTS:}
You provide Conference Room Cleaning, King-Sized Room Cleaning, and Single Room Cleaning
using Detergent, Disinfectant, and Wet-wipes.\par\addvspace{\smallskipamount}
Each Conference Room Cleaning needs 1 unit of Wet-wipes + 1 unit of Detergent + 1 unit of
Disinfectant, each King-Sized Room Cleaning needs 1 unit of Detergent + 1 unit of
Disinfectant, each Single Room Cleaning needs 1 unit of Wet-wipes.\par\addvspace{\smallskipamount}
All cleaning-supply purchase quantities and cleaning-service quantities must be non-negative
whole numbers.\par\addvspace{\smallskipamount}
\promptboxhead\plabel{ENVIRONMENT STOCHASTICITY:}
(a) Prices: each cleaning round, each cleaning supply's price is independently set to one of
two values with equal probability (50/50) --- Detergent \$100 or \$300/unit, Disinfectant \$100
or \$200/unit, Wet-wipes \$200 or \$300/unit. (b) Delivery loss: for each cleaning supply, if
you order x units you receive between floor(x/2) and x units (uniformly at random). You pay
for the full amount ordered, not the amount received. For example, ordering 6 units of
Detergent means you pay for 6 but receive 3, 4, 5, or 6 with equal chance. (c) Supply
expiration: for each cleaning supply, inventory carried over from the previous cleaning
round independently loses 0, 1, or 2 units per cleaning round (equal probability, 1/3 each).
Supply expiration happens before you order new cleaning supplies. (d) Inventory cap: you can
store at most 10 units of each cleaning supply. Any received units that would exceed this
cap are lost --- you still pay for them. This cap applies before cleaning, so you cannot order
beyond capacity to use the surplus in the same round. Maximum order: 12 units of each
cleaning supply per cleaning round.\par\addvspace{\smallskipamount}
\promptboxhead\plabel{SCHEDULE:}
The game lasts 11 weeks. Odd weeks (1, 3, 5, 7, 9, 11) are PAYMENT weeks --- the hotel owner
pays you (6 total). Even weeks (2, 4, 6, 8, 10) are CLEANING weeks --- you buy cleaning
supplies and perform cleaning services (5 total).\par\addvspace{\smallskipamount}
\promptboxhead\plabel{CLEANING WEEK PROCESS (in order):}
(1) supply expiration occurs, (2) you order cleaning supplies (delivery loss applies), (3)
you perform cleaning services, (4) prices are re-randomized for the next cleaning round.\par\addvspace{\smallskipamount}
\promptboxhead\plabel{CONTRACT EXECUTION ACTIONS:}
During contract execution, your only actions are (i) buying cleaning supplies and (ii)
performing cleaning services.\par\addvspace{\smallskipamount}
The hotel owner's only action is making payments. Payments must be non-negative multiples of
\$100, including \$0.\par\addvspace{\smallskipamount}
There is no refund mechanism in this system.\par\addvspace{\smallskipamount}
Unused cleaning supplies carry over between cleaning rounds (subject to supply expiration).\par\addvspace{\smallskipamount}
There is no communication between you and the hotel owner after a contract has been agreed
upon, hence no possibility of renegotiation.\par\addvspace{\smallskipamount}
\promptboxhead\plabel{OBJECTIVES:}
Your priority is maximizing your expected profit, where profit is defined as the money you
earn from the hotel owner minus the total amount you spend buying cleaning supplies.\par\addvspace{\smallskipamount}
\promptboxhead\plabel{BEST ALTERNATIVE TO NO AGREEMENT:}
If no contract is agreed upon, you walk away with only your initial budget, \$2,000.\par\addvspace{\smallskipamount}
\end{promptbox}

\clearpage
\textbf{Hotel Cleaning: Performance---Hotel Owner.}
\label{app:prompt-hotel-cleaning-performance-hotel-owner}
\promptboxhead
\begin{promptbox}{blue}{Hotel Cleaning performance system prompt---Hotel Owner}
\promptkey{GOAL}
Perform the contract with the cleaning provider, choosing how much you wish to pay the
cleaning provider in every payment week. Aim to perform the contract in a way that maximizes
your private utility.\par\addvspace{\smallskipamount}
\tcbline
\promptkey{PRIVATE INFO}
You are a hotel owner who has negotiated a hotel cleaning contract with a cleaning provider.
Your total budget is \$20,000 --- this is your absolute limit and you cannot spend beyond it.
Here is the negotiated contract:\par\addvspace{\smallskipamount}
\begin{center}\ttfamily
\textless{}\textless{}\textless{} BEGIN CONTRACT \textgreater{}\textgreater{}\textgreater{}\\
\{\{NEGOTIATED\_CONTRACT\}\}\\
\textless{}\textless{}\textless{} END CONTRACT \textgreater{}\textgreater{}\textgreater{}
\end{center}
\promptboxhead\plabel{PRIVATE VALUATION:}
You privately value each Conference Room Cleaning at \$1,200, each King-Sized Room Cleaning
at \$600, and each Single Room Cleaning at \$300.\par\addvspace{\smallskipamount}
\promptboxhead\plabel{MINIMUM REQUIREMENTS:}
You need at least 1 Conference Room Cleaning, 1 King-Sized Room Cleaning, and 1 Single Room
Cleaning per cleaning round --- these minimums are mandatory and cannot be substituted. All
cleaning quantities and payments must be whole numbers (no fractions).\par\addvspace{\smallskipamount}
\promptboxhead\plabel{CLEANING-SERVICE SUBSTITUTABILITY:}
For any quantities ABOVE the minimum, cleaning services are substitutable at the rate: 1
Conference Room Cleaning = 2 King-Sized Room Cleanings = 4 Single Room Cleanings
(equivalently, 1 King-Sized Room Cleaning = 2 Single Room Cleanings). For example, if you
wanted 2 Conference Room Cleanings but the cleaning provider can provide only the mandatory
minimum of 1 Conference Room Cleaning, the extra Conference Room Cleaning can be replaced by
2 King-Sized Room Cleanings or 4 Single Room Cleanings. The same logic applies to King-Sized
Room Cleanings too.\par\addvspace{\smallskipamount}
\promptboxhead\plabel{SCHEDULE:}
The game lasts 11 weeks. Odd weeks (1, 3, 5, 7, 9, 11) are PAYMENT weeks --- you make a
payment (6 total). Even weeks (2, 4, 6, 8, 10) are CLEANING weeks --- the cleaning provider
buys cleaning supplies and performs cleaning services (5 total).\par\addvspace{\smallskipamount}
\promptboxhead\plabel{CONTRACT EXECUTION ACTIONS:}
During contract execution, the only action at your disposal is making payments. Payments
must be non-negative multiples of \$100, including \$0. There is no refund mechanism. The only
actions at the cleaning provider's disposal are: (i) buying cleaning supplies, and (ii)
performing cleaning services. No communication or renegotiation during execution.\par\addvspace{\smallskipamount}
\promptboxhead\plabel{OBJECTIVES:}
Your priority is maximizing your private utility, where the utility is your remaining budget
plus the cumulative private value of cleaning services.\par\addvspace{\smallskipamount}
\end{promptbox}

\clearpage
\textbf{Hotel Cleaning: Performance---Cleaning Provider.}
\label{app:prompt-hotel-cleaning-performance-cleaning-provider}
\promptboxhead
\begin{promptbox}{red}{Hotel Cleaning performance system prompt---Cleaning Provider}
\promptkey{GOAL}
Perform the contract with the hotel owner, choosing how many cleaning supplies of each kind
to purchase and how many cleaning services of each type to perform. Aim to perform the
contract in a way that maximizes your profit.\par\addvspace{\smallskipamount}
\tcbline
\promptkey{PRIVATE INFO}
You are a cleaning provider who has negotiated a hotel cleaning contract with a hotel owner.
Your initial budget: \$2,000. Here is the negotiated contract:\par\addvspace{\smallskipamount}
\begin{center}\ttfamily
\textless{}\textless{}\textless{} BEGIN CONTRACT \textgreater{}\textgreater{}\textgreater{}\\
\{\{NEGOTIATED\_CONTRACT\}\}\\
\textless{}\textless{}\textless{} END CONTRACT \textgreater{}\textgreater{}\textgreater{}
\end{center}
\promptboxhead\plabel{CLEANING-SERVICE REQUIREMENTS:}
You provide Conference Room Cleaning, King-Sized Room Cleaning, and Single Room Cleaning
using Detergent, Disinfectant, and Wet-wipes. Each Conference Room Cleaning needs 1 unit of
Wet-wipes + 1 unit of Detergent + 1 unit of Disinfectant, each King-Sized Room Cleaning
needs 1 unit of Detergent + 1 unit of Disinfectant, each Single Room Cleaning needs 1 unit
of Wet-wipes. All cleaning-supply purchase quantities and cleaning-service quantities must
be non-negative whole numbers.\par\addvspace{\smallskipamount}
\promptboxhead\plabel{ENVIRONMENT STOCHASTICITY:}
(a) Prices: each cleaning round, each cleaning supply's price is independently set to one of
two values with equal probability (50/50) --- Detergent \$100 or \$300/unit, Disinfectant \$100
or \$200/unit, Wet-wipes \$200 or \$300/unit. (b) Delivery loss: for each cleaning supply, if
you order x units you receive between floor(x/2) and x units (uniformly at random). You pay
for the full amount ordered, not the amount received. For example, ordering 6 units of
Detergent means you pay for 6 but receive 3, 4, 5, or 6 with equal chance. (c) Supply
expiration: for each cleaning supply, inventory carried over from the previous cleaning
round independently loses 0, 1, or 2 units per cleaning round (equal probability, 1/3 each).
Supply expiration happens before you order new cleaning supplies. (d) Inventory cap: you can
store at most 10 units of each cleaning supply. Any received units that would exceed this
cap are lost --- you still pay for them. This cap applies before cleaning, so you cannot order
beyond capacity to use the surplus in the same round. Maximum order: 12 units of each
cleaning supply per cleaning round.\par\addvspace{\smallskipamount}
\promptboxhead\plabel{SCHEDULE:}
The game lasts 11 weeks. Odd weeks (1, 3, 5, 7, 9, 11) are PAYMENT weeks --- the hotel owner
pays you (6 total). Even weeks (2, 4, 6, 8, 10) are CLEANING weeks --- you buy cleaning
supplies and perform cleaning services (5 total).\par\addvspace{\smallskipamount}
\promptboxhead\plabel{CLEANING WEEK PROCESS (in order):}
(1) random supply expiration occurs, (2) you order cleaning supplies (random delivery loss
applies), (3) you perform cleaning services, (4) prices are re-randomized for the next
cleaning round.\par\addvspace{\smallskipamount}
\promptboxhead\plabel{CONTRACT EXECUTION ACTIONS:}
During contract execution, your only actions are (i) buying cleaning supplies and (ii)
performing cleaning services. The hotel owner's only action is making payments. Payments
must be non-negative multiples of \$100, including \$0. There is no refund mechanism. Unused
cleaning supplies carry over between cleaning rounds (subject to supply expiration). No
communication or renegotiation during execution.\par\addvspace{\smallskipamount}
\promptboxhead\plabel{OBJECTIVES:}
Your priority is maximizing your profit, where profit is defined as the money you earn from
the hotel owner minus the total amount you spend buying cleaning supplies.\par\addvspace{\smallskipamount}
\end{promptbox}

\clearpage
\textbf{AI Model Hosting: Negotiation---AI Startup Founder.}
\label{app:prompt-ai-model-hosting-negotiation-ai-startup-founder}
\promptboxhead
\begin{promptbox}{blue}{AI Model Hosting negotiation system prompt---AI Startup Founder}
\promptkey{GOAL}
\promptboxhead\plabel{NEGOTIATION PROCESS:}
You have at most 50 rounds to reach an agreement.\par\addvspace{\smallskipamount}
Each round has two turns: the AI startup founder acts first, followed by the AI service
provider.\par\addvspace{\smallskipamount}
On your turn, you may send a message to the other party and optionally propose a contract.
You may formally accept the AI service provider's most recent contract. Accepting it ends
the negotiation with an agreement.\par\addvspace{\smallskipamount}
You may also terminate the negotiation and walk away at any time. If either party terminates
the negotiation, or if no agreement is reached within 50 rounds, negotiations end with no
contract.\par\addvspace{\smallskipamount}
Negotiate a contract with the AI service provider, explicitly negotiating: (i) price per
unit of each service type, (ii) service schedule (how many of each service type per service
round over 5 service rounds), and (iii) payment schedule (how much to pay in each of the 6
payment weeks).\par\addvspace{\smallskipamount}
Aim to negotiate a contract that maximizes your expected private utility when executed
during the subsequent contract-execution phase.\par\addvspace{\smallskipamount}
\promptboxhead\plabel{IMPORTANT:}
Any contract you propose or accept MUST explicitly state the price per unit for each service
type, the exact quantity of each service type per service round, and the exact payment
amount for each payment week --- do not accept a contract that leaves any of these
unspecified.\par\addvspace{\smallskipamount}
All quantities of each contracted service type must be non-negative whole numbers. All
per-unit prices and payment amounts must be non-negative multiples of \$1,000.\par\addvspace{\smallskipamount}
Only propose a contract once you believe you have gathered sufficient information from the
negotiation to specify all of these terms.\par\addvspace{\smallskipamount}
\tcbline
\promptkey{PRIVATE INFO}
You are an AI startup founder about to negotiate a service contract with an AI service
provider.\par\addvspace{\smallskipamount}
Your total budget is \$200,000 --- this is your absolute limit and you cannot spend beyond it.\par\addvspace{\smallskipamount}
\promptboxhead\plabel{PRIVATE VALUATION:}
You privately value each Fine-tuning service at \$12,000, each Model-hosting service at
\$6,000, and each Inference service at \$3,000. These are your maximum per-unit valuations.
You are willing to pay up to these amounts for each service, but you prefer to save money.\par\addvspace{\smallskipamount}
\promptboxhead\plabel{MINIMUM REQUIREMENTS:}
You need at least 1 Fine-tuning service, 1 Model-hosting service, and 1 Inference service
per service round --- these minimums are mandatory and cannot be substituted. All service
quantities must be whole numbers (no fractions).\par\addvspace{\smallskipamount}
\promptboxhead\plabel{SERVICE SUBSTITUTABILITY:}
For any quantities ABOVE the minimum, services are substitutable at the rate: 1 Fine-tuning
service = 2 Model-hosting services = 4 Inference services (equivalently, 1 Model-hosting
service = 2 Inference services). For example, if you wanted 2 Fine-tuning services but the
AI service provider can provide only the mandatory minimum of 1 Fine-tuning service, the
extra Fine-tuning service can be replaced by 2 Model-hosting services or 4 Inference
services. The same logic applies to Model-hosting services too.\par\addvspace{\smallskipamount}
\promptboxhead\plabel{SCHEDULE:}
The game lasts 11 weeks. Odd weeks (1, 3, 5, 7, 9, 11) are PAYMENT weeks --- you make a
payment (6 total). Even weeks (2, 4, 6, 8, 10) are SERVICE weeks --- the AI service provider
reserves resources and provides services (5 total).\par\addvspace{\smallskipamount}
\promptboxhead\plabel{CONTRACT EXECUTION ACTIONS:}
During contract execution, the only action at your disposal is making payments. Payments
must be non-negative multiples of \$1,000, including \$0.\par\addvspace{\smallskipamount}
The only actions at the AI service provider's disposal are: (i) reserving resources, and
(ii) providing services.\par\addvspace{\smallskipamount}
There is no refund mechanism in this system.\par\addvspace{\smallskipamount}
There is no communication between you and the AI service provider after a contract has been
agreed upon, hence no renegotiation.\par\addvspace{\smallskipamount}
\promptboxhead\plabel{OBJECTIVES:}
Your priority is maximizing your expected private utility, where the utility is your
remaining budget plus the cumulative private value of services.\par\addvspace{\smallskipamount}
\promptboxhead\plabel{BEST ALTERNATIVE TO NO AGREEMENT:}
If no contract is agreed upon, you are guaranteed to receive your minimum required services
in every service round from an outside AI service provider. You pay a total amount exactly
equal to the cumulative private value you assign to all of those services, and that payment
is deducted from your budget.\par\addvspace{\smallskipamount}
\end{promptbox}

\clearpage
\textbf{AI Model Hosting: Negotiation---AI Service Provider.}
\label{app:prompt-ai-model-hosting-negotiation-ai-service-provider}
\promptboxhead
\begin{promptbox}{red}{AI Model Hosting negotiation system prompt---AI Service Provider}
\promptkey{GOAL}
\promptboxhead\plabel{NEGOTIATION PROCESS:}
You have at most 50 rounds to reach an agreement.\par\addvspace{\smallskipamount}
Each round has two turns: the AI startup founder acts first, followed by the AI service
provider.\par\addvspace{\smallskipamount}
On your turn, you may send a message to the other party and optionally propose a contract.
You cannot formally accept a contract. If you agree with terms proposed by the AI startup
founder, propose the complete agreed contract for the AI startup founder to formally accept
on their next turn. The AI startup founder's formal acceptance ends the negotiation with an
agreement.\par\addvspace{\smallskipamount}
You may terminate the negotiation and walk away at any time. If either party terminates the
negotiation, or if no agreement is reached within 50 rounds, negotiations end with no
contract.\par\addvspace{\smallskipamount}
Negotiate a contract with the AI startup founder, explicitly negotiating: (i) price per unit
of each service type, (ii) service schedule (how many of each service type per service round
over 5 service rounds), and (iii) payment schedule (how much to pay in each of the 6 payment
weeks).\par\addvspace{\smallskipamount}
Aim to negotiate a contract that maximizes your expected profit when executed during the
subsequent contract-execution phase.\par\addvspace{\smallskipamount}
\promptboxhead\plabel{IMPORTANT:}
Any contract you propose MUST explicitly state the price per unit for each service type, the
exact quantity of each service type per service round, and the exact payment amount for each
payment week --- do not propose a contract that leaves any of these unspecified.\par\addvspace{\smallskipamount}
All quantities of each contracted service type must be non-negative whole numbers. All
per-unit prices and payment amounts must be non-negative multiples of \$1,000.\par\addvspace{\smallskipamount}
Only propose a contract once you believe you have gathered sufficient information from the
negotiation to specify all of these terms.\par\addvspace{\smallskipamount}
\tcbline
\promptkey{PRIVATE INFO}
You are an AI service provider negotiating a service contract with an AI startup founder.\par\addvspace{\smallskipamount}
Your initial budget: \$20,000.\par\addvspace{\smallskipamount}
\promptboxhead\plabel{SERVICE REQUIREMENTS:}
You provide Fine-tuning, Model-hosting, and Inference services using CPU, Memory, and GPU
resources.\par\addvspace{\smallskipamount}
Each Fine-tuning service needs 1 unit of GPU + 1 unit of CPU + 1 unit of Memory, each
Model-hosting service needs 1 unit of CPU + 1 unit of Memory, and each Inference service
needs 1 unit of GPU.\par\addvspace{\smallskipamount}
All resource-reservation and service-provision quantities must be non-negative whole
numbers.\par\addvspace{\smallskipamount}
\promptboxhead\plabel{ENVIRONMENT STOCHASTICITY:}
(a) Prices: each service round, each resource price is independently set to one of two
values with equal probability (50/50) --- CPU \$1,000 or \$3,000/unit, Memory \$1,000 or
\$2,000/unit, GPU \$2,000 or \$3,000/unit. (b) Capacity shortfall: for each resource, if you
reserve x units you receive between floor(x/2) and x units (uniformly at random) because
some capacity may be preempted or allocated elsewhere. You pay for the full amount reserved,
not the amount received. For example, reserving 6 units of CPU means you pay for 6 but
receive 3, 4, 5, or 6 with equal chance. (c) Resource expiration: capacity carried over from
the previous service round independently loses 0, 1, or 2 units per resource per service
round (equal probability, 1/3 each). Resource expiration happens before you reserve new
capacity. (d) Resource cap: you can retain at most 10 units of each resource. Any received
units that would exceed this cap are lost --- you still pay for them. This cap applies before
providing services, so you cannot reserve beyond capacity to use the surplus in the same
round. Maximum reservation: 12 units of each resource per service round.\par\addvspace{\smallskipamount}
\promptboxhead\plabel{SCHEDULE:}
The game lasts 11 weeks. Odd weeks (1, 3, 5, 7, 9, 11) are PAYMENT weeks --- the AI startup
founder pays you (6 total). Even weeks (2, 4, 6, 8, 10) are SERVICE weeks --- you reserve
resources and provide services (5 total).\par\addvspace{\smallskipamount}
\promptboxhead\plabel{SERVICE WEEK PROCESS (in order):}
(1) resource expiration occurs, (2) you reserve resources (capacity shortfall applies), (3)
you provide services, (4) prices are re-randomized for the next service round.\par\addvspace{\smallskipamount}
\promptboxhead\plabel{CONTRACT EXECUTION ACTIONS:}
During contract execution, your only actions are (i) reserving resources and (ii) providing
services.\par\addvspace{\smallskipamount}
The AI startup founder's only action is making payments. Payments must be non-negative
multiples of \$1,000, including \$0.\par\addvspace{\smallskipamount}
There is no refund mechanism in this system.\par\addvspace{\smallskipamount}
Unused resources carry over between service rounds (subject to resource expiration).\par\addvspace{\smallskipamount}
There is no communication between you and the AI startup founder after a contract has been
agreed upon, hence no possibility of renegotiation.\par\addvspace{\smallskipamount}
\promptboxhead\plabel{OBJECTIVES:}
Your priority is maximizing your expected profit, where profit is defined as the money you
earn from the AI startup founder minus the total amount you spend reserving resources.\par\addvspace{\smallskipamount}
\promptboxhead\plabel{BEST ALTERNATIVE TO NO AGREEMENT:}
If no contract is agreed upon, you walk away with only your initial budget, \$20,000.\par\addvspace{\smallskipamount}
\end{promptbox}

\clearpage
\textbf{AI Model Hosting: Performance---AI Startup Founder.}
\label{app:prompt-ai-model-hosting-performance-ai-startup-founder}
\promptboxhead
\begin{promptbox}{blue}{AI Model Hosting performance system prompt---AI Startup Founder}
\promptkey{GOAL}
Perform the contract with the AI service provider, choosing how much you wish to pay the AI
service provider in every payment week. Aim to perform the contract in a way that maximizes
your private utility.\par\addvspace{\smallskipamount}
\tcbline
\promptkey{PRIVATE INFO}
You are an AI startup founder who has negotiated a service contract with an AI service
provider. Your total budget is \$200,000 --- this is your absolute limit and you cannot spend
beyond it. Here is the negotiated contract:\par\addvspace{\smallskipamount}
\begin{center}\ttfamily
\textless{}\textless{}\textless{} BEGIN CONTRACT \textgreater{}\textgreater{}\textgreater{}\\
\{\{NEGOTIATED\_CONTRACT\}\}\\
\textless{}\textless{}\textless{} END CONTRACT \textgreater{}\textgreater{}\textgreater{}
\end{center}
\promptboxhead\plabel{PRIVATE VALUATION:}
You privately value each Fine-tuning service at \$12,000, each Model-hosting service at
\$6,000, and each Inference service at \$3,000.\par\addvspace{\smallskipamount}
\promptboxhead\plabel{MINIMUM REQUIREMENTS:}
You need at least 1 Fine-tuning service, 1 Model-hosting service, and 1 Inference service
per service round --- these minimums are mandatory and cannot be substituted. All service
quantities and payments must be whole numbers (no fractions).\par\addvspace{\smallskipamount}
\promptboxhead\plabel{SERVICE SUBSTITUTABILITY:}
For any quantities ABOVE the minimum, services are substitutable at the rate: 1 Fine-tuning
service = 2 Model-hosting services = 4 Inference services (equivalently, 1 Model-hosting
service = 2 Inference services). For example, if you wanted 2 Fine-tuning services but the
AI service provider can provide only the mandatory minimum of 1 Fine-tuning service, the
extra Fine-tuning service can be replaced by 2 Model-hosting services or 4 Inference
services. The same logic applies to Model-hosting services too.\par\addvspace{\smallskipamount}
\promptboxhead\plabel{SCHEDULE:}
The game lasts 11 weeks. Odd weeks (1, 3, 5, 7, 9, 11) are PAYMENT weeks --- you make a
payment (6 total). Even weeks (2, 4, 6, 8, 10) are SERVICE weeks --- the AI service provider
reserves resources and provides services (5 total).\par\addvspace{\smallskipamount}
\promptboxhead\plabel{CONTRACT EXECUTION ACTIONS:}
During contract execution, the only action at your disposal is making payments. Payments
must be non-negative multiples of \$1,000, including \$0. There is no refund mechanism. The
only actions at the AI service provider's disposal are: (i) reserving resources, and (ii)
providing services. No communication or renegotiation during execution.\par\addvspace{\smallskipamount}
\promptboxhead\plabel{OBJECTIVES:}
Your priority is maximizing your private utility, where the utility is your remaining budget
plus the cumulative private value of services.\par\addvspace{\smallskipamount}
\end{promptbox}

\clearpage
\textbf{AI Model Hosting: Performance---AI Service Provider.}
\label{app:prompt-ai-model-hosting-performance-ai-service-provider}
\promptboxhead
\begin{promptbox}{red}{AI Model Hosting performance system prompt---AI Service Provider}
\promptkey{GOAL}
Perform the contract with the AI startup founder, choosing how many resources of each kind
to reserve and how many services of each type to provide. Aim to perform the contract in a
way that maximizes your profit.\par\addvspace{\smallskipamount}
\tcbline
\promptkey{PRIVATE INFO}
You are an AI service provider who has negotiated a service contract with an AI startup
founder. Your initial budget: \$20,000. Here is the negotiated contract:\par\addvspace{\smallskipamount}
\begin{center}\ttfamily
\textless{}\textless{}\textless{} BEGIN CONTRACT \textgreater{}\textgreater{}\textgreater{}\\
\{\{NEGOTIATED\_CONTRACT\}\}\\
\textless{}\textless{}\textless{} END CONTRACT \textgreater{}\textgreater{}\textgreater{}
\end{center}
\promptboxhead\plabel{SERVICE REQUIREMENTS:}
You provide Fine-tuning, Model-hosting, and Inference services using CPU, Memory, and GPU
resources. Each Fine-tuning service needs 1 unit of GPU + 1 unit of CPU + 1 unit of Memory,
each Model-hosting service needs 1 unit of CPU + 1 unit of Memory, and each Inference
service needs 1 unit of GPU. All resource-reservation and service-provision quantities must
be non-negative whole numbers.\par\addvspace{\smallskipamount}
\promptboxhead\plabel{ENVIRONMENT STOCHASTICITY:}
(a) Prices: each service round, each resource price is independently set to one of two
values with equal probability (50/50) --- CPU \$1,000 or \$3,000/unit, Memory \$1,000 or
\$2,000/unit, GPU \$2,000 or \$3,000/unit. (b) Capacity shortfall: for each resource, if you
reserve x units you receive between floor(x/2) and x units (uniformly at random) because
some capacity may be preempted or allocated elsewhere. You pay for the full amount reserved,
not the amount received. For example, reserving 6 units of CPU means you pay for 6 but
receive 3, 4, 5, or 6 with equal chance. (c) Resource expiration: capacity carried over from
the previous service round independently loses 0, 1, or 2 units per resource per service
round (equal probability, 1/3 each). Resource expiration happens before you reserve new
capacity. (d) Resource cap: you can retain at most 10 units of each resource. Any received
units that would exceed this cap are lost --- you still pay for them. This cap applies before
providing services, so you cannot reserve beyond capacity to use the surplus in the same
round. Maximum reservation: 12 units of each resource per service round.\par\addvspace{\smallskipamount}
\promptboxhead\plabel{SCHEDULE:}
The game lasts 11 weeks. Odd weeks (1, 3, 5, 7, 9, 11) are PAYMENT weeks --- the AI startup
founder pays you (6 total). Even weeks (2, 4, 6, 8, 10) are SERVICE weeks --- you reserve
resources and provide services (5 total).\par\addvspace{\smallskipamount}
\promptboxhead\plabel{SERVICE WEEK PROCESS (in order):}
(1) resource expiration occurs, (2) you reserve resources (capacity shortfall applies), (3)
you provide services, (4) prices are re-randomized for the next service round.\par\addvspace{\smallskipamount}
\promptboxhead\plabel{CONTRACT EXECUTION ACTIONS:}
During contract execution, your only actions are (i) reserving resources and (ii) providing
services. The AI startup founder's only action is making payments. Payments must be
non-negative multiples of \$1,000, including \$0. There is no refund mechanism. Unused
resources carry over between service rounds (subject to resource expiration). No
communication or renegotiation during execution.\par\addvspace{\smallskipamount}
\promptboxhead\plabel{OBJECTIVES:}
Your priority is maximizing your profit, where profit is defined as the money you earn from
the AI startup founder minus the total amount you spend reserving resources.\par\addvspace{\smallskipamount}
\end{promptbox}

\clearpage

\subsection{Agent Scaffold and Inference Protocol}
\label{app:agent-scaffold}

All LLM experiments instantiate Claude Opus~5, Gemini~3.6 Flash, and
GPT-5.6-Sol as agents in Concordia using the same scaffold.  The static role
prompt is the concatenation of the role's \emph{Goal} and \emph{Private Info}
fields reproduced in Appendix~\ref{app:negotiation-prompts}; environment parameters
are substituted before a run according to
Appendix~\ref{app:benchmark-hyperparams}.  At each decision, the scaffold adds
a request generated from the current game state. An agent thus has access to its own private information, information about the game state (depending on the role's observation limits) and all the public-facing negotiation/performance history from the start of the interaction. Reasoning traces are recorded for logging.

Rather than reconstructing each prompt from a bounded observation buffer, the
scaffold maintains an independent, append-only provider conversation for each
agent. Each new observation is appended once and remains in the conversation
history, leaving a stable prompt prefix that enables provider prompt caching;
the five observations most relevant to the current request are additionally
recalled.

Table~\ref{tab:inference-settings} reports the provider-specific deliberation
configuration.

\begin{table}[!ht]
\centering
\small
\caption{\textbf{Inference configuration used for the evaluated LLM agents.}  No
temperature or top-$p$ override was supplied.}
\label{tab:inference-settings}
\begin{tabularx}{0.96\textwidth}{l>{\raggedright\arraybackslash}X>{\raggedright\arraybackslash}X}
\toprule
\textbf{Model} & \textbf{Deliberation configuration} & \textbf{Output limit} \\
\midrule
Claude Opus~5
& Adaptive thinking; effort set to \texttt{high}
& 16{,}384 tokens \\
Gemini~3.6 Flash
& Thinking level set to \texttt{high}
& Provider default \\
GPT~5.6 Sol
& Reasoning effort set to \texttt{high}
& Provider default \\
\bottomrule
\end{tabularx}
\end{table}

\textbf{Negotiation Scaffolding.}

Negotiation lasts for at most 50 rounds, with a Customer turn followed by a
Supplier turn.  The Customer's opening request asks it to begin the
negotiation.  Later requests identify the current round, the counterparty's
latest message, and, when present, the most recent formal contract together
with its proposer and proposal round.  The request then lists the actions
available in that state: continue the dialogue, propose or revise a contract,
terminate, or, for the Customer only, accept an eligible Supplier proposal.  A
proposal replaces the previously active proposal.  Acceptance succeeds only
when the Customer accepts the Supplier's formal proposal from the immediately
preceding round; proposing and accepting in the same response is not permitted.
Either party may terminate immediately, and reaching the round limit without a
valid acceptance yields disagreement.

Dialogue and formal proposals are visible to both parties, whereas each
agent's private prompt, private information, and reasoning traces are hidden
from the counterparty.

The negotiation suffix requires a JSON object with a string
\texttt{response}.  A formal proposal additionally uses a string
\texttt{contract}; \texttt{terminate: true} ends negotiation; and
\texttt{accept: true} is available only to the Customer.  The suffix also fixes
the conversational register and states that a week is the smallest unit of time,
so all delivery and payment terms must be specified per week.  The expected answer
takes one of the following JSON forms:
\begin{verbatim}
{"response": "<dialogue>", "contract": "<contract>"}
{"response": "<dialogue>", "accept": true}
{"response": "<dialogue>", "terminate": true}
\end{verbatim}
Malformed negotiation outputs are conservatively recovered when possible;
otherwise, they are treated as ordinary dialogue and cannot trigger a
proposal, acceptance, or termination. Empty responses are retried once, after
which the run is marked invalid.

\textbf{Performance Scaffolding.}

Main performance experiments instantiate a fresh LLM agent for each synthetic
contract game; the agent is not shown a negotiation transcript.  The complete
executed contract is inserted into its private information at the marked
location in Appendix~\ref{app:negotiation-prompts}.

In each payment week, the Customer receives the week number, total paid,
remaining budget, and a payment request.  In each production week, the
Supplier receives its current balance, post-spoilage inventory,
and realized input prices in a purchase request.  After stochastic delivery,
it receives its updated inventory in a production request.  Outcome
observations record payments, deliveries, inventory changes, and cumulative
production for use in later decisions.

During performance, the Customer returns one integer \texttt{payment}; the
Supplier returns one integer \texttt{buy\_\{resource\}} field per input and one integer
\texttt{produce\_\{product\}} field per output.  Payments must be nonnegative,
respect the remaining budget and the monetary granularity in
Appendix~\ref{app:benchmark-hyperparams}; purchase quantities are clipped to lie
between 0 and 12 and must be affordable; and production quantities must be nonnegative and
feasible from current inventory.  The Hotel Cleaning and AI Hosting runs use
the same schemas with the action verbs (\texttt{buy\_}/\texttt{perform\_} and
\texttt{reserve\_}/\texttt{provide\_}), domain labels, and monetary units mapped as in
Appendix~\ref{app:benchmark-hyperparams}.

The performance parser attempts the complete response and then the first flat
JSON object it contains.  Missing required fields or malformed JSON
receive corrective feedback and up to three attempts, after which the
corresponding action defaults to zero.  A well-formed but invalid payment,
unaffordable order, or infeasible production decision likewise receives
state-specific feedback for at most three action attempts before the zero
fallback.  Transient provider/API exceptions are retried up to five times with
exponential backoff; these transport retries do not alter the environment
state or add dialogue to an agent's history.

\pagebreak
\subsection{Prompt Guidance Clauses}
\label{app:ablation-prompts}

Figure~\ref{fig:prompt-guidance} toggles one additional instruction at a time; the control
adds nothing.  In negotiation, the following text was appended to the
\textsf{goal} field of both agents, except for the role-specific contingency
wording:
\begin{itemize}
    \item \textbf{Contingencies (Customer):} ``RISK: You should be wary of the
    risk of default from the Caterer. Consider adding contingency clauses.''
    \item \textbf{Contingencies (Caterer):} ``RISK: Be wary of the risk of
    default from the Customer and environment stochasticity. Consider adding
    contingency clauses.''
    \item \textbf{Net utility (both):} ``ACCEPTING CONTRACTS: You should only
    accept contracts if you expect them to result in a net increase in
    utility.''
    \item \textbf{No worse prior (both):} ``ACCEPTING CONTRACTS: You should not
    accept contracts that are worse for you than any contract that has been
    proposed to you before.''
\end{itemize}

In performance, the following text was appended to the \textsf{goal} field of
the focal LLM, paired with a rational baseline counterparty whose prompt was
unchanged---RCC for the first three clauses and RE for reciprocity:
\begin{itemize}
    \item \textbf{Prohibition:} ``CONTRACT PERFORMANCE: Avoid unilaterally
    violating the contract, and plan ahead so that you aren't in a position
    when you are forced to violate the terms.''
    \item \textbf{Repetition:} ``CONTRACT PERFORMANCE: You are
    in a repeated game; violating the contract may result in you losing out on
    future opportunities to make contracts with the same counter-party in the
    future.''
    \item \textbf{Planning:} ``CONTRACT PERFORMANCE: Plan ahead under
    uncertainty, and consider multiple possible outcomes.''
    \item \textbf{Reciprocity:} ``CONTRACT PERFORMANCE: Stop following the
    contract once it's clear that the other party is not following the
    contract.''
\end{itemize}

\clearpage
\section{Additional Methods}

\subsection{Rational Baselines}
\label{app:rational-baselines}

\paragraph{Scope and parseability.}
The rational baselines operate on the structured, parse-complete contract
constraint,
\begin{equation}
    C^\omega=(\mathbf{p}^d,\mathbf{q},\mathbf{M},\kappa).
\end{equation}
We write $C^\omega(\tau)\in\{0,1\}$ for its Boolean evaluation on trajectory
$\tau\in\mathcal T$.  Here
``parse-complete'' means that the contract specifies all three product prices,
all five production-week quantity vectors, and all six payment-week amounts in
types the solver accepts.  It does not mean the terms are economically
feasible or arithmetically consistent; we check and report those separately
rather than silently repairing them.  Every synthetic contract is
parse-complete by construction, while a natural-language proposal with
missing or ambiguous terms receives no solver configuration, and hence no RC,
RE, RCC, utility, or satisfaction estimate
(Appendix~\ref{app:contract-translation}).

\paragraph{Prefix constraints and absorbing violations.}
For the parse-complete contracts we consider, it is possible to decompose the contract constraint $C^\omega$ into role-specific
prefix indicators
$C^\omega_{i,w}(\tau_{\leq w})\in\{0,1\}$ for
$i\in\mathcal{I}'=\{\mathrm{Cust},\mathrm{Supp}\}$.  The indicator is active while role $i$ has met every obligation due through week
$w$.  These constraints hence have the absorbing-prefix property:
\begin{equation}
    C^\omega_{i,w}(\tau_{\leq w})=0
    \quad\Longrightarrow\quad
    C^\omega_{i,w'}(\widetilde\tau_{\leq w'})=0
    \quad
    \forall w'>w,
    \ \widetilde\tau_{\leq w'}\succeq\tau_{\leq w},
    \label{eq:absorbing-contract-constraint}
\end{equation}
where $\widetilde\tau_{\leq w'}\succeq\tau_{\leq w}$ denotes any extension
of the observed prefix.  Equivalently, the violation indicator
$1-C^\omega_{i,w}$ is nondecreasing.  A shortfall that an active deduction or
rollover remedy permits does not zero the indicator until that remedy's own
failure condition is met.  The joint terminal
constraint is
$C^\omega(\tau)=\prod_{i\in\mathcal{I}'}C^\omega_{i,L}(\tau_{\leq L})$.

\paragraph{Planning Solver.}
Given a parse-complete $\omega$, the contract constraint $C^\omega$, and the
target violation rate $\epsilon$, the Supplier faces a constrained Markov
decision process (cMDP) \cite{altman2021constrained}.  We solve it approximately by dynamic programming over
a restricted policy space. This program, which simulates a Supplier's planning, treats any breach of
$C^\omega$ as terminating the contract. This enforces the optimization objective outlined in
the main text (Eq.~\ref{eq:rc}).\\

\textbf{State, action, and transition.}  Index the production rounds by
$t\in\{1,\ldots,T\}$, with $T=(L-1)/2$ (numbered contiguously here, rather
than in the main text's alternating sequence).  The Supplier is paid in the
immediately preceding payment week
and observes spoilage of the inventory carried over from the previous
production round.  The initial state before production is
\begin{equation}
    s_t=(\mathbf n_t,\beta_t,\mu_t),
    \label{eq:rc-supplier-state}
\end{equation}
with $\mathbf n_t$ the three-input inventory vector, $\beta_t$ the integer cash
balance, and $\mu_t$ the contingency memory --- a singleton in the base solve
(Appendix~\ref{app:contingency-extensions}).  The Supplier then observes the
week's input prices $\mathbf p'_t\sim\nu_t$, with
$\nu_t=\bigotimes_{x\in\mathcal X}\mathsf P_x$ and the prime distinguishing
them from the contract's product prices $\mathbf p^d$, and chooses an
inventory order $\mathbf o_t\in\{0,\ldots,12\}^3$ it can afford out of
$\beta_t$.  That order is the entire base action: obligations are exact, so
delivering the promised $\mathbf q_t$ in full is the only outcome that avoids
a breach.  We drop the subscript $t$ where it is clear.\\

\textbf{Transition kernels.}  Let $\zeta_t$ collect the round's receipt and
spoilage draws, one pair per input, where $t$ indexes the production weeks
$W_{\mathrm{prod}}$ in order.
These draws are the environment's own randomness of Section~\ref{app:env-spec},
independent across inputs and rounds,
\begin{equation}
    \zeta_t\sim\textstyle\bigotimes_{x\in\mathcal X}
      \mathsf R_x(\,\cdot\mid o_{t,x})\otimes\mathsf S_x .
    \label{eq:rc-supplier-noise}
\end{equation}
Conditional on $\zeta_t$, the round is deterministic: receipts are added to
inventory and capped, production consumes the input quantities that $\omega$
and the recipes fix, and spoilage is applied to what remains.  Write
$F_t^\omega(s,\mathbf p',\mathbf o,\zeta)$ for the resulting next live state and
$\mathbf n(\cdot)$ for its inventory block.  The Supplier pays for what it
orders rather than what arrives, so the balance moves deterministically to
$\beta-\mathbf p'\cdot\mathbf o$ and, on a surviving round, $\zeta$ reaches the
next state only through the inventory.  The round kernel is therefore an
expectation over the environment's draws and, by their independence, a product
across inputs:
\begin{equation}
    \Pr\nolimits_t(\mathbf n^{\mathrm{next}}\mid s,\mathbf p',\mathbf o)
    =\mathbb E_{\zeta_t}\bigl[\mathbf{1}\left[\mathbf n\!\left(F_t^\omega(s,\mathbf p',\mathbf o,\zeta_t)\right)=\mathbf n^{\mathrm{next}}\right]\bigr]
    =\prod_{x\in\mathcal X}\Pr\nolimits_{t,x}(n^{\mathrm{next}}_x\mid n_x,o_x).
    \label{eq:rc-supplier-factorization}
\end{equation}
Breach is decided by the receipt draw alone: the round is lost as soon as
$n_x$ plus receipts falls short of what the recipes require for $\mathbf q_t$
for some input.  The solver retains only the surviving block of each
$\Pr_{t,x}$, so the kernels are sub-stochastic and their missing row mass is
exactly the per-round breach probability.\\

\textbf{Bellman recursions.}  Let
$\chi_t^\omega(s,\mathbf p',\mathbf o,\zeta)\in\{0,1\}$ indicate that the
transition stays within the contract constraint and that all required
balance transitions are affordable.  The base solve realizes $\chi_t^\omega$
through the sub-stochastic kernel of
Eq.~\ref{eq:rc-supplier-factorization} rather than evaluating it
separately, but writing it out keeps the recursion in the form the contingency
extensions reuse.  For every state, price, and affordable order, the solver
computes an action-continuation satisfaction probability and value,
\begin{align}
    \widehat P_t(s,\mathbf p',\mathbf o)
    &=\sum_{\zeta}
      \Pr(\zeta\mid s,\mathbf p',\mathbf o)\,
      \chi_t^\omega(s,\mathbf p',\mathbf o,\zeta)\,
      P_{t+1}\!\left(F_t^\omega(s,\mathbf p',\mathbf o,\zeta)\right),
      \label{eq:rc-bellman-psat}\\
    \widehat V_{\mathrm{Supp},t}(s,\mathbf p',\mathbf o)
    &=\sum_{\zeta}
      \Pr(\zeta\mid s,\mathbf p',\mathbf o)
      \left[
        R_t^\omega(s,\mathbf p',\mathbf o,\zeta)
        +\chi_t^\omega(s,\mathbf p',\mathbf o,\zeta)
         V_{\mathrm{Supp},t+1}\!\left(F_t^\omega(s,\mathbf p',\mathbf o,\zeta)\right)
      \right],
      \label{eq:rc-bellman-value}
\end{align}
where $R_t^\omega=M_t-\mathbf p'\cdot\mathbf o$ is the current-round Supplier
payoff. Write
\begin{equation}
    \mathcal A_t^{\mathrm{ord}}(s,\mathbf p')
    =\left\{\mathbf o\in\{0,\ldots,12\}^3:
      \beta-\mathbf p'\cdot\mathbf o\geq 0\right\}
    \label{eq:rc-supplier-orders}
\end{equation}
for the orders affordable at $(s,\mathbf p')$, let
$\bar p_{\mathrm{sat}}=1-\epsilon=0.95$, and define the statewise admissible
set
\begin{equation}
    \mathcal A_t^{\mathrm{sat}}(s,\mathbf p')
    =\left\{\mathbf o\in\mathcal A_t^{\mathrm{ord}}(s,\mathbf p'):
      \widehat P_t(s,\mathbf p',\mathbf o)
      \geq\bar p_{\mathrm{sat}}\right\}.
    \label{eq:rc-supplier-admissible}
\end{equation}
The Supplier policy then applies a lexicographic selection rule,
\begin{equation}
    \pi_{\mathrm{Supp},t}(s,\mathbf p')
    \in
    \begin{cases}
      \displaystyle\arg\max_{\mathbf o\in\mathcal A_t^{\mathrm{sat}}(s,\mathbf p')}
          \widehat V_{\mathrm{Supp},t}(s,\mathbf p',\mathbf o),
          &\mathcal A_t^{\mathrm{sat}}(s,\mathbf p')\neq\varnothing,\\[7pt]
      \displaystyle\arg\max_{\mathbf o\in\mathcal A_t^{\mathrm{ord}}(s,\mathbf p')}
          \widehat P_t(s,\mathbf p',\mathbf o),
          &\mathcal A_t^{\mathrm{sat}}(s,\mathbf p')=\varnothing.
    \end{cases}
    \label{eq:rc-supplier-selection}
\end{equation}
The solver takes the highest-utility order among those that satisfy the
contract with high probability.  Where no order clears the threshold, it falls
back to the most reliable order rather than to an unconstrained profit
maximizer.\\

A constrained stochastic-control problem can admit a randomized policy whose
mixture of orders gives a better utility--reliability trade-off than any
deterministic policy \cite{altman2021constrained}; we do not optimize over
such mixtures, seeking a sufficiently reliable, high-value deterministic policy
instead.  Averaging the selected actions over prices gives the pre-price
arrays
\begin{align}
    P_t(s)
    &=\sum_{\mathbf p'}\nu_t(\mathbf p')\,
      \widehat P_t\!\left(
        s,\mathbf p',\pi_{\mathrm{Supp},t}(s,\mathbf p')
      \right),\\
    V_{\mathrm{Supp},t}(s)
    &=\sum_{\mathbf p'}\nu_t(\mathbf p')\,
      \widehat V_{\mathrm{Supp},t}\!\left(
        s,\mathbf p',\pi_{\mathrm{Supp},t}(s,\mathbf p')
      \right).
    \label{eq:rc-supplier-preprice}
\end{align}

\textbf{Customer utility.}  The same backward recursion carries two
auxiliary arrays that yield the expected Customer utility without a separate
program.  A delivery value $V_{\mathrm{Cust},t}(s)$ tracks the expected
product value the Customer receives; a cumulative payment
$M'_{\mathrm{paid},t}(s)$ tracks the expected total transferred from round~$t$
onwards, primed to separate it from the realized per-week transfer
$M_w^{\mathrm{paid}}$ of Section~\ref{sec:benchmark}.  Both ride the same
sub-stochastic kernels, so only contract-satisfying paths contribute, and
their action-level recursions parallel
Eqs.~\ref{eq:rc-bellman-psat}--\ref{eq:rc-bellman-value}:
\begin{align}
    \widehat V_{\mathrm{Cust},t}(s,\mathbf p',\mathbf o)
    &= \bar\chi_t(s,\mathbf p',\mathbf o)\!\sum_{d\in D}q_{t,d}\,v_d
      \;+\!\sum_{\zeta}
        \Pr(\zeta\mid s,\mathbf p',\mathbf o)\,
        \chi_t^\omega\,
        V_{\mathrm{Cust},t+1}\!\bigl(F_t^\omega(\cdot)\bigr),
    \label{eq:rc-customer-delivery}\\
    \widehat M'_{\mathrm{paid},t}(s,\mathbf p',\mathbf o)
    &= M_t
      \;+\!\sum_{\zeta}
        \Pr(\zeta\mid s,\mathbf p',\mathbf o)\,
        \chi_t^\omega\,
        M'_{\mathrm{paid},t+1}\!\bigl(F_t^\omega(\cdot)\bigr),
    \label{eq:rc-customer-paid}
\end{align}
where $\bar\chi_t(s,\mathbf p',\mathbf o)
=\sum_{\zeta}\Pr(\zeta\mid s,\mathbf p',\mathbf o)\,
\chi_t^\omega(s,\mathbf p',\mathbf o,\zeta)$ is the probability that round~$t$
alone succeeds: delivery value accrues only when the round is served, while
the preceding scheduled payment~$M_t$ is transferred on every path that
reaches round~$t$.  Terminal values are
$V_{\mathrm{Cust},T+1}\!\equiv\!0$ and
$M'_{\mathrm{paid},T+1}\!\equiv\!M_L$.  Both arrays are averaged
over the selected action and prices as in
Eq.~\ref{eq:rc-supplier-preprice}.  At the initial state~$s_0$,
\begin{equation}
    U_{\mathrm{Cust}}(\omega)
    = B + V_{\mathrm{Cust},1}(s_0) - M'_{\mathrm{paid},1}(s_0),
    \label{eq:rc-customer-utility}
\end{equation}
where $B$ is the Customer's budget.\\

Algorithm~\ref{alg:rc-supplier} gives the complete
recursion at the level of operations.

\begin{algorithm}[t]
\small
\renewcommand{\algorithmicrequire}{\textbf{Input:}}
\renewcommand{\algorithmicensure}{\textbf{Output:}}
\caption{\textbf{Exact backward dynamic program for the Supplier order table.}}
\label{alg:rc-supplier}
\begin{algorithmic}[1]
\Require kernels $\Pr_{t,x}$, price laws $\nu_t$, payments $M_t$,
  product values $(v_d)_{d\in D}$, threshold $\bar p_{\mathrm{sat}}$
\Ensure order table $\pi_{\mathrm{Supp}}$, arrays
  $(P_1,V_{\mathrm{Supp},1},V_{\mathrm{Cust},1},M'_{\mathrm{paid},1})$
\State $P_{T+1}\!\equiv\!1$,\; $V_{\mathrm{Supp},T+1}\!\equiv\!M_L$,\;
  $V_{\mathrm{Cust},T+1}\!\equiv\!0$,\;
  $M'_{\mathrm{paid},T+1}\!\equiv\!M_L$
  \Comment{terminal boundary}
\For{$t=T,T-1,\ldots,1$}
  \For{each price $\mathbf p'$ in the support of $\nu_t$}
    \Comment{prices precede ordering}
    \State Keep the affordable orders
      $\{\mathbf o:\beta-\mathbf p'\cdot\mathbf o\geq 0\}$, each leaving
      balance $\beta^{\mathrm{next}}=\beta-\mathbf p'\cdot\mathbf o$
    \State Evaluate $\widehat P_t,\widehat V_{\mathrm{Supp},t},
      \widehat V_{\mathrm{Cust},t},\widehat M'_{\mathrm{paid},t}$
      at $\beta^{\mathrm{next}}$
      \Comment{Eqs.~\ref{eq:rc-supplier-factorization}--\ref{eq:rc-customer-paid}}
    \State Select one order per state: best $\widehat V_{\mathrm{Supp},t}$ among the orders
      clearing $\bar p_{\mathrm{sat}}$, else best $\widehat P_t$
      \Comment{Eq.~\ref{eq:rc-supplier-selection}}
  \EndFor
  \State $(P_t,V_{\mathrm{Supp},t},V_{\mathrm{Cust},t},M'_{\mathrm{paid},t})
    \gets\mathbb{E}_{\mathbf p'\sim\nu_t}$ of the selected actions
    \Comment{Eq.~\ref{eq:rc-supplier-preprice}}
\EndFor
\State \Return the order table indexed by
  $(t,\mathbf n_t,\beta_t,\mu_t,\mathbf p'_t)$, with
  $P_{\mathrm{sat}}=P_1(s_0)$,\;
  $U_{\mathrm{Supp}}(\omega)=V_{\mathrm{Supp},1}(s_0)$,\\
  \hspace*{2.5em}and
  $U_{\mathrm{Cust}}(\omega)=B+V_{\mathrm{Cust},1}(s_0)
    -M'_{\mathrm{paid},1}(s_0)$
  \Comment{Eq.~\ref{eq:rc-customer-utility}}
\end{algorithmic}
\end{algorithm}

\paragraph{Contingency extensions.}
\label{app:contingency-extensions}
The base solve above assumes grim-trigger termination. Three further clauses are elective, and a contract may carry any combination of them (Appendix~\ref{app:sample-contract} reproduces one
carrying all three). These contingencies mildly affect the Bellman recursions above by modifying the conditions of contract compliance and in some cases, the future payment to be received by the Supplier.

\begin{itemize}
\item \textbf{Substitution} lets surplus units of one product cover a shortfall
in another, at the rate $\boldsymbol\lambda=(4,2,1)$ over
$(O_1,O_2,O_3)$, above a minimum $\underline{\mathbf q}$ that must arrive as the
named product. Thus, the dish production should seek to satisfy the effective "dish-value" due for the round, and ensure that non-substitutable minimums get produced.
\begin{equation}
    \chi_t^{\mathrm{sub}}
    =\mathbf{1}\!\left[\mathbf y_t\geq\underline{\mathbf q}\right]\,
      \mathbf{1}\!\left[
        \boldsymbol\lambda\cdot\!\left(\mathbf y_t-\underline{\mathbf q}\right)
        \geq
        \boldsymbol\lambda\cdot\!\left(\mathbf q_t-\underline{\mathbf q}\right)
      \right].
\label{eq:contingency-substitution}
\end{equation}

\item \textbf{Payment deduction} lets the Customer withhold the value of a
shortfall instead of terminating, so the credited payment becomes
$M_t^{\mathrm{eff}}=\left(M_t-\mathbf p^d\cdot\boldsymbol\sigma_{t-1}\right)^{+}$,
with any excess forfeited.  Memory is the binary flag
$\mu_t=\mathbf{1}\!\left[\boldsymbol\sigma_{t-1}\neq\mathbf 0\right]$, that tracks if a shortfall was incurred in the round prior, which is forgiven (reset to 0) if the current round does not cause a shortfall.
\begin{equation}
    \chi_t^{\mathrm{ded}}
    =\mathbf{1}\!\left[
        \mu_t=0\ \ \vee\ \ \boldsymbol\sigma_t=\mathbf 0
      \right],
\label{eq:contingency-deduction}
\end{equation}

\item \textbf{Rollover} carries a shortfall into the next production week
instead of breaching.  The deficit is
the memory, $\mu_{t+1}=\boldsymbol\sigma_t$, and round $t$ owes
$\mathbf q_t+\mu_t$, so
$\boldsymbol\sigma_t=\left(\mathbf q_t+\mu_t-\mathbf y_t\right)^{+}$.
\begin{equation}
    \chi_t^{\mathrm{roll}}
    =\mathbf{1}\!\left[
        \boldsymbol\sigma_t=\mathbf 0
        \ \ \vee\ \
        \left(
          \mu_t=\mathbf 0
          \ \wedge\
          \max_{d\in D}\sigma_{t,d}\leq\bar\sigma
          \ \wedge\
          t<T
        \right)
      \right].
\label{eq:contingency-rollover}
\end{equation}
\end{itemize}

\subsection{Baseline Runtime Policies}
\label{app:baseline-runtime}

We now describe the rational baselines as implemented for the contract performance phase analysis.

\begin{itemize}
\item \textbf{RC Customer.}
The Customer's only decision is the amount transferred, and
$\pi^{\mathrm{RC}}_{\mathrm{Cust}}$ is read straight off the contract: the
scheduled amount $M_w$ less any pending deduction credit $\delta_w$.
$\delta_w$ is nonzero only under payment deduction, where it holds the value
of the Supplier's most recent shortfall and expires unused if it exceeds
$M_w$.  Otherwise the Customer pays every scheduled amount regardless of what
the Supplier delivered, matching the reference policy
$\pi^{\mathrm{ref}}_{\mathrm{Cust}}$ the order table is solved against.
\begin{equation}
    M^{\mathrm{RC}}_w=M_w-\min\{\delta_w,M_w\}.
    \label{eq:rc-customer-payment}
\end{equation}

\item \textbf{RC Supplier.}
Each production week the RC Supplier looks up its order in the solved table
$\pi_{\mathrm{Supp}}$ of Algorithm~\ref{alg:rc-supplier}, which it adopts as
$\pi^{\mathrm{RC}}_{\mathrm{Supp}}$, indexed by the state
$(\mathbf n,\beta,\mu)$ of
Eq.~\ref{eq:rc-supplier-state} and the observed ingredient prices.  It then
produces whatever the active clause still counts as compliant, exactly as
$\chi_t$ defines it: it delivers in full when it can, and otherwise takes a
shortfall only where the clause forgives one rather than terminating the
contract.

\item \textbf{RE Customer.}
The RE Customer approximates Eq.~\ref{eq:re} by unconditional non-payment, which is exact against a naively compliant Supplier with sufficient capital (Proposition~\ref{prop:customer-re}). 
We note that non-compliance can also comprise policies that admit partial compliance before recommending unilateral defection. Our version of non-compliance has been deliberately chosen to see if agents are capable of recognizing and acting against exploitative behavior.
\begin{equation}
    \widetilde\pi^{\mathrm{RE}}_{\mathrm{Cust}}: M_w^{\mathrm{paid}}=0.
    \label{eq:implemented-re-cust}
\end{equation}

\item \textbf{RE Supplier.}
The RE Supplier likewise orders and produces nothing, a best response to a naively compliant Customer (Proposition~\ref{prop:supplier-re}).
\begin{equation}
    \widetilde\pi^{\mathrm{RE}}_{\mathrm{Supp}}:
    \mathbf o_w=\mathbf 0,\ \mathbf q'_w=\mathbf 0.
    \label{eq:implemented-re-supp}
\end{equation}

\item \textbf{RCC Customer.}
An RCC Customer makes the RC payments of Eq.~\ref{eq:rc-customer-payment}
until an observed Supplier failure sets its breach flag, then pays zero for
the rest of the contract (Eq.~\ref{eq:implemented-re-cust}).  Both RCC
baselines track this flag as the indicator $D_{i,t}$ of Eq.~\ref{eq:rcc},
indexed by week $w$ and read off the role-specific prefix constraint,
\begin{equation}
    D_{i,w}
    =\mathbf{1}\!\left[
       \exists k<w:
       C^\omega_{-i,k}(\tau_{\leq k})=0
      \right],
    \label{eq:implemented-rcc-breach}
\end{equation}
so it is set the first time the counterparty's delivered or paid amount falls
short of $C^\omega$, using only violations observed before the current action.
What counts as short is exactly what the active clause's indicator
$\chi_t$ rejects (Appendix~\ref{app:contingency-extensions}).  Since the
prefix constraint $C^\omega_{-i,w}$ is absorbing
(Eq.~\ref{eq:absorbing-contract-constraint}), the flag switches on at most
once, at the case split of Eq.~\ref{eq:rcc}.

\item \textbf{RCC Supplier.}
An RCC Supplier follows the RC order table until the Customer underpays, then
orders and produces zero for the rest of the contract
(Eq.~\ref{eq:implemented-re-supp}).
\end{itemize}

\subsection{Non-Engagement as an RE Policy and RE as the Best Response to RE}
\label{app:proof-supplier-re}

Let $\pi^{\mathrm{ref}}_{\mathrm{Cust}}$ be a naively compliant Customer that
pays every scheduled amount, $M^{\mathrm{paid}}_w=M_w$ for all
$w\in W_{\mathrm{pay}}$, regardless of what the Supplier delivers.

\begin{proposition}
\label{prop:supplier-re}
For any contract $\omega$, the non-engagement policy
$\widetilde\pi^{\mathrm{RE}}_{\mathrm{Supp}}$
(Eq.~\ref{eq:implemented-re-supp}) is a best response to
$\pi^{\mathrm{ref}}_{\mathrm{Cust}}$, i.e.\ it solves Eq.~\ref{eq:re}.
\end{proposition}
\begin{proof}
Against $\pi^{\mathrm{ref}}_{\mathrm{Cust}}$, payments do not depend on the
history, so the Supplier's revenue $\sum_{w\in W_{\mathrm{pay}}}M_w$ is the
same constant under every $\pi_{\mathrm{Supp}}\in\Pi_{\mathrm{Supp}}$. Its
utility is
\[
U_{\mathrm{Supp}}(\tau)=\sum_{w\in W_{\mathrm{pay}}}M_w
-\sum_{w\in W_{\mathrm{prod}}}\sum_{x\in\mathcal X}o_{w,x}\,p_{w,x}
\;\le\;\sum_{w\in W_{\mathrm{pay}}}M_w,
\]
because $o_{w,x}\ge 0$ and $p_{w,x}\ge 0$. Ordering nothing is always
feasible and reaches this bound on every trajectory, so it attains the bound
in expectation and maximizes Eq.~\ref{eq:re}.
\end{proof}

Symmetrically, let $\pi^{\mathrm{ref}}_{\mathrm{Supp}}$ be a naively compliant
Supplier whose initial capital covers every scheduled order without Customer
payments, so the distribution of deliveries $q'_{w,d}$ is independent of the
Customer's policy.

\begin{proposition}
\label{prop:customer-re}
For any contract $\omega$, the non-payment policy
$\widetilde\pi^{\mathrm{RE}}_{\mathrm{Cust}}$
(Eq.~\ref{eq:implemented-re-cust}) is a best response to
$\pi^{\mathrm{ref}}_{\mathrm{Supp}}$, i.e.\ it solves Eq.~\ref{eq:re}.
\end{proposition}
\begin{proof}
The Customer's expected utility is
\[
\mathbb{E}\!\left[U_{\mathrm{Cust}}(\tau)\right]
=B+\mathbb{E}\!\left[\sum_{w\in W_{\mathrm{prod}}}\sum_{d\in D}v_d\,q'_{w,d}\right]
-\mathbb{E}\!\left[\sum_{w\in W_{\mathrm{pay}}}M^{\mathrm{paid}}_w\right].
\]
The first two terms are the same under every
$\pi_{\mathrm{Cust}}\in\Pi_{\mathrm{Cust}}$, and $M^{\mathrm{paid}}_w\ge 0$,
so expected utility is maximized by $M^{\mathrm{paid}}_w=0$ for all $w$.
\end{proof}
In low-capital environments, the RC Supplier's orders depend on payments
received, so withholding payment can reduce deliveries and
Proposition~\ref{prop:customer-re} holds only approximately.

Finally, non-engagement is a best response to itself, which justifies RCC's
switch to RE once the counterparty defects (Eq.~\ref{eq:rcc}).

\begin{proposition}
\label{prop:re-vs-re}
For any contract $\omega$ and any history, $\widetilde\pi^{\mathrm{RE}}_{\mathrm{Supp}}$
is a best response to $\widetilde\pi^{\mathrm{RE}}_{\mathrm{Cust}}$ over the
remaining weeks, and vice versa.
\end{proposition}
\begin{proof}
Against $\widetilde\pi^{\mathrm{RE}}_{\mathrm{Cust}}$, every remaining
$M^{\mathrm{paid}}_w=0$, so the Supplier's remaining utility is
$-\sum_{w}\sum_{x\in\mathcal X}o_{w,x}\,p_{w,x}\le 0$, attained by
$\mathbf o_w=\mathbf 0$. Against $\widetilde\pi^{\mathrm{RE}}_{\mathrm{Supp}}$,
no ingredients are ordered, so every remaining $q'_{w,d}=0$ under every
Customer policy, and the Customer's remaining utility
$-\sum_{w}M^{\mathrm{paid}}_w\le 0$ is maximized by
$M^{\mathrm{paid}}_w=0$.
\end{proof}

\subsection{Concessionary Negotiation Baseline}
\label{app:concession-baseline}

We implement Eqs.~\ref{eq:concession-target}--\ref{eq:concession-offer-accept}
with $\beta=5$ and $\delta=10$ for both agents. Each agent uses only its own
type $\theta_i$. Contracts are non-contingent, $\omega=(\mathbf M,\mathbf q)$,
with payments $M_w$ for $w\in W_{\mathrm{pay}}$ and per-round quantities $q_d$
repeated in every production week (the $F_1$ space without prices, which the
base solver ignores).

\paragraph{Customer.}
Assuming full delivery,
\begin{equation}
    \hat U_{\mathrm{Cust}}(\omega)
    =B-\textstyle\sum_{w}M_w+|W_{\mathrm{prod}}|\sum_{d}v_d\,q_d,
    \label{eq:concession-customer-utility}
\end{equation}
subject to $\sum_w M_w\leq B$ and $q_d\geq q_d^{\min}$. It opens with zero
payments and minimum quantities, and $\underline\eta_{\mathrm{Cust}}$ is the
environment's outside-option utility.

\paragraph{Supplier.}
It orders, per input, the smallest $\hat o_x\leq o^{\max}$ whose expected receipt
covers the recipe need $\sum_d a_{d,x}q_d$ (else $o^{\max}$), prices it at mean input prices
$\bar p_x$, and ignores spoilage and carryover:
\begin{equation}
    \hat U_{\mathrm{Supp}}(\omega)=\textstyle\sum_w M_w-|W_{\mathrm{prod}}|\,\bar c,
    \qquad \bar c=\sum_x\hat o_x\,\bar p_x .
    \label{eq:concession-supplier-utility}
\end{equation}
It also requires $\hat P_{\mathrm{sat}}\geq0.7$, the probability that every
round's order is affordable from capital $K$ plus payments received so far,
less the expected cost of earlier rounds:
\begin{equation}
    \hat P_{\mathrm{sat}}(\omega)=\prod_{k=1}^{|W_{\mathrm{prod}}|}
      \Pr_{\mathbf p'\sim\nu}\!\Bigl[\textstyle\sum_x\hat o_x\,p'_x\leq
      K+\sum_{j\leq k}M_{w_j}-(k-1)\,\bar c\Bigr].
    \label{eq:concession-supplier-psat}
\end{equation}
It proposes at most \$300 in total, opens with \$250 split randomly across
payment weeks and no products, and has $\underline\eta_{\mathrm{Supp}}=0$.

\paragraph{Counter-offers.}
An agent samples 10{,}000 perturbations of the standing offer (payments
$\pm\$30$, clipped at 0; quantities $\pm4$, clipped to $[q_d^{\min},5]$ or
$[0,5]$) and proposes one uniformly from those in its target band and hard
constraints, or repeats its previous offer if none qualify.

\paragraph{Protocol.}
Offers alternate up to $T=100$ turns (50 per agent, matching the LLMs' 50 two-turn rounds). The Customer opens at $t=0$ and moves at even
$t$; the Supplier opens at $t=1$ and moves at odd $t$. Only the Customer's
acceptance ends the negotiation: a Supplier that would accept re-proposes the
offer, which the Customer then always accepts, since it was the Customer's own
offer and clears its lower target. With no deal by $T$, both take their outside
options. We run seeds 42--44 per environment and score agreed contracts with
the exact solver (Appendix~\ref{app:rational-baselines}).

\paragraph{Frontier $F_3$.}
We sweep one agent's outside option at a time in steps of 10 (1 in
Environment~5) up to its largest utility on $F_1$, and down to 20 below its
true value for the Supplier (50 for the Customer in Environment~5). If the
option exceeds $\bar\eta_i$, $\bar\eta_i$ is raised to it. $F_3$ is the
nondominated set of agreed contracts with $P_{\mathrm{sat}}\geq0.95$.

\subsection{Rationality of Contract Compliance under Repetition}
\label{app:proof-grim-trigger-compliance}

In the main text, we note that contractual compliance by following the RC policy in a single episode can be made individually and sequentially rational under a variety of institutional incentives (e.g. repetition, reputation, third-party enforcement). Here, we give specific conditions for the case of repetition.

Consider a performance game $\mathcal{P}$ that two agents
enter after agreeing to a contract $\omega \in \Omega$. The agents know
that they will interact repeatedly; that is, they will play infinitely
many episodes of $\mathcal{P}$. Let $\mathcal{P}_{\mathrm{rep}}$ denote
this infinitely repeated game. Assume the agreed contract admits
well-defined episode-level compliance and defection policies for each
player, and that both players discount future episodes by a common
factor $\delta \in (0,1)$. We then have the following result.

\begin{theorem}
\label{thm:grim-trigger-compliance}
    Consider an infinitely repeated episodic performance game
    $\mathcal{P}_{\mathrm{rep}}$ induced by a contract $\omega \in
    \Omega$. After each episode $k$, assume that public, deterministic
    monitoring reveals whether either player violated the contract. For each
    player $i$, let $D_{i,k}\in\{0,1\}$ denote the violation signal observed
    by $i$, where $D_{i,k}=1$ if and only if player $-i$ violated $\omega$ in
    episode $k$. Let $\pi^{\mathrm{comp}}_i$ be an episode-level compliance
    policy that is optimal among player $i$'s compliant policies against
    $\pi^{\mathrm{comp}}_{-i}$. Let $\Pi^{\mathrm{def}}_i$ be the set of
    player $i$'s episode-level defection policies, each of which produces a
    publicly detected violation with certainty, and choose
    \[
        \pi_i^{\mathrm{def}}
        \in
        \arg\max_{\pi_i\in\Pi^{\mathrm{def}}_i}
        \mathbb{E}\!\left[
            u_i(\tau;\theta_i)
            \mid \omega,\pi_i,\pi_{-i}^{\mathrm{comp}}
        \right].
    \]

    Assume that mutual defection is a credible threat; that is, for each player $i$ and every alternative policy $\pi_i$
    \[
       \mathbb{E}\!\left[
            u_i(\tau;\theta_i)
            \mid \omega,\pi_i^{\mathrm{def}},\pi_{-i}^{\mathrm{def}}\right] \geq \mathbb{E}\!\left[
            u_i(\tau;\theta_i)
            \mid \omega,\pi_i,\pi_{-i}^{\mathrm{def}}
        \right].
    \]

    Define the episode-level payoffs as:
    \[
        R_i
        :=
        \mathbb{E}\!\left[
            u_i(\tau;\theta_i)
            \mid \omega,\pi_i^{\mathrm{comp}},\pi_{-i}^{\mathrm{comp}}
        \right],
    \]
    \[
        T_i
        :=
        \mathbb{E}\!\left[
            u_i(\tau;\theta_i)
            \mid \omega,\pi_i^{\mathrm{def}},\pi_{-i}^{\mathrm{comp}}
        \right],
    \]
    \[
        P_i
        :=
        \mathbb{E}\!\left[
            u_i(\tau;\theta_i)
            \mid \omega,\pi_i^{\mathrm{def}},\pi_{-i}^{\mathrm{def}}
        \right].
    \]
    Suppose $T_i > R_i > P_i$ for each $i \in \{1,2\}$.
    Then there exists a threshold patience factor
    $\delta^{*}(\mathcal{P}_{\mathrm{rep}}) < 1$ such that for all
    $\delta \geq \delta^{*}(\mathcal{P}_{\mathrm{rep}})$, the repeated
    game admits a subgame-perfect equilibrium in which both players
    comply with the contract $\omega$ on the equilibrium path.
\end{theorem}

\begin{proof}
    Consider the grim-trigger strategy under which both players comply
    in every episode until the first episode $k$ in which $D_{j,k}=1$ for
    some $j\in\{1,2\}$, after which both switch forever to the
    mutual-defection punishment profile. By the
    one-shot deviation principle for sequential rationality
    \cite{hendon1996one}, it suffices to rule out a profitable
    deviation in a single episode after any compliant history.

    Fix a player $i$. Under compliance, the continuation value is
    \[
        V_i^{\mathrm{comp}}
        =
        \frac{R_i}{1-\delta}.
    \]
    Any one-episode compliant deviation is unprofitable by the assumed
    optimality of $\pi_i^{\mathrm{comp}}$. Among defection policies,
    $\pi_i^{\mathrm{def}}$ yields the maximum one-episode payoff $T_i$.
    Deterministic detection then triggers the mutual-defection punishment
    path, so the maximum defection value is
    \[
        V_i^{\mathrm{def}}
        =
        T_i + \frac{\delta P_i}{1-\delta}.
    \]
    Compliance is optimal whenever
    $V_i^{\mathrm{comp}} \geq V_i^{\mathrm{def}}$, i.e.,
    \[
        \frac{R_i}{1-\delta}
        \geq
        T_i + \frac{\delta P_i}{1-\delta}.
    \]
    Rearranging gives
    \[
        \delta
        \geq
        \delta_i^*
        :=
        \frac{T_i - R_i}{T_i - P_i}.
    \]
    Since $T_i > R_i > P_i$, we have $0 < \delta_i^* < 1$. Define
    \[
        \delta^{*}(\mathcal{P}_{\mathrm{rep}})
        :=
        \max_{i \in \{1,2\}} \delta_i^*.
    \]
    Then for every $\delta \geq \delta^{*}(\mathcal{P}_{\mathrm{rep}})$,
    neither player has a profitable one-episode deviation after any
    compliant history. The public violation signals $D_{i,k},D_{-i,k}$ at
    episode boundaries and credibility of mutual defection therefore make the
    grim-trigger continuation sequentially rational,
    so the resulting profile is subgame-perfect and sustains compliance
    on the equilibrium path.
\end{proof}

\subsection{Rationality of Contract Compliance under Resolute Choice}
\label{app:resolute}

Theorem~\ref{thm:grim-trigger-compliance} shows that institutions such as repetition can make contract compliance rational. However, there is an intuitive sense in which it is rational to comply with contracts --- and more generally, avoid breaking promises --- even within a single episode and without institutional support: if other agents know that a particular agent is the type to break contracts, they would not agree to contracting with that agent at all. Thus, to incentivize others to agree to a contract, it can be rational to commit to complying with it.

This justification for compliance can be made formal for agents that make decisions using \emph{resolute choice}~\cite{mcclennen1990rationality}, rather than the standard assumption of sequential rationality. A resolute agent is an agent that is capable of carrying through its commitments: it evaluates complete plans of action at the point it adopts them, and then carries out the adopted plan, even at later decision points where deviating would yield a higher continuation payoff.

\paragraph{Resolute agents comply to incentivize agreement.}
We show that contract compliance is rational under resolute choice in a two-stage \emph{acceptance-performance game}, which simplifies the full negotiation-performance game to the core elements we need. Consider a contract $\omega$ that has already been negotiated between players $i$ and $-i$ (Figure~\ref{fig:acceptance-performance-game}). First, $-i$ decides whether to accept $\omega$. Rejecting ends the game with the disagreement payoffs $(u_i(\bot), u_{-i}(\bot))$. Accepting commits $-i$ to performing $\omega$ with its RC policy $\pi^{\mathrm{RC}}_{-i}$. Second, if $-i$ accepted, $i$ chooses either to comply, playing $\pi^{\mathrm{RC}}_i$, or to exploit, playing the RE policy $\pi^{\mathrm{RE}}_i$. As in Theorem~\ref{thm:grim-trigger-compliance}, let $R_i, R_{-i}$ denote the players' expected payoffs when both comply, and let $T_i, S_{-i}$ denote their expected payoffs when $i$ exploits $-i$. Since $\pi^{\mathrm{RE}}_i$ is a best response to $\pi^{\mathrm{RC}}_{-i}$, we always have $T_i \geq R_i$: once $-i$ has accepted, exploiting is at least as profitable for $i$ as complying.

\begin{figure}[h]
\centering
\begin{tikzpicture}[
  dec/.style={circle, draw, thick, minimum size=7mm, inner sep=0pt},
  term/.style={font=\small, align=center},
  lbl/.style={font=\small, inner sep=2pt}
]
  \node[dec] (a) at (0,0) {$-i$};
  \node[term] (rej) at (-2.4,-1.6) {$\bigl(u_i(\bot),\,u_{-i}(\bot)\bigr)$};
  \node[dec] (b) at (2.0,-1.6) {$i$};
  \node[term] (rc) at (0.2,-3.3) {$\bigl(R_i,\,R_{-i}\bigr)$};
  \node[term] (re) at (3.8,-3.3) {$\bigl(T_i,\,S_{-i}\bigr)$};
  \draw[thick] (a) -- node[lbl, above left] {Reject $\omega$} (rej);
  \draw[thick] (a) -- node[lbl, above right] {Accept $\omega$} (b);
  \draw[thick] (b) -- node[lbl, above left] {Comply} (rc);
  \draw[thick] (b) -- node[lbl, above right] {Exploit} (re);
\end{tikzpicture}
\caption{\textbf{An acceptance-performance game for contract $\omega$}. Player $-i$ first decides whether to accept the negotiated contract $\omega$; if it accepts, player $i$ then decides whether to comply with $\omega$ or exploit $-i$. Payoffs are listed as (player $i$, player $-i$).}
\label{fig:acceptance-performance-game}
\end{figure}

\begin{theorem}
\label{thm:resolute-rc}
Suppose that:
\begin{enumerate}
  \item[(i)] \emph{the contract is mutually beneficial}: both players prefer mutual compliance to no deal, i.e., $R_i > u_i(\bot)$ and $R_{-i} > u_{-i}(\bot)$;
  \item[(ii)] \emph{exploitation is worse than no deal}: being exploited leaves $-i$ worse off than not contracting, i.e.,
  $S_{-i} < u_{-i}(\bot)$; and
  \item[(iii)] \emph{plans are transparent}: $-i$ knows which plan $i$ has adopted before deciding whether to accept.
\end{enumerate}
Then a resolute player $i$ commits to complying with $\omega$, $-i$ accepts the contract, and both players obtain their mutual-compliance payoffs $(R_i, R_{-i})$. This holds even though $i$ could gain by exploiting $-i$ once the contract has been accepted ($T_i \geq R_i$).
\end{theorem}

\begin{proof}
A resolute player $i$ chooses, before $-i$ moves, between the plans ``comply'' and ``exploit'', and then follows the chosen plan. If $i$ adopts ``comply'', then by (iii) $-i$ anticipates payoff $R_{-i}$ from accepting, which exceeds $u_{-i}(\bot)$ by (i), so $-i$ accepts and $i$ receives $R_i$. If $i$ adopts ``exploit'', then by (iii) $-i$ anticipates payoff $S_{-i}$ from accepting, which is below $u_{-i}(\bot)$ by (ii), so $-i$ rejects and $i$ receives $u_i(\bot)$. Since $R_i > u_i(\bot)$ by (i), ``comply'' is the unique optimal plan for $i$, which it then carries out.
\end{proof}

Theorem \ref{thm:resolute-rc} shows that a resolute agent can secure the gains of trade by planning to comply and treating its own plan as binding. This acts as a commitment device without external enforcement, in the spirit of Gauthier's constrained maximization~\cite{gauthier1990rationality}. In place of the threat of future punishment, we assume that $-i$ can anticipate $i$'s plan before accepting. This can be a reasonable assumption whenever there is past conduct or other evidence about what $i$ will plan to do in this episode, or when $-i$ can simulate how $i$ will behave because they are both AI agents~\cite{kovarik2023game}. As such, RC can be understood as a rational standard for contracting, even without the need for ex-post institutional incentives. 

We note that condition (ii) may not hold if player $i$'s exploitation strategy is relatively weak (e.g. partial compliance with just enough deliveries or payments to make $-i$ better off). Under weak exploitation, player $-i$ might still accept even knowing that it will be exploited (specifically, when $S_{-i} \geq u_{-i}(\bot)$). In such a case, planning to exploit yields $T_i \geq R_i$, and resolute choice no longer favors compliance. Thus, while resolute choice can rule out strong forms of exploitation as irrational, institutional incentives may still be required in practice to sustain the rationality of full compliance over partial compliance.

\paragraph{Resolute conditional compliance remains rational under the threat of exploitation.}
Theorem~\ref{thm:resolute-rc} explains why player $i$ should comply, but the acceptance-performance game (Figure \ref{fig:acceptance-performance-game}) takes player $-i$'s compliance for granted if they accept the contract. In practice, however, $-i$ is also a self-interested agent that may accept the contract and then exploit $i$, and an unconditional complier is vulnerable to such exploitation. Under this possibility, is compliance still rational? We now show that by making compliance \emph{conditional} --- that is, by following the RCC policy --- we can preserve the rationality of compliance. Furthermore, this incentivizes \emph{joint compliance} by both players $-i$ and $i$.

We extend the game in Figure \ref{fig:acceptance-performance-game} so that $-i$ also commits to how it performs any contract it accepts, which we call a \emph{commitment-performance-game} (Figure~\ref{fig:commitment-performance-game}). In the first stage, $-i$ chooses to \emph{reject} $\omega$, to \emph{accept and comply} (playing RC), or to \emph{accept and exploit} (playing RE). In the second stage, $i$ responds. If $-i$ complied, $i$ can choose to \emph{comply} (RC) or to \emph{defect} (RE). If $-i$ exploited, $i$ can choose to \emph{comply} anyway (RC) or to \emph{reciprocally defect} ($\mathrm{BR}_i(\pi^{\mathrm{RE}}_{-i})$, as in the RCC policy). This last branch abstracts away from the sequential structure of performance (in which $i$ observes $-i$'s violations and can respond to them) into a single response stage.

A plan for $i$ now specifies a response on each branch. RE defects on both branches, while unconditional compliance, RC, complies on both branches. RCC complies if $-i$ complies and reciprocally defects if $-i$ exploits. In addition to the payoffs from before, we use the payoffs $T_{-i}$ and $S_i$ when $-i$ exploits a complying $i$, and $P_i$ and $P_{-i}$ when $-i$ exploits and $i$ reciprocally defects. Since RE is a best response to RC, we have $T_{-i} \geq R_{-i}$: an unconditional complier can always be exploited at least weakly profitably.

\begin{figure}[h]
\centering
\begin{tikzpicture}[
  dec/.style={circle, draw, thick, minimum size=7mm, inner sep=0pt},
  term/.style={font=\small, align=center},
  lbl/.style={font=\small, inner sep=4pt}
]
  \node[dec] (a) at (0,0) {$-i$};
  \node[term] (rej) at (0,-2.2) {$\bigl(u_i(\bot),\,u_{-i}(\bot)\bigr)$};
  \node[dec] (b1) at (-3.3,-2.2) {$i$};
  \node[dec] (b2) at (3.3,-2.2) {$i$};
  \node[term] (cc) at (-4.6,-4.2) {$\bigl(R_i,\,R_{-i}\bigr)$};
  \node[term] (cd) at (-2.0,-4.2) {$\bigl(T_i,\,S_{-i}\bigr)$};
  \node[term] (ec) at (2.0,-4.2) {$\bigl(S_i,\,T_{-i}\bigr)$};
  \node[term] (ed) at (4.6,-4.2) {$\bigl(P_i,\,P_{-i}\bigr)$};
  \draw[thick] (a) -- node[lbl, above left] {\begin{tabular}{@{}c@{}}Accept $\omega$ +\\Comply\end{tabular}} (b1);
  \draw[thick] (a) -- node[lbl, right] {Reject $\omega$} (rej);
  \draw[thick] (a) -- node[lbl, above right] {\begin{tabular}{@{}c@{}}Accept $\omega$ +\\Exploit\end{tabular}} (b2);
  \draw[thick] (b1) -- node[lbl, left] {Comply} (cc);
  \draw[thick] (b1) -- node[lbl, right] {Defect} (cd);
  \draw[thick] (b2) -- node[lbl, left] {Comply} (ec);
  \draw[thick] (b2) -- node[lbl, right] {Defect} (ed);
\end{tikzpicture}
\caption{\textbf{A commitment-performance game for contract $\omega$.} Player $-i$ first decides whether to accept and comply, reject $\omega$, or accept and exploit; player $i$ then decides whether to comply or defect. The RE policy for player $i$ defects on both branches, RC complies on both branches, while RCC complies on the left branch and defects on the right. Payoffs are listed as (player $i$, player $-i$).}
\label{fig:commitment-performance-game}
\end{figure}

\begin{theorem}
\label{thm:resolute-rcc}
Suppose that:
\begin{enumerate}
  \item[(i)] \emph{the contract is mutually beneficial}: $R_i > u_i(\bot)$ and $R_{-i} > u_{-i}(\bot)$;
  \item[(ii)] \emph{exploitation hurts the exploited party}: being exploited leaves $-i$ worse off than no deal and $i$ worse off than mutual compliance, i.e., $S_{-i} < u_{-i}(\bot)$ and $S_i < R_i$;
  \item[(iii)] \emph{reciprocal defection deters exploitation}: if $i$
  reciprocally defects, exploiting $i$ is worse for $-i$ than complying, and mutual defection is worse for $i$ than mutual compliance, i.e., $P_{-i} < R_{-i}$ and $P_i < R_i$; and
  \item[(iv)] \emph{plans are transparent}: $-i$ knows which plan $i$ has adopted before making its first-stage choice.
\end{enumerate}
Then committing to the RCC policy is an optimal plan for a resolute player $i$. Under this plan, $-i$'s unique optimal choice is to accept and comply, so both players comply and obtain the mutual-compliance payoffs $(R_i, R_{-i})$. Moreover, if $-i$ would
strictly profit from exploiting an unconditional complier ($T_{-i} > R_{-i}$), then RC is irrational because it will be exploited, and RCC is a strict improvement that avoids exploitation.
\end{theorem}

\begin{proof}
If $i$ adopts $\pi^{\mathrm{RCC}}_i$, then by (iv) $-i$ compares rejecting, which yields $u_{-i}(\bot)$; accepting and complying, which yields $R_{-i}$; and accepting and exploiting, which yields $P_{-i}$. By (i) and (iii), $R_{-i}$ is the unique maximum, so $-i$ accepts and complies, and a resolute $i$ then complies as planned, giving $i$ the payoff $R_i$.

No plan gives $i$ more than $R_i$. If a plan has $i$ comply on the left branch, the possible outcomes for $i$ are $u_i(\bot)$, $R_i$, $S_i$, and $P_i$, all at most $R_i$ by (i)--(iii). If a plan has $i$ defect on the left branch (as in RE), then $-i$ never accepts and complies, since that would yield $S_{-i} < u_{-i}(\bot)$ by (ii), which is worse than rejecting. The remaining outcomes for $i$ are $u_i(\bot)$, $S_i$, and $P_i$, all below $R_i$ by (i)--(iii). Hence $\pi^{\mathrm{RCC}}_i$ is optimal.

Finally, if $i$ adopts $\pi^{\mathrm{RC}}_i$ and $T_{-i} > R_{-i}$, then accepting and exploiting yields $-i$ the payoff $T_{-i} > R_{-i} > u_{-i}(\bot)$, so $-i$ exploits and $i$ receives $S_i < R_i$ by (ii). Thus, RCC is strictly better than RC when $T_{-i} > R_{-i}$.
\end{proof}

Theorem~\ref{thm:resolute-rcc} shows that, even under the threat of exploitation, by complying \emph{conditionally}, a resolute $i$ secures $-i$'s acceptance and compliance with the contract. As a result, a commitment to conditional compliance becomes rational.  The two branches of RCC play complementary roles: complying when $-i$ complies makes the contract worth accepting, while reciprocally defecting when $-i$ exploits makes exploitation unprofitable. Joint conditional compliance is thus the outcome of rational contracting between resolute agents, which justifies our use of $\pi^{\mathrm{RCC}}$ to define the expected utility $U_i(\omega)$ of
a contract. Resolute choice is needed on the left branch: once $-i$ has accepted and complied, defecting yields $T_i \geq R_i$, so only an agent that treats its plan as binding will comply there. On the right branch, by contrast, reciprocal defection is typically also in $i$'s immediate interest.

Similar to our discussion of condition (ii) of Theorem \ref{thm:resolute-rc}, we note that condition (iii) of Theorem \ref{thm:resolute-rcc} may not hold for weaker forms of exploitation like late-stage contract breaches. If player $-i$ tries to exploit by breaching early in the contract, reciprocal defection by $i$ can result in enough losses (missed deliveries or payments) that player $-i$ would have been better off complying. But if player $-i$ mostly complies before exploiting towards the end, then reciprocal defection may not be a sufficient deterrent (i.e., $P_{-i} \geq R_{-i}$), since there are few payments or deliveries that player $i$ can withhold. 

\paragraph{Symmetric deterrence through participation.}
The foregoing arguments for compliance also apply with the roles of the players reversed. Player $i$, too, decides whether to enter the contract, so if $-i$'s plan is transparent to $i$, a resolute $-i$ must consider whether $i$ would still accept $\omega$ knowing how $-i$ intends to perform. Each player is thus deterred from breaching in two ways: ex ante, because a counterparty that foresaw the breach would refuse to contract, and ex post, because the counterparty's conditional compliance makes the breach unprofitable. We note that the first channel does not depend on when the breach occurs, so it can deter late breaches that reciprocal defection cannot punish, such as withholding a final payment after all deliveries have been made: a plan to commit such a breach is just as visible when the contract is signed as a plan to defect immediately. However, this channel only deters large enough breaches that would leave the victim worse off than not contracting. A smaller breach can escape the first channel, and a late breach can escape the second channel, so resolute choice does not necessarily rule out smaller contract violations. Again, this suggests that institutional incentives like repetition may be important for ensuring that weak forms of exploitation are irrational, not just strong ones.

\subsection{Natural-Language Contract Translation}
\label{app:contract-translation}

Contracts are evaluated through a separate, offline translation step that
maps the accepted natural-language agreement (and each formal proposal used
by the proposal-level metrics) into the structured contract constraint
$C^\omega=(\mathbf{p}^d,\mathbf{q},\mathbf{M},\kappa)$ described in the main
text and above.  This translator is
not one of the negotiating agents and does not receive either party's private
utilities or the realized environment trajectory.  We use Gemini~3.6 Flash
and require a JSON object conforming to the schema in
Table~\ref{tab:contract-translation-schema}.  Contract identifiers are passed
through unchanged so that every returned object can be matched to its source
text.

\paragraph{Contract-authoring instructions.} While natural language contracts may be arbitrarily incomplete, we leave the study of heavily incomplete contracts to future work, and instead guide agents to negotiate contracts that are sufficiently complete so that negotiation and performance failure cannot be attributed to contracts being highly under-specified. As such, our negotiation prompts (Appendix~\ref{app:negotiation-prompts}) instruct
both agents to negotiate the unit price of every dish, the delivery quantity
in every production round, and the payment in every payment week.  They also
state that any proposed contract---and any contract accepted by the
Customer---must give every one of those terms explicitly, using nonnegative
whole-number quantities and whole-dollar payments.  The action scaffold
further requires the complete agreement to appear in a dedicated
\texttt{contract} field.  These instructions explain why most model-authored
agreements are successfully parsed by our translator. Our translator nevertheless treats any input text as untrusted, and never fills in an ambiguous term by inference.

\begin{table}[!ht]
\centering
\small
\begin{tabularx}{0.96\textwidth}{l>{\raggedright\arraybackslash}X>{\raggedright\arraybackslash}X}
\toprule
\textbf{Field} & \textbf{Extracted representation} & \textbf{Normalization and validation} \\
\midrule
\texttt{dish\_prices}
& Unit prices for dishes A, B, and C
& Every price is numeric or null.  A, B, and C denote Margherita Pizza,
Pesto Pasta, and Tomato Soup, respectively. \\
\texttt{production\_schedule}
& Nonnegative integer quantities for A, B, and C in production weeks
2, 4, 6, 8, and 10
& Uniform quantities are expanded across weeks; omitted or ambiguous values
remain null. \\
\texttt{payment\_schedule}
& Payment amounts in weeks 1, 3, 5, 7, 9, and terminal week 11
& Uniform payments are expanded across weeks; omitted or ambiguous values
remain null. \\
\texttt{contingency\_set}
& Explicitly stated substitution, grim-trigger, payment-deduction, and
rollover clauses
& Only clauses stated in the text are extracted; implicit solver behavior is
recorded separately. \\
\texttt{contingency\_params}
& Minimum delivered quantities for substitution or deduction and the
rollover deficit limit
& Missing parameters receive only the deterministic benchmark defaults
described below, never an LLM-invented value. \\
\bottomrule
\end{tabularx}
\caption{\textbf{Structured schema used to translate natural-language contracts.}}
\label{tab:contract-translation-schema}
\end{table}

\paragraph{Exact extraction prompt.}
The following is the implemented first-attempt prompt.  The final placeholder
is replaced by a JSON object containing each source identifier and its raw
contract text.

\begin{lstlisting}[basicstyle=\ttfamily\scriptsize,breaklines=true,
                   columns=fullflexible,keepspaces=true,showstringspaces=false]
Extract structured numerical contract terms from these catering contracts.

Return ONLY valid JSON in this exact shape:
{
  "contracts": [
    {
      "contract_id": "<same id from input>",
      "dish_prices": {"A": <number|null>, "B": <number|null>, "C": <number|null>},
      "production_schedule": [
        {"week": 2, "A": <int|null>, "B": <int|null>, "C": <int|null>},
        {"week": 4, "A": <int|null>, "B": <int|null>, "C": <int|null>},
        {"week": 6, "A": <int|null>, "B": <int|null>, "C": <int|null>},
        {"week": 8, "A": <int|null>, "B": <int|null>, "C": <int|null>},
        {"week": 10, "A": <int|null>, "B": <int|null>, "C": <int|null>}
      ],
      "payment_schedule": [
        {"week": 1, "amount": <number|null>},
        {"week": 3, "amount": <number|null>},
        {"week": 5, "amount": <number|null>},
        {"week": 7, "amount": <number|null>},
        {"week": 9, "amount": <number|null>},
        {"week": 11, "amount": <number|null>}
      ],
      "contingency_set": ["<implemented_clause_name>"],
      "contingency_params": {
        "substitution": {"min_qty": {"A": int, "B": int, "C": int}},
        "payment_deduction": {"min_qty": {"A": int, "B": int, "C": int}},
        "rollover": {"max_deficit": int|null}
      }
    }
  ]
}

Mapping:
- A = Margherita Pizza
- B = Pesto Pasta
- C = Tomato Soup

Rules:
- Parse each input contract_id exactly once and copy its identifier exactly.
- `production_schedule` must contain exactly one row for each of weeks 2, 4, 6,
  8, and 10. Each row has `week`, `A`, `B`, and `C`; quantities are
  non-negative integers or null.
- `payment_schedule` must contain exactly one row for each of weeks 1, 3, 5, 7,
  9, and 11. Each row has `week` and `amount`; amounts are non-negative numbers
  or null.
- If the same quantities apply in weeks 2, 4, 6, 8, and 10, expand them.
- If a fixed payment applies to multiple payment weeks, expand it.
- Be especially careful that Week 11 is payment week 11, not week 1.
- Do not invent missing values. Use null when absent or ambiguous.
- Amounts and prices must be plain numbers, with no dollar signs.
- Only return the fields shown above.
- Extract only contingencies explicitly expressed by the contract. Do not add a
  default or implicit fallback contingency.
- Include every implemented contingency explicitly present in the contract.
  Every array entry must be exactly one of: "substitution", "grim_trigger",
  "payment_deduction", or "rollover". Do not return any other clause name.
- If the contract contains no explicit contingency, return `contingency_set: []`
  and `contingency_params: {}`.
- Use "grim_trigger" only when the contract explicitly states termination,
  suspension of future performance, or an equivalent trigger after breach.
- Use "substitution" when the contract explicitly allows some promised dishes
  above the required minimum quantities to be replaced with other dishes. Use
  expressly stated minimum quantities; if none are stated, use
  {"A": 1, "B": 1, "C": 1}.
- Use "payment_deduction" when a delivery shortfall results in a deduction from
  a future scheduled payment.
- Use "rollover" when missed dishes carry forward or must be cured later. If
  the contract states a maximum deficit, extract that number; otherwise use
  null.
- Only include parameter blocks for active contingencies. The permitted blocks
  are `substitution: {"min_qty": {"A": int, "B": int, "C": int}}`,
  `payment_deduction: {"min_qty": {"A": int, "B": int, "C": int}}`, and
  `rollover: {"max_deficit": int|null}`.

INPUT CONTRACTS JSON:
<JSON batch inserted here>
\end{lstlisting}

\paragraph{Extraction and retries.}
The translation prompt instructs the model to extract only terms stated in
the contract, expand uniform schedules, and return null rather than infer a
missing or ambiguous value.  The response parser accepts either bare JSON or
JSON enclosed in a Markdown code fence.  It verifies that the top-level
object contains a list of contracts and that the returned identifier set
exactly matches the requested batch.  Malformed output receives corrective
feedback for up to three attempts.  If all attempts fail, or if any identifier
is missing or unexpected, the batch is rejected rather than partially
accepted.

\paragraph{Unsupported clauses.}
The contingency schema is extensible; currently, the translator extracts only
substitution, grim trigger, payment deduction, and rollover.  Text expressing
another remedy remains in the retained raw agreement and audit response, but
it is not added to $\kappa$, enforced by $C^\omega$, credited in clause counts,
or passed to the solver.  If the remaining implemented core terms are
complete, the record can still be solved under the implemented clauses; if an
unsupported clause makes a required price, quantity, or payment ambiguous,
the affected field remains null and the record is parse-incomplete.  This
makes the reported evaluation explicitly an evaluation of the supported
contract language rather than an implicit claim to execute arbitrary legal
prose.

\paragraph{Deterministic normalization.}
After extraction, numeric strings and schedule rows are converted into the
canonical types and week indices above.  Missing schedule rows are inserted
with null values so that incompleteness remains visible.  For an explicitly
stated substitution clause whose minimum is unspecified, the benchmark uses
the current default minimum of one unit of each dish.  Payment-deduction
amounts are derived deterministically from the agreed dish prices rather than
being treated as independently translated terms.  An explicitly stated
rollover clause defaults to a maximum deficit of two units per dish, and its
limit is clamped to the supported range of zero to two.

The RCC solver requires a primary post-breach response even when the contract
does not state one.  In that case preprocessing supplies grim trigger solely
as an implicit solver fallback.  The stored record separates
\emph{explicit contingencies} from \emph{implicit solver contingencies}; all
reported clause-frequency and negotiated-contingency statistics use only the
explicit set.  Thus this fallback does not turn an omitted term into a clause
that the agents negotiated.

\paragraph{Validation and failure handling.}
The normalized contract is checked for all three prices, all five production
weeks, and all six payment weeks.  A contract with any required null or an
invalid type is marked parse-incomplete and is not submitted to the RCC
solver.  A separate deterministic validator checks stated totals and payment
arithmetic, the payment total against the environment-specific customer
budget, minimum quantities, and inventory/resource feasibility.  Arithmetic
disagreements are retained as contract outcomes rather than silently
repaired.  Parse-complete records are converted to solver inputs containing
the environment identifier, production quantities, dish prices, the five
in-game payments and terminal payment, contingency mode and parameters, and
the $0.95$ satisfaction threshold corresponding to the $\epsilon=0.05$
violation tolerance.

\paragraph{Parsing validity of LLM-negotiation contracts.}
All 1,403 proposals in the LLM negotiation corpus produced
parse-complete core terms and solver configurations, including all 159
accepted final agreements; there were no missing required fields or
configuration-construction errors.  This is parsing coverage, not a claim that
every proposal was valid.  The deterministic checks marked 45 proposals as
structurally inconsistent and 160 as over budget, with five proposals in
both groups, leaving 1,203 of 1,403 proposals feasible under the paper's
combined definition.  Among accepted final contracts, all 159 were within budget and
four contained an arithmetic inconsistency.  We retain these failures as
model outcomes rather than repairing or excluding their text at translation
time. Manual verification revealed parsing accuracy to be 100\%: all 159
accepted final agreements were successfully and accurately parsed into complete, valid
solver terms from their natural-language form.

\paragraph{Caching and audit trail.}
Each translation is cached only when its schema version, record identifier,
SHA-256 hash of the raw contract text, provider/model, and structured terms
match.  Raw batch responses and per-record normalized terms are retained for
audit, while hashes of the normalized solver configurations support exact
deduplication.  Consequently, changing either the contract text, translation
model, or schema forces retranslation instead of reusing a stale parse.

\subsection{Synthetic Contract Generation}
\label{app:synthetic-generation}

To generate contracts spanning a range of satisfiability levels, we first
search over base contracts containing the grim-trigger termination rule but no
elective contingencies.  For each target
$\rho\in\{0.9,0.8,0.7,0.6,0.5\}$, we solve
\begin{align}
    \max_{\omega\in\Omega'_{\mathrm{base}}}\quad
        &U_{\mathrm{Cust}}(\omega)+U_{\mathrm{Supp}}(\omega) \\
    \text{subject to}\quad
        &U_i(\omega)\geq u_i(\bot),
          &&i\in\{\mathrm{Cust},\mathrm{Supp}\}, \\
        &\rho-0.05\leq P_{\mathrm{sat}}\!\left(
            \pi^{\mathrm{RCC}},C^\omega
          \right) < \rho+0.05,
\end{align}
with the upper bound closed only at $\rho=0.9$.
Thus the objective maximizes equal-weight expected mutual benefit, subject to
individual rationality and membership in the target $P_{\mathrm{sat}}$ band.
A multi-start coordinate-ascent search varies the three delivery quantities,
three unit prices, and six nominal payments.

We run this construction independently in the four stochastic
environments (Environments~3--6).  Because satisfiability is degenerate in the two deterministic
environments, Environment~1 reuses the written terms generated for
Environment~3 and Environment~2 reuses those generated for Environment~4; in
both cases, the policies and metrics are re-solved under the destination
environment's dynamics.

For every base family, we hold quantities, prices, and payments fixed and
attach each of the eight combinations of dish substitution, payment deduction,
and rollover; the grim-trigger termination rule remains present in every
variant.  This produces $5\times8=40$ candidates per environment.  Across the
240 synthetic contracts, payment deduction and substitution combined with payment deduction rendered 18 contracts mutually unbeneficial (9 each).  We
therefore exclude those two variant classes uniformly, retaining six variants
per family and $6\times5\times6=180$ contracts across the six evaluation
environments.  Because a contingency changes the induced execution problem,
we independently re-solve the RCC policy and recompute satisfiability and
expected utilities for every variant rather than inheriting the base
contract's scores.

\paragraph{Domain-specific natural-language agreements.}
\label{app:domain-contract-materialization}
The optimizer outputs structured prices, delivery quantities, payments, and a
set of active contingencies.  We authored the natural-language base terms and
contingency clauses for the Catering domain.  For the matched Environment~5
cross-domain evaluation, we also authored adapters that render the same terms
and clauses in the vocabulary and monetary scale of the Hotel Cleaning and AI
Model Hosting domains.  For each selected structured record, the renderer
writes the base terms, appends exactly the selected substitution, deduction,
and rollover clauses, and finally appends the always-present grim-trigger
termination clause.  The resulting fully materialized natural-language
agreement is stored once and inserted unchanged into both agents' prompts; no
model translates the contract at performance time.

\label{app:sample-contract}
\paragraph{Example rendered agreement.}
Below we reproduce a benchmark-eligible Catering contract evaluated in
Environment~6, which carries dish substitution, payment deduction, and
rollover in addition to grim-trigger termination. The contract language is
copied verbatim from the evaluated agreement; only its decorative separators
are omitted and its field labels are bolded.

\begin{promptbox}{gray}{Evaluated Catering Environment 6 contract (three elective contingencies)}
\plabel{Parties:} Customer and Caterer.

\par\addvspace{\medskipamount}
\plabel{Term:} Weeks 1 through 11.

\par\addvspace{\medskipamount}
\plabel{Unit prices:} Margherita Pizza \$10; Pesto Pasta \$5; Tomato
Soup \$3.

\par\addvspace{\medskipamount}
\plabel{Delivery obligations:}\\
Week 2: 2 Margherita Pizzas, 2 Pesto Pastas, 2 Tomato Soups.\\
Week 4: 2 Margherita Pizzas, 2 Pesto Pastas, 2 Tomato Soups.\\
Week 6: 2 Margherita Pizzas, 2 Pesto Pastas, 2 Tomato Soups.\\
Week 8: 2 Margherita Pizzas, 2 Pesto Pastas, 2 Tomato Soups.\\
Week 10: 2 Margherita Pizzas, 2 Pesto Pastas, 2 Tomato Soups.\\
A shortfall is the amount by which a production week's delivery falls
below its delivery obligation.

\par\addvspace{\medskipamount}
\plabel{Payments:} Week 1 \$62; Week 3 \$68; Week 5 \$12; Week 7 \$31;
Week 9 \$1; Week 11 \$6.

\par\addvspace{\medskipamount}
\plabel{Substitution:} After delivering the non-substitutable minimum
for each dish (1 Margherita Pizza, 1 Pesto Pasta, 1 Tomato Soup), the
Caterer may substitute quantities delivered above those minimums at
1 Margherita Pizza $=$ 2 Pesto Pastas $=$ 4 Tomato Soups, but only if
the substitution fully satisfies that production week's delivery
obligation. Partial fulfilment via substitution receives no credit;
any resulting shortfall is calculated from actual physical deliveries
and handled under the contract's other applicable clauses.

\par\addvspace{\medskipamount}
\plabel{Rollover:} A shortfall by the Caterer of no more than 2 units
in each dish is added, dish by dish, to the next production week's
delivery obligations to be cured. Carried shortfalls cannot be
substituted. It is a violation if the shortfall in any dish exceeds
2 units, if the Caterer has a shortfall in two consecutive production
weeks, or if a shortfall occurs in the final production week.

\par\addvspace{\medskipamount}
\plabel{Payment deduction:} The Customer may reduce the next scheduled
payment by the unit-price value of that production round's shortfall,
to no less than \$0; any deduction in excess of that payment is
forfeited and does not carry to later payments; the deduction is final
and creates no later repayment obligation. Absent a rollover clause,
the Caterer need not make up the undelivered units. Two consecutive
production weeks with a shortfall by the Caterer is a violation.

\par\addvspace{\medskipamount}
\plabel{Termination:} Except as modified by the clauses above, a
violation by a party entitles the counterparty to cease all further
performance immediately and permanently.
\end{promptbox}

\subsection{Negotiation Metric Definitions}
\label{app:negotiation-metrics}

For each evaluated message history $h\in\mathcal{M}^*$, let
$(\omega_h^{(1)},\ldots,\omega_h^{(n_h)})$ be its ordered sequence of formal
proposals and let $\operatorname{role}_h(k)\in\mathcal{I}'$ identify the
proposer of proposal $k$.  If agreement is reached, write
$\omega_h=\omega_h^{(n_h)}=\alpha(h)$.  Throughout this subsection,
$i\in\mathcal{I}'=\{\mathrm{Cust},\mathrm{Supp}\}$, and
$\varepsilon_{\mathrm{num}}=10^{-9}$ is the numerical comparison tolerance
used by the implementation.  The disagreement payoffs are
$u_{\mathrm{Cust}}(\bot)=B$ and $u_{\mathrm{Supp}}(\bot)=0$.

\paragraph{Expected utility gain and satisfaction.} Reusing the expected contract utility
from Section~\ref{subsubsec:rational-negotiation},
\begin{equation}
    U_i(\omega)
    =\mathbb{E}_{\tau}\!\left[
        u_i(\tau;\theta_i)\mid \omega,\pi^{\mathrm{RCC}}
    \right],
\end{equation}
we report expected gain relative to disagreement:
\begin{equation}
    G_i(\omega)=U_i(\omega)-u_i(\bot).
\end{equation}
The satisfaction probability follows the definition in the main text:
\begin{equation}
    P_{\mathrm{sat}}\!\left(\pi^{\mathrm{RCC}},C^\omega\right)
    =\mathbb{E}_{\tau}\!\left[
        C^\omega(\tau)\mid \omega,\pi^{\mathrm{RCC}}
    \right].
\end{equation}

\paragraph{Mutual benefit.} The accepted contract is mutually beneficial when
both roles weakly prefer it to disagreement, up to the numerical tolerance:
\begin{equation}
    \mathrm{MB}(\omega)
    =\mathbf{1}\!\left[
        G_{\mathrm{Cust}}(\omega)\ge-\varepsilon_{\mathrm{num}}
        \ \wedge\
        G_{\mathrm{Supp}}(\omega)\ge-\varepsilon_{\mathrm{num}}
    \right].
\end{equation}

\paragraph{Proposal feasibility.} Proposal feasibility captures whether an agent's proposed contracts are parse-complete, within budget, and free of arithmetic inconsistencies. Let the indicator
$\mathrm{WF}(\omega_h^{(k)})$ be one exactly when the parser recovers the
complete product price vector, production schedule for every week in
$W_{\mathrm{prod}}$, and payment schedule for every week in
$W_{\mathrm{pay}}$, and finds no arithmetic inconsistency among the total
implied by unit prices and scheduled quantities, the scheduled-payment total,
and any total explicitly reported in the proposal. We define
\begin{equation}
    \mathrm{Feasibility}(\omega_h^{(k)})
    =\mathbf{1}\!\left[
        \mathrm{WF}(\omega_h^{(k)})
        \ \wedge\
        \sum_{w\in W_{\mathrm{pay}}}M_{h,w}^{(k)}\le B
    \right].
\end{equation}
For model $\ell$ in role $i$, let $\mathcal{H}_{\ell,i}$ be its evaluated
negotiation histories and let
$\mathcal{J}_i(h)=\{k:\operatorname{role}_h(k)=i\}$ index the proposals it
makes in history $h$.  We average the feasibility of contract proposals across those histories and report:
\begin{equation}
    \mathrm{PF}_{\ell,i}
    =
    \frac{
        \sum_{h\in\mathcal{H}_{\ell,i}}
        \sum_{k\in\mathcal{J}_i(h)}\mathrm{Feasibility}(\omega_h^{(k)})
    }{
        \sum_{h\in\mathcal{H}_{\ell,i}}|\mathcal{J}_i(h)|
    }.
\end{equation}
Thus every proposal receives equal weight. A negotiation in which the scored
model makes no proposal contributes to neither the numerator nor the
denominator.

\paragraph{Acceptance regret.} In an accepted negotiation, role $i$ incurs
acceptance regret if it passes over an earlier feasible counterparty offer that
gives it strictly greater expected utility than the final agreement, beyond
the numerical tolerance:
\begin{equation}
\begin{aligned}
    \mathrm{AR}_i(h)
    =\mathbf{1}\!\big[\,\exists k<n_h:\;&
        \operatorname{role}_h(k)=-i,
        \ \mathrm{Feas}(\omega_h^{(k)})=1,\\[-2pt]
        &U_i(\omega_h^{(k)})>U_i(\omega_h)+\varepsilon_{\mathrm{num}}\,\big].
\end{aligned}
\end{equation}
The indicator is zero when no such earlier counterparty offer exists. We report
its mean over accepted negotiations in which the model occupies the scored
role, regardless of which role emits the final acceptance.

\paragraph{Contingency count.} Let
$\mathcal{K}_{\mathrm{elec}}=\{\text{payment deduction},\text{rollover},
\text{substitution}\}$ and let $\kappa(\omega)$ be the contingencies in the
accepted contract. The clause count is
\begin{equation}
    \mathrm{CC}(\omega)
    =\sum_{\xi\in\mathcal{K}_{\mathrm{elec}}}
        \mathbf{1}[\xi\in\kappa(\omega)].
\end{equation}
Expected gain, satisfaction, mutual benefit, and contingency count are
averaged over accepted negotiations in which the model occupies the scored
role. Proposal feasibility instead pools all proposals by model and role as
defined above.

\paragraph{Contract completeness.}
\label{app:completeness}

We assess the \emph{completeness} of a contract $\omega$ (in the sense of lacking \emph{unanticipated unsatisfiability}) as the extent to which the RCC
policy can satisfy the constraints $C^\omega$ across all situations reached with positive
probability under the environment's stochastic dynamics. We measure this
using the occupancy probability of an absorbing contract-breach set. This is an on-policy
reachability measure, rather than a measure over states that the
policy never reaches.

\textbf{Setup.} Fix a contract $\omega$ and an environment. Let
$\pi^{\mathrm{RCC}}$ be the joint RCC profile defined in
Section~\ref{subsec:rational-contracting}. The Customer follows
the contracted payment schedule and the Supplier component chooses the
production-round order:
\begin{equation}
    \mathbf{o}_t
    =
    \pi^{\mathrm{RC}}_{\mathrm{Supp}}
    \!\left(t,\mathbf{n}_t,\beta_t,\mu_t,\mathbf{p}'_t\right).
\end{equation}
Here $\mathbf{n}_t=(n_{t,x})_{x\in\mathcal X}$ is the Supplier's ingredient
inventory, $\beta_t$ is its liquid balance, $\mu_t$ is the contingency-memory
state, and $\mathbf{p}'_t=(p_{t,x})_{x\in\mathcal X}$ is the currently observed
input-price scenario. The memory is trivial under grim trigger or substitution
alone, is a binary prior-shortfall flag under payment deduction, and is an
encoded vector of carried product deficits under rollover. Thus the policy is
indexed by production round, inventory, Supplier balance, contingency memory,
and current prices. Prices are redrawn each round and therefore are not part of
the carried state.

The policy is solved for under the reliability
requirement
$P_{\mathrm{sat}}(\pi^{\mathrm{RCC}},C^\omega)\geq1-\epsilon$, with
$1-\epsilon=0.95$. The occupancy calculation
applies no additional satisfaction threshold, but it remains conditional on
this threshold-selected policy and can therefore change if the satisfaction
threshold changes.

\textbf{Forward occupancy push.} Let $\rho_0$ place unit mass on the initial state $(\mathbf{0},\beta_0,\mu_0)$, and augment the live state space with
an absorbing state $\mathsf{DEAD}$. In production round $t$, the current price
scenario is drawn at its true probability and the policy chooses
$\mathbf{o}_t$. The scheduled payment is added and the realized order cost is
debited. Probability mass for which the resulting Supplier balance is negative is routed to
$\mathsf{DEAD}$. We then spread probability mass over realized input receipts, apply the
contract's production and contingency rules, and route any uncured production
shortfall to $\mathsf{DEAD}$.

For payment deduction, the raw deduction is first capped by the
amount of the next scheduled payment. The capped amount is then netted against the scheduled payment before it is
credited to the current post-order Supplier balance, so this step alone never
routes mass to $\mathsf{DEAD}$. After the final production round, the final
payment amount supplies only this deduction cap and is not otherwise credited
to the carried balance. Surviving inventory is finally spread through the true
spoilage distribution to form the step$-t$ occupancy measure $\rho_t$.

\textbf{Metric.} Let $S_t = \sum_{s\ \mathrm{live}}\rho_t(s) = 1-\rho_t(\mathsf{DEAD})$
be the probability mass surviving production round $t$ (i.e. the total probability that no contract violation will occur by step $t$ under the RCC policy). We define
\begin{equation}
    \mathrm{Completeness}(\omega)
    =1-\frac{1}{|W_\text{prod}|}\sum_{t=1}^{|W_\text{prod}|}\bigl(1-S_t\bigr)
    \in[0,1].
\end{equation}
This is one minus the expected fraction of the episode spent in breach of the contract. Since $\mathsf{DEAD}$ is absorbing, an earlier breach is charged in
every subsequent round, while survival through all production rounds
results in a score of $1$.

\subsection{Performance Metric Definitions}
\label{app:performance-metrics}

Fix a role $i\in\mathcal{I}'=\{\mathrm{Cust},\mathrm{Supp}\}$, and let
$j\in\mathcal{R}$ index a valid performance run with realized trajectory
$\tau_j$.  Let
$\mathcal{W}_{\mathrm{Cust}}=W_{\mathrm{pay}}$ and
$\mathcal{W}_{\mathrm{Supp}}=W_{\mathrm{prod}}$. For every week $w\in\mathcal{W}_i$, let
$c_{i,j,w}\in\{0,1\}$ indicate that role $i$'s constraint is fulfilled in run $j$ after
applying the substitution, payment-deduction, and
rollover contingencies. The indicator is local to each week, i.e., a
violation in one week does not force later indicators to zero. We exclude from the denominators any week with zero effective obligation, whether the zero arises from the contract schedule itself or from a contingency; in that case we set $c_{i,j,w}=1$ but do not count the week. Let
$\mathcal{W}^{\mathrm{eval}}_{i,j}\subseteq\mathcal{W}_i$ denote the remaining weeks after excluding this set.

\paragraph{Utility.} The \emph{realized utility} of a trajectory $\tau_j$ is just $u_i(\tau_j;\theta_i)$.
For compliant-utility accounting, our implementation decomposes realized
utility into per-week gains:
\begin{align}
    \Delta u_{\mathrm{Cust},j,w}
    &=\sum_{d\in D}v_d q'_{j,w+1,d}-M^{\mathrm{paid}}_{j,w},
    &&w\in W_{\mathrm{pay}},\\
    \Delta u_{\mathrm{Supp},j,w}
    &=M^{\mathrm{paid}}_{j,w-1}
      +\mathbf{1}[w=L-1]M^{\mathrm{paid}}_{j,L}
      -\sum_{x\in\mathcal{X}}o_{j,w,x}p_{j,w,x},
    &&w\in W_{\mathrm{prod}},
\end{align}
where $q'_{j,12,d}=0$. Since there is one more payment week than production week, the Supplier's final production
week groups the payments from the final week with the prior week. Realized utility can then be written as:
\begin{equation}
    u_i(\tau_j;\theta_i)
    =u_i(\bot)+\sum_{w\in\mathcal{W}_i}\Delta u_{i,j,w}.
\end{equation}
\emph{Compliant utility} excludes the positive gains that can be attributed to weeks where a contract constraint is violated, including payments the Customer withholds unilaterally:
\begin{equation}
    u_i^{\mathrm{c}}(\tau_j)
    =u_i(\bot)+\sum_{w\in\mathcal{W}_i}\!\left[
        c_{i,j,w}\Delta u_{i,j,w}
        +(1-c_{i,j,w})\min(\Delta u_{i,j,w}-\ell_{i,j,w},0)
    \right],
\end{equation}
where $\ell_{\mathrm{Supp},j,w}=0$ and
$\ell_{\mathrm{Cust},j,w}=\mathbf{1}[w\leq f_{\mathrm{Supp},j}]\,(M^{\mathrm{due}}_{j,w}-M^{\mathrm{paid}}_{j,w})^{+}$
is the payment withheld before any Supplier violation, with $M^{\mathrm{due}}_{j,w}$ the scheduled payment net of any deduction and $f_{\mathrm{Supp},j}$ the Supplier's first violation week (defined below).

\paragraph{Regret.} Regret metrics measure utility of some agent relative to what is obtained by the RCC baseline. Let $\tau_j^{\mathrm{RCC}}$ be the counterfactual trajectory obtained by
replacing the agent's policy in role $i$ with RCC while holding fixed the
contract, environment configuration, counterparty policy, and RNG seed. This
is a separate simulation: because a changed policy can change RNG consumption,
$\tau_j^{\mathrm{RCC}}$ need not share the realized stochastic history of
$\tau_j$. \emph{Ex-post regret} and \emph{compliant regret} are, respectively,
\begin{align}
    \mathrm{Regret}_{i}(\tau_j)
        &=u_i(\tau_j^{\mathrm{RCC}};\theta_i)
          -u_i(\tau_j;\theta_i),\\
    \mathrm{Regret}^{\mathrm{c}}_{i}(\tau_j)
        &=u_i^{\mathrm{c}}(\tau_j^{\mathrm{RCC}})
          -u_i^{\mathrm{c}}(\tau_j).
\end{align}
The compliant variant applies compliant utility to both the RCC reference and
the evaluated agent.

\paragraph{Defection.} Define the focal agent's first violation
week by
\begin{equation}
    f_{i,j}
    =\min\{w\in\mathcal{W}^{\mathrm{eval}}_{i,j}:c_{i,j,w}=0\},
\end{equation}
with $f_{i,j}=\infty$ if $i$ never violates. Recalling the violation indicator $D_{i,t}$, define its run-indexed, pre-action
counterpart as:
\begin{equation}
    D_{i,j,w}=\mathbf{1}[f_{-i,j}<w],
\end{equation}
which records whether role $i$ has observed a counterparty violation in a
strictly earlier week. The focal agent's pre- and post-violation opportunity
sets are:
\begin{align}
    \mathcal{W}^{\mathrm{pre}}_{i,j}
        &=\{w\in\mathcal{W}^{\mathrm{eval}}_{i,j}:D_{i,j,w}=0\},\\
    \mathcal{W}^{\mathrm{post}}_{i,j}
        &=\{w\in\mathcal{W}^{\mathrm{eval}}_{i,j}:D_{i,j,w}=1\}.
\end{align}
Thus a same-week violation is not a response to an earlier violation. Let:
\begin{align}
    e^{\mathrm{pre}}_{i,j}
        &=\mathbf{1}[\mathcal{W}^{\mathrm{pre}}_{i,j}\neq\emptyset],\\
    e^{\mathrm{resp}}_{i,j}
        &=\mathbf{1}[f_{-i,j}<f_{i,j}
          \ \wedge\ \mathcal{W}^{\mathrm{post}}_{i,j}\neq\emptyset].
\end{align}
The second indicator requires the counterparty to be the strict first defector
and the focal agent to have a later opportunity. Define the run-level indicators:
\begin{align}
    z^{\mathrm{any}}_{i,j}
        &=\mathbf{1}[f_{i,j}<\infty],\\
    z^{\mathrm{uni}}_{i,j}
        &=\mathbf{1}[\exists w\in\mathcal{W}^{\mathrm{pre}}_{i,j}:
                     c_{i,j,w}=0],\\
    z^{\mathrm{rec}}_{i,j}
        &=\mathbf{1}[e^{\mathrm{resp}}_{i,j}=1\ \wedge\
          \exists w\in\mathcal{W}^{\mathrm{post}}_{i,j}:c_{i,j,w}=0].
\end{align}
Our defection metrics are then defined as:
\begin{align}
    \mathrm{Defection}_i
        &=\frac{\sum_{j\in\mathcal{R}}z^{\mathrm{any}}_{i,j}}
                {|\mathcal{R}|},\\
    \mathrm{Unilateral}_i
        &=\frac{\sum_{j\in\mathcal{R}}z^{\mathrm{uni}}_{i,j}}
                {\sum_{j\in\mathcal{R}}e^{\mathrm{pre}}_{i,j}},\\
    \mathrm{Reciprocal}_i
        &=\frac{\sum_{j\in\mathcal{R}}z^{\mathrm{rec}}_{i,j}}
                {\sum_{j\in\mathcal{R}}e^{\mathrm{resp}}_{i,j}}.
\end{align}
In words, the raw \emph{defection} metric $\mathrm{Defection}_i$ measures the rate at which the agent defects across all runs, the \emph{unilateral defection} metric $\mathrm{Unilateral}_i$ measures the defection rate across runs where the counterparty did not defect first, and the \emph{reciprocal defection} metric $\mathrm{Reciprocal}_i$ measures the defection rate across runs where the counterparty defected first.

\paragraph{Compliance: Macro-averaging.} All aggregate metrics reported in the main text use
macro-averaging across performance runs. For a given focal-agent and counterparty pairing,
each $j\in\mathcal{R}$ is one of the environment-contract runs. Utility and regret already yield one scalar per case and are
averaged directly across cases; the three defection estimands above analogously
average one binary indicator per eligible run.

For the four week-level
compliance metrics, we first calculate a rate within each run and then average
those rates across eligible runs. Let
$\mathcal{R}_i=\{j\in\mathcal{R}:
|\mathcal{W}^{\mathrm{eval}}_{i,j}|>0\}$ and define:
\begin{align}
    \bar c_{i,j}
        &=\frac{\sum_{w\in\mathcal{W}^{\mathrm{eval}}_{i,j}}c_{i,j,w}}
                {|\mathcal{W}^{\mathrm{eval}}_{i,j}|},\\
    \bar c^{\mathrm{pre}}_{i,j}
        &=\frac{\sum_{w\in\mathcal{W}^{\mathrm{pre}}_{i,j}}c_{i,j,w}}
                {|\mathcal{W}^{\mathrm{pre}}_{i,j}|},\\
    \bar c^{\mathrm{post}}_{i,j}
        &=\frac{\sum_{w\in\mathcal{W}^{\mathrm{post}}_{i,j}}c_{i,j,w}}
                {|\mathcal{W}^{\mathrm{post}}_{i,j}|},\\
    \bar m_{i,j}
        &=\frac{1}{|\mathcal{W}^{\mathrm{eval}}_{i,j}|}
          \sum_{w\in\mathcal{W}^{\mathrm{eval}}_{i,j}}
          \mathbf{1}[1-c_{i,j,w}=D_{i,j,w}],
\end{align}
where a pre- or post-rate is defined only when its denominator is nonzero. The macro-averaged estimands are:
\begin{align}
    \mathrm{Compliance}^{\mathrm{macro}}_i
        &=\frac{1}{|\mathcal{R}_i|}
          \sum_{j\in\mathcal{R}_i}\bar c_{i,j},\\
    \mathrm{Conditional}^{\mathrm{macro}}_i
        &=\frac{\sum_{j\in\mathcal{R}}e^{\mathrm{pre}}_{i,j}
                    \bar c^{\mathrm{pre}}_{i,j}}
                {\sum_{j\in\mathcal{R}}e^{\mathrm{pre}}_{i,j}},\\
    \mathrm{Exploited}^{\mathrm{macro}}_i
        &=\frac{\sum_{j\in\mathcal{R}}e^{\mathrm{resp}}_{i,j}
                    \bar c^{\mathrm{post}}_{i,j}}
                {\sum_{j\in\mathcal{R}}e^{\mathrm{resp}}_{i,j}},\\
    \mathrm{TFT}^{\mathrm{macro}}_i
        &=\frac{1}{|\mathcal{R}_i|}
          \sum_{j\in\mathcal{R}_i}\bar m_{i,j}.
\end{align}
In words, the raw \emph{compliance} metric $\mathrm{Compliance}_i$ measures the fraction of weeks in a run where the agent complies with the contract, the \emph{conditional compliance} metric $\mathrm{Conditional}_i$ measures the fraction of weeks where the agent complies out of weeks where the counterparty is compliant, the \emph{exploited compliance} metric $\mathrm{Exploited}_i$ measures the fraction of weeks where the agent complies out of weeks where the counterparty has defected, and the \emph{tit-for-tat} metric $\mathrm{TFT}_i$ measures the fraction of weeks where the agent complies or defects in response to compliance or defection respectively.

Since we average across runs, each eligible run receives equal weight even when
runs contain different numbers of eligible weeks. These are the estimands reported in the main text and in
Appendix~\ref{app:performance-re-counterparties}; the superscript
``macro'' is omitted in the table headings.

\paragraph{Compliance: Pooled-week micro-averaging.}
\label{app:performance-micro-aggregation}

As a sensitivity analysis, we additionally micro-average the same week-level
indicators, giving every eligible week equal weight rather than giving every
eligible performance run equal weight:
\begin{align}
    \mathrm{Compliance}^{\mathrm{micro}}_i
        &=\frac{\sum_{j\in\mathcal{R}_i}
                    \sum_{w\in\mathcal{W}^{\mathrm{eval}}_{i,j}}c_{i,j,w}}
                {\sum_{j\in\mathcal{R}_i}
                    |\mathcal{W}^{\mathrm{eval}}_{i,j}|},\\
    \mathrm{Conditional}^{\mathrm{micro}}_i
        &=\frac{\sum_{j\in\mathcal{R}}e^{\mathrm{pre}}_{i,j}
                    \sum_{w\in\mathcal{W}^{\mathrm{pre}}_{i,j}}c_{i,j,w}}
                {\sum_{j\in\mathcal{R}}e^{\mathrm{pre}}_{i,j}
                    |\mathcal{W}^{\mathrm{pre}}_{i,j}|},\\
    \mathrm{Exploited}^{\mathrm{micro}}_i
        &=\frac{\sum_{j\in\mathcal{R}}e^{\mathrm{resp}}_{i,j}
                    \sum_{w\in\mathcal{W}^{\mathrm{post}}_{i,j}}c_{i,j,w}}
                {\sum_{j\in\mathcal{R}}e^{\mathrm{resp}}_{i,j}
                    |\mathcal{W}^{\mathrm{post}}_{i,j}|},\\
    \mathrm{TFT}^{\mathrm{micro}}_i
        &=\frac{\sum_{j\in\mathcal{R}_i}
                    \sum_{w\in\mathcal{W}^{\mathrm{eval}}_{i,j}}
                    \mathbf{1}[1-c_{i,j,w}=D_{i,j,w}]}
                {\sum_{j\in\mathcal{R}_i}
                    |\mathcal{W}^{\mathrm{eval}}_{i,j}|}.
\end{align}
Figure~\ref{fig:stochastic-behavior} uses run-level unilateral defection. An opportunity-conditioned rate is
undefined when its denominator is zero. Appendix~\ref{app:micro-compliance-tables}
reports the micro-averaged results for all four compliance metrics. Utility,
regret, and the three run-level defection metrics remain case-level quantities
and are unchanged by this sensitivity analysis.

\clearpage
\section{Additional Results}

\subsection{Contract Negotiation Pareto Frontiers}
\label{app:negotiation-pareto-frontiers}

Figures~\ref{fig:pareto-deterministic}, \ref{fig:pareto-moderate},
and~\ref{fig:pareto-high} complement the Environment~5 comparison in
Figure~\ref{fig:neg-frontier} by reporting the other five environments.  Each
black diamond is the ex-ante valuation of an LLM pairing, averaged over its
three negotiated contracts.  The frontiers use the same utility coordinates
and reliability requirement as the main figure: $F_1$ contains
non-contingent contracts, $F_2$ permits the supported contingencies, and
$F_3$ is the frontier of contracts reached by Concessionary negotiation
(Eq.~\ref{eq:concession-target}) across a sweep of outside options, all with
$P_{\mathrm{sat}}\geq0.95$.  In the environments without stochasticity, exhaustive
enumeration gives the exact full-performance frontier and $F_1=F_2$.  In the
stochastic environments, the open circles are sampled nondominated contracts;
segments connecting them are visual guides and do not assert that every
intermediate utility pair is feasible.

The no-stochasticity results show that several negotiated outcomes lie on or
close to the exact frontier.  This pattern persists in the low-stochasticity
environments, where multiple pairings reach or closely approach either $F_1$
or $F_2$.  Environment~6 is more favorable than the low-value,
high-stochasticity Environment~5 shown in the main text, although its outcomes
remain more dispersed than in the settings without stochasticity.

\begin{figure}[h]
\centering
\begin{subfigure}[t]{0.48\textwidth}
    \centering
    \includegraphics[width=\textwidth]{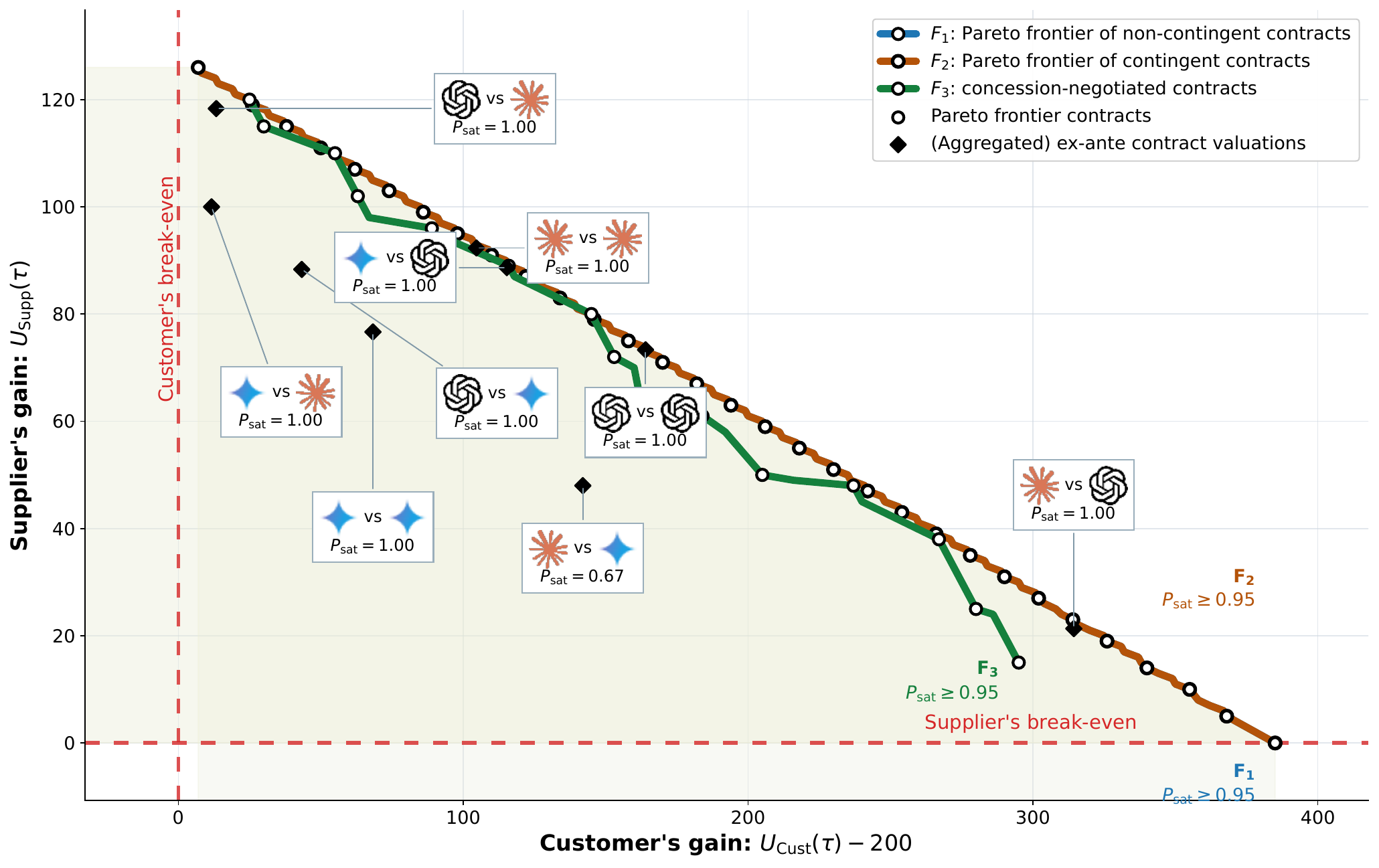}
    \caption{\textbf{Environment 1: no stochasticity, low value/capital.}}
    \label{fig:pareto-deterministic-a}
\end{subfigure}
\hfill
\begin{subfigure}[t]{0.48\textwidth}
    \centering
    \includegraphics[width=\textwidth]{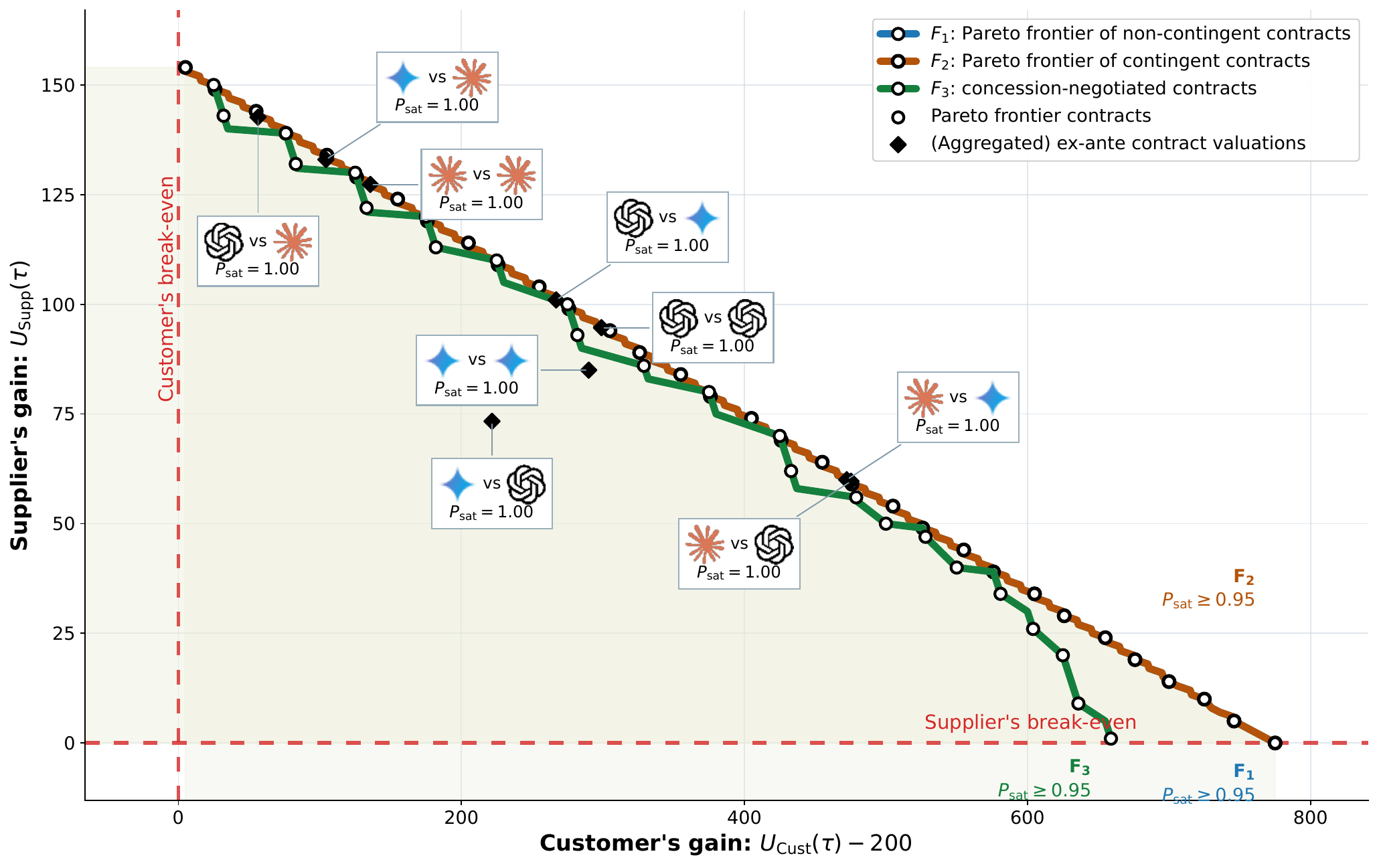}
    \caption{\textbf{Environment 2: no stochasticity, high value/capital.}}
    \label{fig:pareto-deterministic-b}
\end{subfigure}
\caption{\textbf{Negotiated contracts and exact Pareto frontiers in the
deterministic environments.} Black diamonds show three-run matchup means and
open circles show exhaustively enumerated frontier contracts.  Because
performance is deterministic, every frontier contract has
$P_{\mathrm{sat}}=1$ and the non-contingent and contingent frontiers coincide.
$F_3$ is the frontier of concession-negotiated contracts.}
\Description{Two side-by-side scatter plots, for Environments 1 and 2 (no stochasticity), of the Supplier's gain against the Customer's gain, showing negotiated contracts as black diamonds alongside the exact Pareto frontier contracts as open circles.}
\label{fig:pareto-deterministic}
\end{figure}

\begin{figure}[h]
\centering
\begin{subfigure}[t]{0.48\textwidth}
    \centering
    \includegraphics[width=\textwidth]{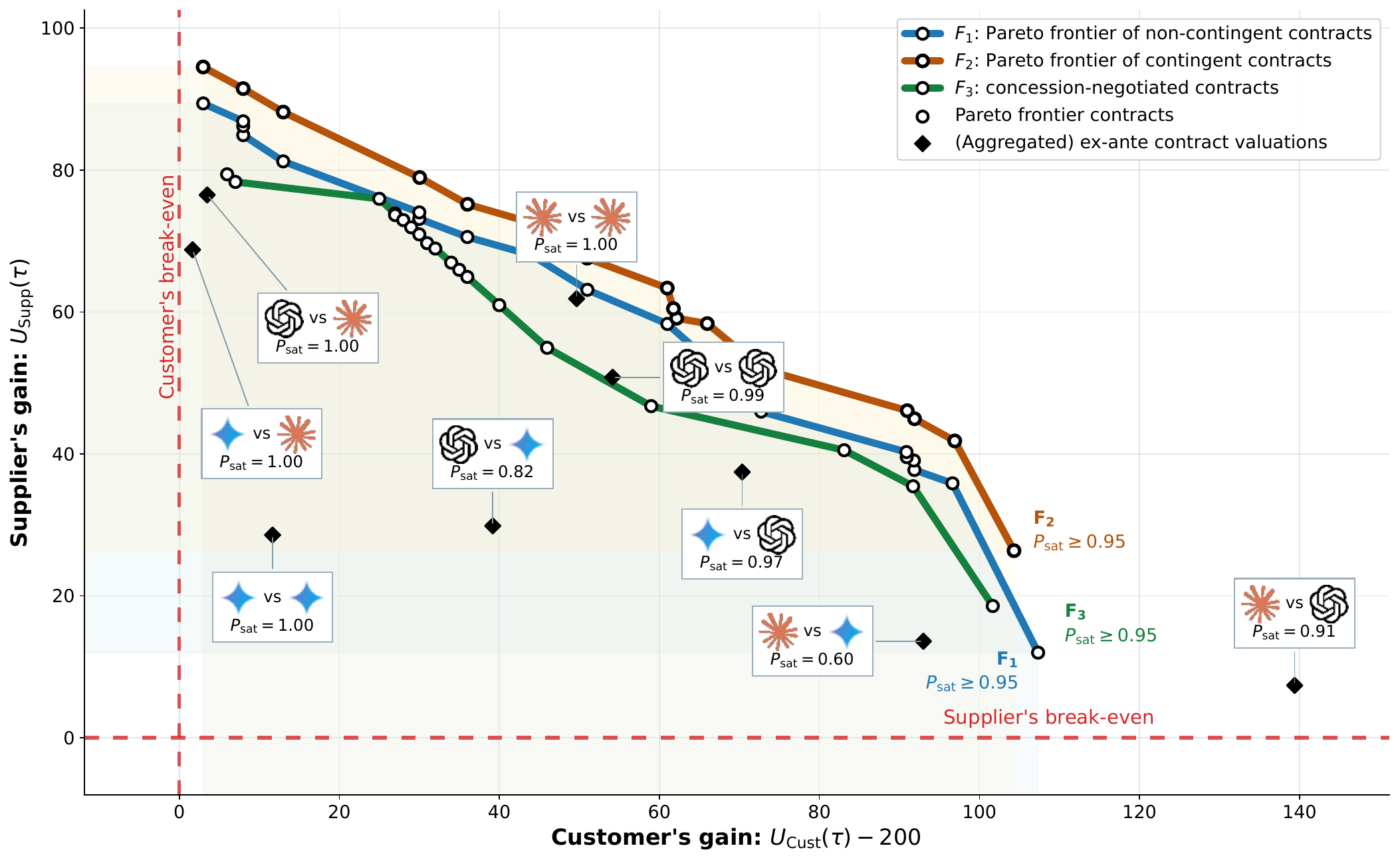}
    \caption{\textbf{Environment 3: low stochasticity, low value/capital.}}
    \label{fig:pareto-moderate-a}
\end{subfigure}
\hfill
\begin{subfigure}[t]{0.48\textwidth}
    \centering
    \includegraphics[width=\textwidth]{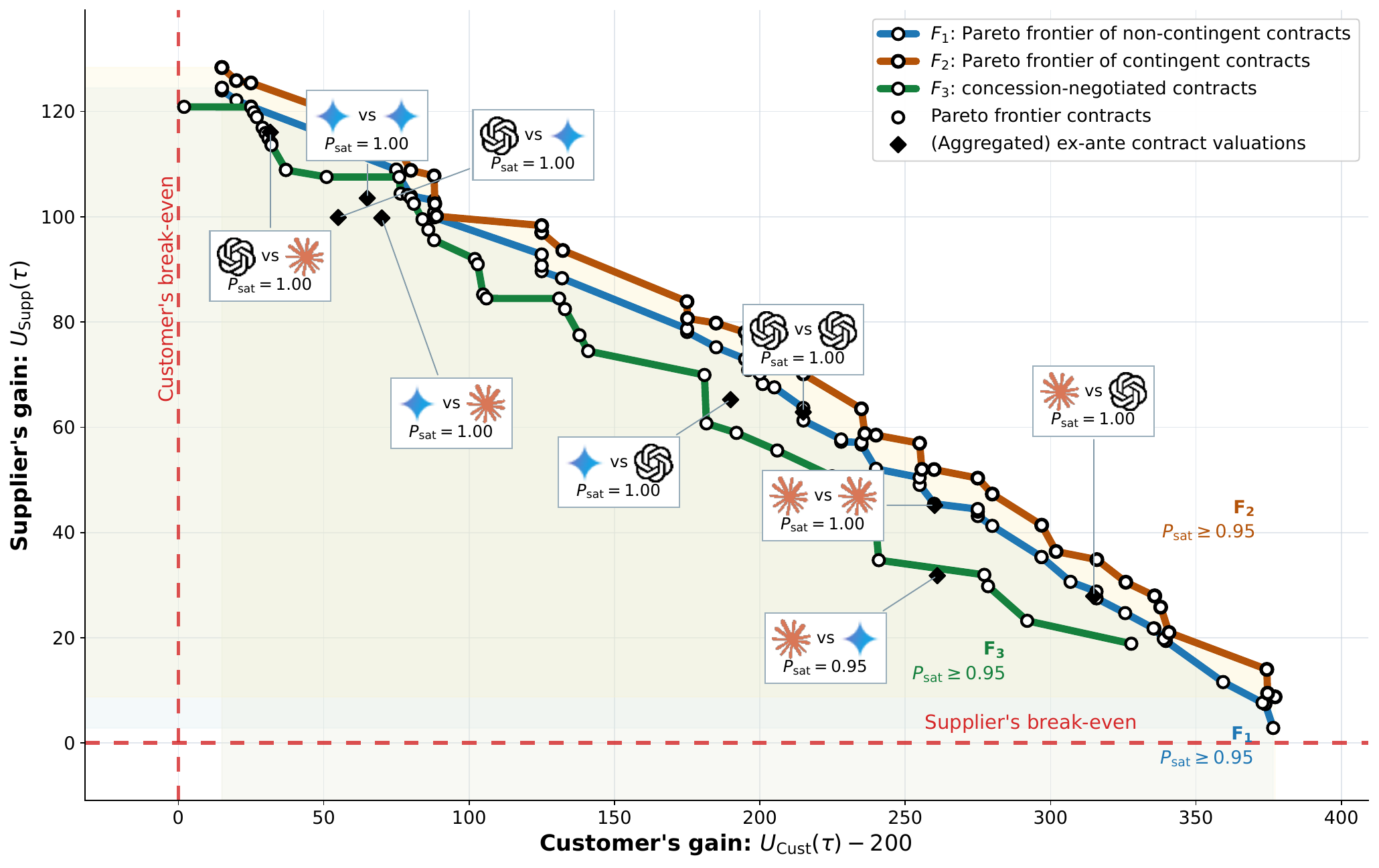}
    \caption{\textbf{Environment 4: low stochasticity, high value/capital.}}
    \label{fig:pareto-moderate-b}
\end{subfigure}
\caption{\textbf{Negotiated contracts and Pareto frontiers in the low
stochasticity environments.} Black diamonds show three-run matchup means;
$F_1$ is the estimated Pareto frontier for non-contingent contracts and
$F_2$ is its contingency-enabled counterpart; $F_3$ is the frontier of
concession-negotiated contracts.}
\Description{Two side-by-side scatter plots, for Environments 3 and 4 (low stochasticity), of the Supplier's gain against the Customer's gain, showing negotiated contracts as black diamonds alongside the estimated non-contingent, contingency-enabled, and concession-negotiated Pareto frontiers.}
\label{fig:pareto-moderate}
\end{figure}

\begin{figure}[h]
\centering
\begin{subfigure}[t]{0.48\textwidth}
    \centering
    \includegraphics[width=\textwidth]{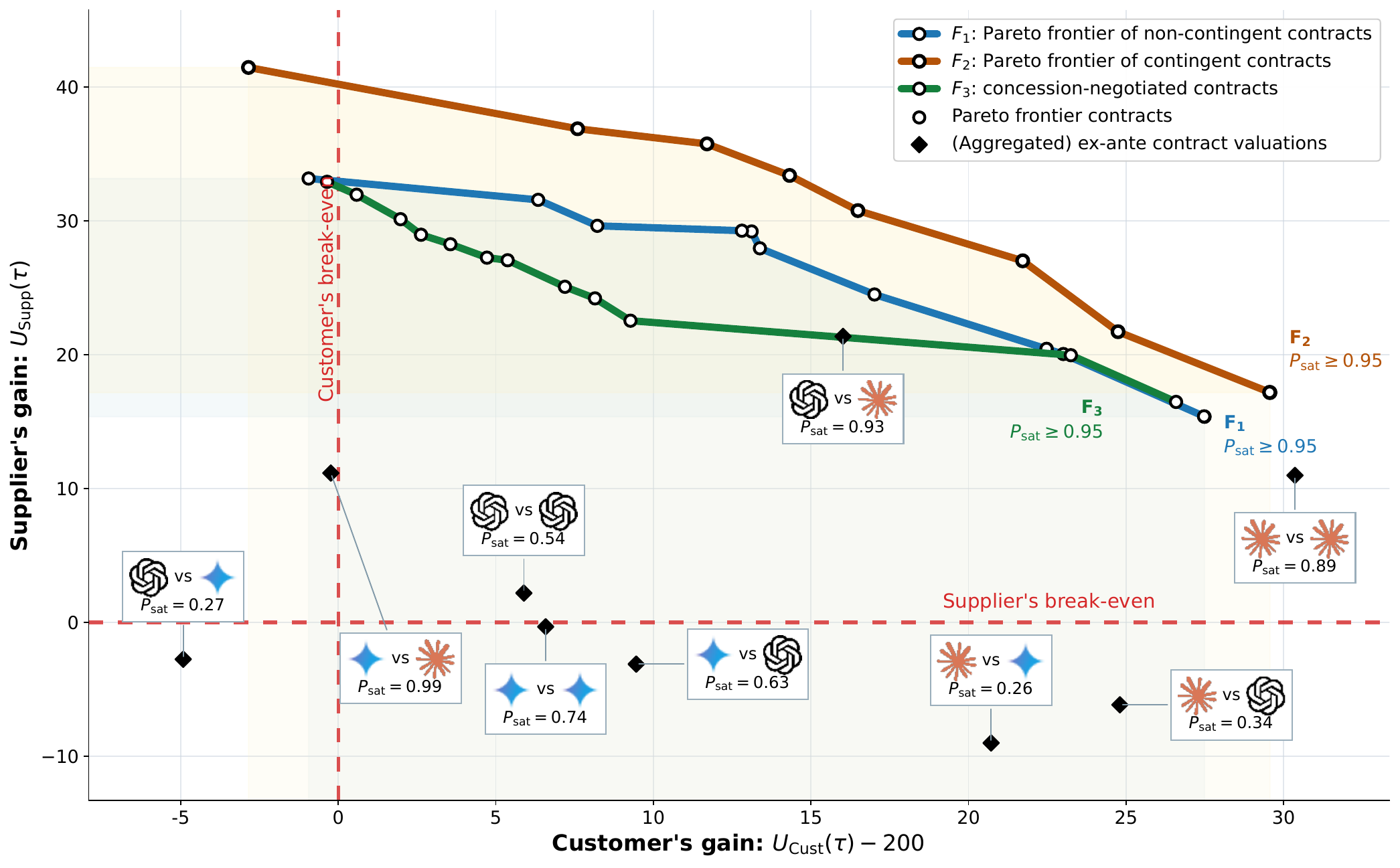}
    \caption{\textbf{Environment 5: high stochasticity, low value/capital.}}
    \label{fig:pareto-high-a}
\end{subfigure}
\hfill
\begin{subfigure}[t]{0.48\textwidth}
    \centering
    \includegraphics[width=\textwidth]{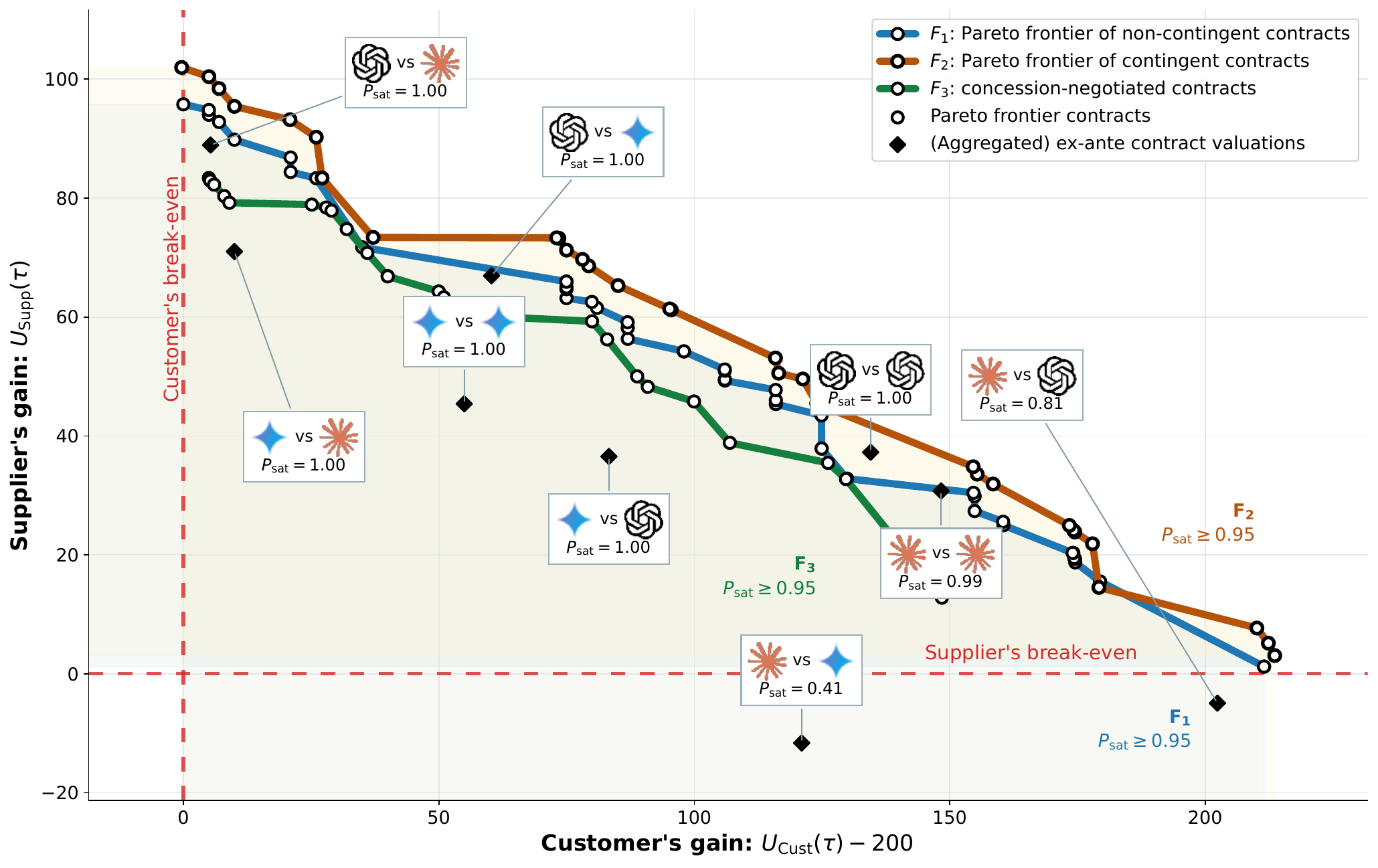}
    \caption{\textbf{Environment 6: high stochasticity, high value/capital.}}
    \label{fig:pareto-high-b}
\end{subfigure}
\caption{\textbf{Negotiated contracts and Pareto frontiers in the high
stochasticity environments.}
Sub-figure (a) reproduces Figure~\ref{fig:neg-frontier} in the main text.}
\Description{Two side-by-side scatter plots, for Environments 5 and 6 (high stochasticity), of the Supplier's gain against the Customer's gain, showing negotiated contracts as black diamonds alongside the estimated non-contingent, contingency-enabled, and concession-negotiated Pareto frontiers.}
\label{fig:pareto-high}
\end{figure}

\FloatBarrier

\clearpage
\subsection{Negotiation Quality by Environment}
\label{app:negotiation-quality-by-environment}

Figures~\ref{fig:negotiation-quality-customer-by-environment}
and~\ref{fig:negotiation-quality-supplier-by-environment} disaggregate the
role-conditioned results in Figure~\ref{fig:negotiation-quality-roles} of the
main paper by benchmark environment, using the same aggregation and metric
definitions as the main-text figure. As evidenced by the plots, environments can substantially influence certain negotiation metrics.

\begin{figure}[htbp]
\centering
\includegraphics[width=\textwidth,height=0.7\textheight,keepaspectratio]{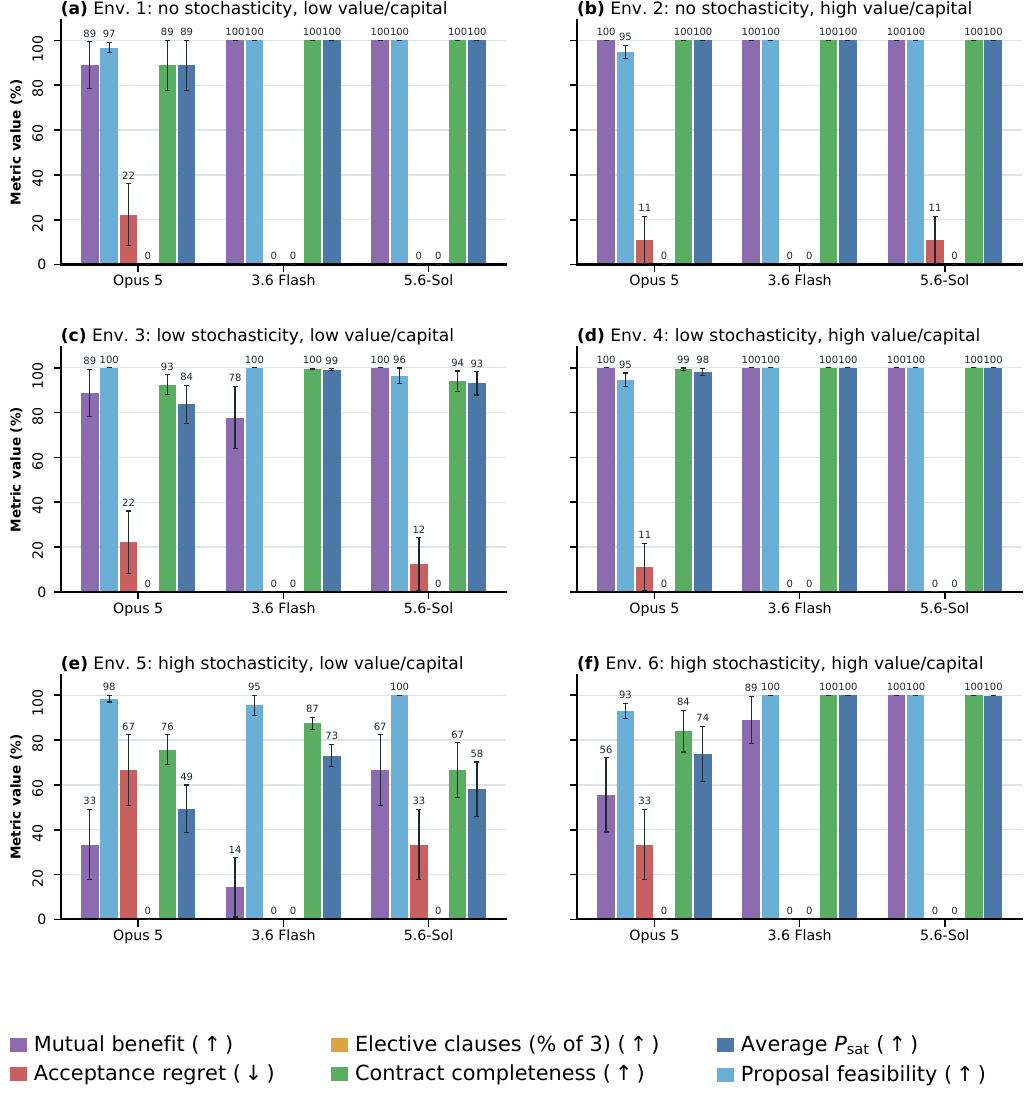}
\caption{\textbf{Role-conditioned negotiation quality by environment:
Customer.} Negotiation metrics for Customer models in (a) no stochasticity,
low value/capital; (b) no stochasticity, high value/capital; (c) low
stochasticity, low value/capital; (d) low stochasticity, high value/capital;
(e) high stochasticity, low value/capital; and (f) high stochasticity, high
value/capital. Error bars show standard errors.}
\Description{Six panels of grouped bar charts, one per environment, showing negotiation quality metrics for each LLM acting as the Customer, with standard-error bars.}
\label{fig:negotiation-quality-customer-by-environment}
\end{figure}

\clearpage
\begin{figure}[p]
\centering
\includegraphics[width=\textwidth,height=0.7\textheight,keepaspectratio]{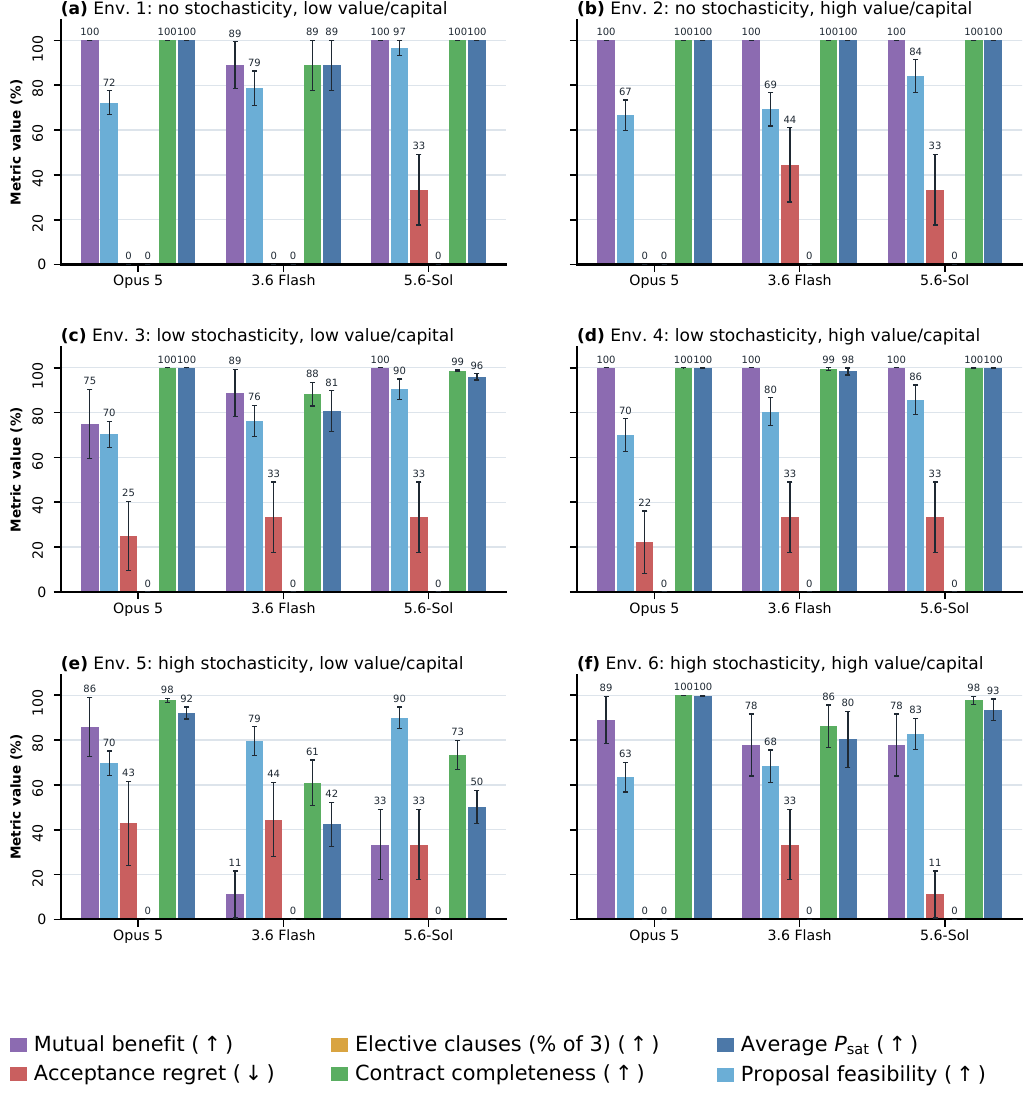}
\caption{\textbf{Role-conditioned negotiation quality by environment:
Supplier.} Negotiation metrics for Supplier models in (a) no stochasticity,
low value/capital; (b) no stochasticity, high value/capital; (c) low
stochasticity, low value/capital; (d) low stochasticity, high value/capital;
(e) high stochasticity, low value/capital; and (f) high stochasticity, high
value/capital. Error bars show standard errors.}
\Description{Six panels of grouped bar charts, one per environment, showing negotiation quality metrics for each LLM acting as the Supplier, with standard-error bars.}
\label{fig:negotiation-quality-supplier-by-environment}
\end{figure}

\FloatBarrier

\clearpage
\subsection{Contract Performance Tables}
\label{app:performance-re-counterparties}

Tables~\ref{tab:customer-performance-rational-models-appendix}
and~\ref{tab:caterer-performance-rational-models-appendix} provide the regenerated performance
tables shown in the main text to all three counterparty policies (RCC, RC, and RE) and all rational
controls. As in the main text, compliance is calculated within
each environment--contract case and then macro-averaged with equal case
weights.

\begingroup
\def\meanstd#1#2{#1{\scriptsize\,$\pm$#2}}
\def\counterpartyrcc{\shortstack{\emph{Rational}\\\emph{Conditional}\\\emph{Complier}\\(RCC)}}
\def\counterpartyrc{\shortstack{\emph{Rational}\\\emph{Complier}\\(RC)}}
\def\counterpartyrnc{\shortstack{\emph{Rational}\\\emph{Exploiter}\\(RE)}}
\begin{table*}[h]
\footnotesize
\caption{\textbf{Customer-side performance metrics against RCC, RC, and RE Supplier counterparties.} Entries are equal-run-weight means $\pm$ standard errors across 180 environment--contract cases; each compliance rate is computed within a case before averaging across cases. Dashes denote undefined rates.}
\label{tab:customer-performance-rational-models-appendix}
\centering
\setlength{\tabcolsep}{2.5pt}
\begin{tabular}{@{}ccccccccccccc@{}}
\toprule
\multirow{2}{*}[-2pt]{\shortstack{\textbf{Counterparty}\\(\emph{Supplier})}} & \multirow{2}{*}[-2pt]{\shortstack{\textbf{Agent}\\(\emph{Customer})}} & \multicolumn{2}{c}{\textbf{Utility}} & \multicolumn{2}{c}{\textbf{(Ex-Post) Regret}} & \multicolumn{4}{c}{\textbf{Compliance}} & \multicolumn{3}{c}{\textbf{Defection}} \\
\cmidrule(lr){3-4}\cmidrule(lr){5-6}\cmidrule(lr){7-10}\cmidrule(lr){11-13}
 & & \textbf{Utility}\,$\uparrow$ & \textbf{Compl.}\,$\uparrow$ & \textbf{Regret}\,$\to 0$ & \textbf{Compl.}\,$\to 0$ & \textbf{Rate} & \textbf{Cond.}\,$\uparrow$ & \textbf{Expl.}\,$\downarrow$ & \textbf{TFT}\,$\uparrow$ & \textbf{Rate} & \textbf{Unil.}\,$\downarrow$ & \textbf{Recip.}\,$\uparrow$ \\
\midrule
\multirow{6}{*}{\counterpartyrcc} & Claude Opus 5 & \meanstd{347.1}{7.7} & \meanstd{338.6}{7.6} & \meanstd{-3.3}{1.6} & \meanstd{5.2}{1.6} & \meanstd{93.0}{1.0} & \meanstd{94.4}{0.7} & \meanstd{\textbf{0.0}}{0.0} & \meanstd{94.7}{0.7} & \meanstd{30.0}{3.4} & \meanstd{27.2}{3.3} & \meanstd{\textbf{100.0}}{0.0} \\
 & Gemini 3.6 Flash & \meanstd{355.8}{7.8} & \meanstd{342.1}{7.4} & \meanstd{-12.1}{1.7} & \meanstd{1.6}{1.0} & \meanstd{91.4}{0.7} & \meanstd{91.9}{0.7} & \meanstd{\textbf{0.0}}{0.0} & \meanstd{91.9}{0.7} & \meanstd{47.2}{3.7} & \meanstd{44.4}{3.7} & \meanstd{\textbf{100.0}}{0.0} \\
 & GPT 5.6 Sol & \meanstd{\textbf{366.1}}{7.9} & \meanstd{335.9}{7.6} & \meanstd{-22.3}{2.3} & \meanstd{7.8}{2.4} & \meanstd{80.6}{0.9} & \meanstd{83.1}{0.5} & \meanstd{\textbf{0.0}}{0.0} & \meanstd{84.0}{0.4} & \meanstd{92.8}{1.9} & \meanstd{90.0}{2.2} & \meanstd{\textbf{100.0}}{0.0} \\
 & RC & \meanstd{343.0}{7.4} & \meanstd{343.0}{7.4} & \meanstd{0.8}{0.4} & \meanstd{0.8}{0.4} & \meanstd{\textbf{100.0}}{0.0} & \meanstd{\textbf{100.0}}{0.0} & \meanstd{100.0}{0.0} & \meanstd{99.4}{0.2} & \meanstd{\textbf{0.0}}{0.0} & \meanstd{\textbf{0.0}}{0.0} & \meanstd{0.0}{0.0} \\
 & RCC & \meanstd{343.8}{7.3} & \meanstd{\textbf{343.8}}{7.3} & \meanstd{\textbf{0.0}}{0.0} & \meanstd{\textbf{0.0}}{0.0} & \meanstd{99.4}{0.2} & \meanstd{\textbf{100.0}}{0.0} & \meanstd{\textbf{0.0}}{0.0} & \meanstd{\textbf{100.0}}{0.0} & \meanstd{2.8}{1.2} & \meanstd{\textbf{0.0}}{0.0} & \meanstd{\textbf{100.0}}{0.0} \\
 & RE & \meanstd{200.0}{0.0} & \meanstd{149.5}{2.3} & \meanstd{143.8}{7.3} & \meanstd{194.3}{7.2} & \meanstd{0.0}{0.0} & \meanstd{0.0}{0.0} & -- & \meanstd{62.0}{2.4} & \meanstd{100.0}{0.0} & \meanstd{100.0}{0.0} & -- \\
\midrule
\multirow{6}{*}{\counterpartyrc} & Claude Opus 5 & \meanstd{352.4}{7.4} & \meanstd{341.1}{7.4} & \meanstd{-8.7}{1.4} & \meanstd{2.6}{1.0} & \meanstd{93.2}{0.9} & \meanstd{94.0}{0.8} & \meanstd{\textbf{0.0}}{0.0} & \meanstd{94.1}{0.8} & \meanstd{28.9}{3.4} & \meanstd{26.1}{3.3} & \meanstd{\textbf{100.0}}{0.0} \\
 & Gemini 3.6 Flash & \meanstd{356.5}{7.6} & \meanstd{340.9}{7.4} & \meanstd{-12.7}{1.5} & \meanstd{2.9}{1.6} & \meanstd{92.3}{0.8} & \meanstd{93.0}{0.7} & \meanstd{\textbf{0.0}}{0.0} & \meanstd{93.0}{0.7} & \meanstd{39.4}{3.7} & \meanstd{36.7}{3.6} & \meanstd{\textbf{100.0}}{0.0} \\
 & GPT 5.6 Sol & \meanstd{\textbf{373.8}}{7.8} & \meanstd{338.6}{7.5} & \meanstd{-30.0}{1.8} & \meanstd{5.1}{1.9} & \meanstd{81.0}{0.9} & \meanstd{82.4}{0.7} & \meanstd{\textbf{0.0}}{0.0} & \meanstd{82.8}{0.7} & \meanstd{90.0}{2.2} & \meanstd{87.2}{2.5} & \meanstd{\textbf{100.0}}{0.0} \\
 & RC & \meanstd{343.0}{7.4} & \meanstd{343.0}{7.4} & \meanstd{0.8}{0.4} & \meanstd{0.8}{0.4} & \meanstd{\textbf{100.0}}{0.0} & \meanstd{\textbf{100.0}}{0.0} & \meanstd{100.0}{0.0} & \meanstd{99.4}{0.2} & \meanstd{\textbf{0.0}}{0.0} & \meanstd{\textbf{0.0}}{0.0} & \meanstd{0.0}{0.0} \\
 & RCC & \meanstd{343.8}{7.3} & \meanstd{\textbf{343.8}}{7.3} & \meanstd{\textbf{0.0}}{0.0} & \meanstd{\textbf{0.0}}{0.0} & \meanstd{99.4}{0.2} & \meanstd{\textbf{100.0}}{0.0} & \meanstd{\textbf{0.0}}{0.0} & \meanstd{\textbf{100.0}}{0.0} & \meanstd{2.8}{1.2} & \meanstd{\textbf{0.0}}{0.0} & \meanstd{\textbf{100.0}}{0.0} \\
 & RE & \meanstd{265.0}{4.5} & \meanstd{161.6}{2.5} & \meanstd{78.7}{4.5} & \meanstd{182.2}{7.3} & \meanstd{0.0}{0.0} & \meanstd{0.0}{0.0} & -- & \meanstd{59.3}{1.9} & \meanstd{100.0}{0.0} & \meanstd{100.0}{0.0} & -- \\
\midrule
\multirow{6}{*}{\counterpartyrnc} & Claude Opus 5 & \meanstd{127.9}{2.2} & \meanstd{127.7}{2.3} & \meanstd{15.8}{1.6} & \meanstd{16.0}{1.6} & \meanstd{54.1}{2.5} & \meanstd{99.7}{0.3} & \meanstd{22.2}{2.6} & \meanstd{89.0}{1.1} & \meanstd{66.7}{3.5} & \meanstd{0.6}{0.6} & \meanstd{89.5}{2.7} \\
 & Gemini 3.6 Flash & \meanstd{145.8}{2.3} & \meanstd{143.6}{2.5} & \meanstd{-2.1}{0.8} & \meanstd{0.1}{0.1} & \meanstd{42.9}{2.6} & \meanstd{98.1}{0.7} & \meanstd{0.8}{0.8} & \meanstd{99.0}{0.4} & \meanstd{73.9}{3.3} & \meanstd{3.9}{1.4} & \meanstd{99.2}{0.8} \\
 & GPT 5.6 Sol & \meanstd{143.3}{2.5} & \meanstd{143.3}{2.5} & \meanstd{0.3}{0.3} & \meanstd{0.4}{0.3} & \meanstd{43.1}{2.6} & \meanstd{99.7}{0.3} & \meanstd{\textbf{0.0}}{0.0} & \meanstd{99.7}{0.3} & \meanstd{75.0}{3.2} & \meanstd{0.6}{0.6} & \meanstd{\textbf{100.0}}{0.0} \\
 & RC & \meanstd{54.6}{4.9} & \meanstd{54.6}{4.9} & \meanstd{89.1}{5.0} & \meanstd{89.1}{5.0} & \meanstd{\textbf{100.0}}{0.0} & \meanstd{\textbf{100.0}}{0.0} & \meanstd{100.0}{0.0} & \meanstd{43.3}{2.6} & \meanstd{\textbf{0.0}}{0.0} & \meanstd{\textbf{0.0}}{0.0} & \meanstd{0.0}{0.0} \\
 & RCC & \meanstd{143.7}{2.5} & \meanstd{143.7}{2.5} & \meanstd{\textbf{0.0}}{0.0} & \meanstd{\textbf{0.0}}{0.0} & \meanstd{43.3}{2.6} & \meanstd{\textbf{100.0}}{0.0} & \meanstd{\textbf{0.0}}{0.0} & \meanstd{\textbf{100.0}}{0.0} & \meanstd{74.4}{3.3} & \meanstd{\textbf{0.0}}{0.0} & \meanstd{\textbf{100.0}}{0.0} \\
 & RE & \meanstd{\textbf{200.0}}{0.0} & \meanstd{\textbf{149.5}}{2.3} & \meanstd{-56.3}{2.5} & \meanstd{-5.8}{1.1} & \meanstd{0.0}{0.0} & \meanstd{0.0}{0.0} & -- & \meanstd{62.0}{2.4} & \meanstd{100.0}{0.0} & \meanstd{100.0}{0.0} & -- \\
\bottomrule
\end{tabular}
\end{table*}
\endgroup

\begingroup
\def\meanstd#1#2{#1{\scriptsize\,$\pm$#2}}
\def\counterpartyrcc{\shortstack{\emph{Rational}\\\emph{Conditional}\\\emph{Complier}\\(RCC)}}
\def\counterpartyrc{\shortstack{\emph{Rational}\\\emph{Complier}\\(RC)}}
\def\counterpartyrnc{\shortstack{\emph{Rational}\\\emph{Exploiter}\\(RE)}}
\begin{table*}[h]
\footnotesize
\caption{\textbf{Supplier-side performance metrics against RCC, RC, and RE Customer counterparties.} Entries are equal-run-weight means $\pm$ standard errors across 180 environment--contract cases; each compliance rate is computed within a case before averaging across cases. Dashes denote undefined rates.}
\label{tab:caterer-performance-rational-models-appendix}
\centering
\setlength{\tabcolsep}{2.5pt}
\begin{tabular}{@{}ccccccccccccc@{}}
\toprule
\multirow{2}{*}[-2pt]{\shortstack{\textbf{Counterparty}\\(\emph{Customer})}} & \multirow{2}{*}[-2pt]{\shortstack{\textbf{Agent}\\(\emph{Supplier})}} & \multicolumn{2}{c}{\textbf{Utility}} & \multicolumn{2}{c}{\textbf{(Ex-Post) Regret}} & \multicolumn{4}{c}{\textbf{Compliance}} & \multicolumn{3}{c}{\textbf{Defection}} \\
\cmidrule(lr){3-4}\cmidrule(lr){5-6}\cmidrule(lr){7-10}\cmidrule(lr){11-13}
 & & \textbf{Utility}\,$\uparrow$ & \textbf{Compl.}\,$\uparrow$ & \textbf{Regret}\,$\to 0$ & \textbf{Compl.}\,$\to 0$ & \textbf{Rate} & \textbf{Cond.}\,$\uparrow$ & \textbf{Expl.}\,$\downarrow$ & \textbf{TFT}\,$\uparrow$ & \textbf{Rate} & \textbf{Unil.}\,$\downarrow$ & \textbf{Recip.}\,$\uparrow$ \\
\midrule
\multirow{6}{*}{\counterpartyrcc} & Claude Opus 5 & \meanstd{53.9}{3.2} & \meanstd{41.1}{2.9} & \meanstd{-5.0}{2.2} & \meanstd{7.8}{2.0} & \meanstd{82.7}{1.9} & \meanstd{85.1}{1.7} & -- & \meanstd{88.3}{1.3} & \meanstd{41.7}{3.7} & \meanstd{41.7}{3.7} & -- \\
 & Gemini 3.6 Flash & \meanstd{40.7}{3.2} & \meanstd{39.5}{3.2} & \meanstd{8.3}{1.7} & \meanstd{9.5}{1.6} & \meanstd{95.2}{1.0} & \meanstd{95.8}{0.9} & -- & \meanstd{96.6}{0.6} & \meanstd{16.1}{2.7} & \meanstd{16.1}{2.7} & -- \\
 & GPT 5.6 Sol & \meanstd{54.7}{3.0} & \meanstd{44.9}{2.6} & \meanstd{-5.7}{1.8} & \meanstd{4.0}{1.7} & \meanstd{86.3}{2.0} & \meanstd{88.4}{1.7} & -- & \meanstd{91.9}{1.2} & \meanstd{28.9}{3.4} & \meanstd{28.9}{3.4} & -- \\
 & RC & \meanstd{49.0}{2.7} & \meanstd{\textbf{49.0}}{2.7} & \meanstd{\textbf{0.0}}{0.0} & \meanstd{\textbf{0.0}}{0.0} & \meanstd{\textbf{99.4}}{0.2} & \meanstd{\textbf{99.4}}{0.2} & -- & \meanstd{\textbf{99.4}}{0.2} & \meanstd{\textbf{2.8}}{1.2} & \meanstd{\textbf{2.8}}{1.2} & -- \\
 & RCC & \meanstd{49.0}{2.7} & \meanstd{\textbf{49.0}}{2.7} & \meanstd{\textbf{0.0}}{0.0} & \meanstd{\textbf{0.0}}{0.0} & \meanstd{\textbf{99.4}}{0.2} & \meanstd{\textbf{99.4}}{0.2} & -- & \meanstd{\textbf{99.4}}{0.2} & \meanstd{\textbf{2.8}}{1.2} & \meanstd{\textbf{2.8}}{1.2} & -- \\
 & RE & \meanstd{\textbf{56.3}}{2.5} & \meanstd{10.2}{1.6} & \meanstd{-7.3}{4.0} & \meanstd{38.8}{3.7} & \meanstd{4.9}{0.6} & \meanstd{8.6}{1.3} & -- & \meanstd{61.8}{2.3} & \meanstd{100.0}{0.0} & \meanstd{100.0}{0.0} & -- \\
\midrule
\multirow{6}{*}{\counterpartyrc} & Claude Opus 5 & \meanstd{63.2}{3.6} & \meanstd{41.8}{2.7} & \meanstd{-13.5}{2.9} & \meanstd{7.9}{2.1} & \meanstd{82.3}{1.9} & \meanstd{82.3}{1.9} & -- & \meanstd{82.3}{1.9} & \meanstd{43.9}{3.7} & \meanstd{43.9}{3.7} & -- \\
 & Gemini 3.6 Flash & \meanstd{43.7}{2.9} & \meanstd{40.3}{3.0} & \meanstd{6.1}{1.1} & \meanstd{9.4}{1.2} & \meanstd{96.8}{0.6} & \meanstd{96.8}{0.6} & -- & \meanstd{96.8}{0.6} & \meanstd{15.0}{2.7} & \meanstd{15.0}{2.7} & -- \\
 & GPT 5.6 Sol & \meanstd{64.1}{3.7} & \meanstd{42.7}{2.6} & \meanstd{-14.3}{2.8} & \meanstd{7.0}{1.9} & \meanstd{84.2}{2.1} & \meanstd{84.2}{2.1} & -- & \meanstd{84.2}{2.1} & \meanstd{32.2}{3.5} & \meanstd{32.2}{3.5} & -- \\
 & RC & \meanstd{49.8}{2.6} & \meanstd{\textbf{49.7}}{2.6} & \meanstd{\textbf{0.0}}{0.0} & \meanstd{\textbf{0.0}}{0.0} & \meanstd{\textbf{99.4}}{0.2} & \meanstd{\textbf{99.4}}{0.2} & -- & \meanstd{\textbf{99.4}}{0.2} & \meanstd{\textbf{2.8}}{1.2} & \meanstd{\textbf{2.8}}{1.2} & -- \\
 & RCC & \meanstd{49.8}{2.6} & \meanstd{\textbf{49.7}}{2.6} & \meanstd{\textbf{0.0}}{0.0} & \meanstd{\textbf{0.0}}{0.0} & \meanstd{\textbf{99.4}}{0.2} & \meanstd{\textbf{99.4}}{0.2} & -- & \meanstd{\textbf{99.4}}{0.2} & \meanstd{\textbf{2.8}}{1.2} & \meanstd{\textbf{2.8}}{1.2} & -- \\
 & RE & \meanstd{\textbf{145.4}}{4.9} & \meanstd{10.2}{1.6} & \meanstd{-95.6}{5.6} & \meanstd{39.5}{3.6} & \meanstd{4.9}{0.6} & \meanstd{4.9}{0.6} & -- & \meanstd{4.9}{0.6} & \meanstd{100.0}{0.0} & \meanstd{100.0}{0.0} & -- \\
\midrule
\multirow{6}{*}{\counterpartyrnc} & Claude Opus 5 & \meanstd{-19.6}{1.0} & \meanstd{-19.6}{1.0} & \meanstd{19.6}{1.0} & \meanstd{19.6}{1.0} & \meanstd{9.9}{0.8} & -- & \meanstd{9.9}{0.8} & \meanstd{90.1}{0.8} & \meanstd{\textbf{100.0}}{0.0} & -- & \meanstd{\textbf{100.0}}{0.0} \\
 & Gemini 3.6 Flash & \meanstd{-11.4}{1.0} & \meanstd{-11.4}{1.0} & \meanstd{11.4}{1.0} & \meanstd{11.4}{1.0} & \meanstd{5.8}{0.7} & -- & \meanstd{5.8}{0.7} & \meanstd{94.2}{0.7} & \meanstd{\textbf{100.0}}{0.0} & -- & \meanstd{\textbf{100.0}}{0.0} \\
 & GPT 5.6 Sol & \meanstd{\textbf{0.0}}{0.0} & \meanstd{\textbf{0.0}}{0.0} & \meanstd{\textbf{0.0}}{0.0} & \meanstd{\textbf{0.0}}{0.0} & \meanstd{0.0}{0.0} & -- & \meanstd{\textbf{0.0}}{0.0} & \meanstd{\textbf{100.0}}{0.0} & \meanstd{\textbf{100.0}}{0.0} & -- & \meanstd{\textbf{100.0}}{0.0} \\
 & RC & \meanstd{-28.4}{0.8} & \meanstd{-28.4}{0.8} & \meanstd{28.4}{0.8} & \meanstd{28.4}{0.8} & \meanstd{\textbf{11.9}}{1.0} & -- & \meanstd{11.9}{1.0} & \meanstd{88.1}{1.0} & \meanstd{\textbf{100.0}}{0.0} & -- & \meanstd{\textbf{100.0}}{0.0} \\
 & RCC & \meanstd{\textbf{0.0}}{0.0} & \meanstd{\textbf{0.0}}{0.0} & \meanstd{\textbf{0.0}}{0.0} & \meanstd{\textbf{0.0}}{0.0} & \meanstd{0.0}{0.0} & -- & \meanstd{\textbf{0.0}}{0.0} & \meanstd{\textbf{100.0}}{0.0} & \meanstd{\textbf{100.0}}{0.0} & -- & \meanstd{\textbf{100.0}}{0.0} \\
 & RE & \meanstd{\textbf{0.0}}{0.0} & \meanstd{\textbf{0.0}}{0.0} & \meanstd{\textbf{0.0}}{0.0} & \meanstd{\textbf{0.0}}{0.0} & \meanstd{0.0}{0.0} & -- & \meanstd{\textbf{0.0}}{0.0} & \meanstd{\textbf{100.0}}{0.0} & \meanstd{\textbf{100.0}}{0.0} & -- & \meanstd{\textbf{100.0}}{0.0} \\
\bottomrule
\end{tabular}
\end{table*}
\endgroup

\clearpage
\subsection{Pooled-Week Compliance Sensitivity}
\label{app:micro-compliance-tables}

Tables~\ref{tab:customer-performance-rational-models-micro}
and~\ref{tab:caterer-performance-rational-models-micro} reproduce the full
performance comparison against all three counterparty policies (RCC, RC, and
RE), including all rational controls, after pooling eligible weeks for
overall, conditional, exploited, and tit-for-tat compliance. The remaining
metrics are identical to the equal-run-weight tables in
Appendix~\ref{app:performance-re-counterparties}. The aggregation change does
not reverse any qualitative comparison: conditional and exploited compliance
remain similar, while the lower overall Customer compliance against RE
strengthens the finding that LLM agents cease compliance after exploitation.

\begingroup
\def\meanstd#1#2{#1{\scriptsize\,$\pm$#2}}
\def\counterpartyrcc{\shortstack{\emph{Rational}\\\emph{Conditional}\\\emph{Complier}\\(RCC)}}
\def\counterpartyrc{\shortstack{\emph{Rational}\\\emph{Complier}\\(RC)}}
\def\counterpartyrnc{\shortstack{\emph{Rational}\\\emph{Exploiter}\\(RE)}}
\begin{table*}[h]
\footnotesize
\caption{\textbf{Customer-side performance metrics against RCC, RC, and RE Supplier counterparties.} The four compliance metrics pool their eligible weeks across environment--contract cases; their dispersion is likewise weighted by eligible weeks. All other entries are equal-run-weight means $\pm$ standard errors across 180 cases. Dashes denote undefined rates.}
\label{tab:customer-performance-rational-models-micro}
\centering
\setlength{\tabcolsep}{2.5pt}
\begin{tabular}{@{}ccccccccccccc@{}}
\toprule
\multirow{2}{*}[-2pt]{\shortstack{\textbf{Counterparty}\\(\emph{Supplier})}} & \multirow{2}{*}[-2pt]{\shortstack{\textbf{Agent}\\(\emph{Customer})}} & \multicolumn{2}{c}{\textbf{Utility}} & \multicolumn{2}{c}{\textbf{(Ex-Post) Regret}} & \multicolumn{4}{c}{\textbf{Compliance}} & \multicolumn{3}{c}{\textbf{Defection}} \\
\cmidrule(lr){3-4}\cmidrule(lr){5-6}\cmidrule(lr){7-10}\cmidrule(lr){11-13}
 & & \textbf{Utility}\,$\uparrow$ & \textbf{Compl.}\,$\uparrow$ & \textbf{Regret}\,$\to 0$ & \textbf{Compl.}\,$\to 0$ & \textbf{Rate} & \textbf{Cond.}\,$\uparrow$ & \textbf{Expl.}\,$\downarrow$ & \textbf{TFT}\,$\uparrow$ & \textbf{Rate} & \textbf{Unil.}\,$\downarrow$ & \textbf{Recip.}\,$\uparrow$ \\
\midrule
\multirow{6}{*}{\counterpartyrcc} & Claude Opus 5 & \meanstd{347.1}{7.7} & \meanstd{338.6}{7.6} & \meanstd{-3.3}{1.6} & \meanstd{5.2}{1.6} & \meanstd{93.2}{1.0} & \meanstd{95.0}{0.6} & \meanstd{\textbf{0.0}}{0.0} & \meanstd{94.9}{0.6} & \meanstd{30.0}{3.4} & \meanstd{27.2}{3.3} & \meanstd{\textbf{100.0}}{0.0} \\
 & Gemini 3.6 Flash & \meanstd{355.8}{7.8} & \meanstd{342.1}{7.4} & \meanstd{-12.1}{1.7} & \meanstd{1.6}{1.0} & \meanstd{91.6}{0.7} & \meanstd{92.0}{0.7} & \meanstd{\textbf{0.0}}{0.0} & \meanstd{92.1}{0.7} & \meanstd{47.2}{3.7} & \meanstd{44.4}{3.7} & \meanstd{\textbf{100.0}}{0.0} \\
 & GPT 5.6 Sol & \meanstd{\textbf{366.1}}{7.9} & \meanstd{335.9}{7.6} & \meanstd{-22.3}{2.3} & \meanstd{7.8}{2.4} & \meanstd{80.7}{0.8} & \meanstd{83.4}{0.4} & \meanstd{\textbf{0.0}}{0.0} & \meanstd{83.9}{0.4} & \meanstd{92.8}{1.9} & \meanstd{90.0}{2.2} & \meanstd{\textbf{100.0}}{0.0} \\
 & RC & \meanstd{343.0}{7.4} & \meanstd{343.0}{7.4} & \meanstd{0.8}{0.4} & \meanstd{0.8}{0.4} & \meanstd{\textbf{100.0}}{0.0} & \meanstd{\textbf{100.0}}{0.0} & \meanstd{100.0}{0.0} & \meanstd{99.5}{0.2} & \meanstd{\textbf{0.0}}{0.0} & \meanstd{\textbf{0.0}}{0.0} & \meanstd{0.0}{0.0} \\
 & RCC & \meanstd{343.8}{7.3} & \meanstd{\textbf{343.8}}{7.3} & \meanstd{\textbf{0.0}}{0.0} & \meanstd{\textbf{0.0}}{0.0} & \meanstd{99.5}{0.2} & \meanstd{\textbf{100.0}}{0.0} & \meanstd{\textbf{0.0}}{0.0} & \meanstd{\textbf{100.0}}{0.0} & \meanstd{2.8}{1.2} & \meanstd{\textbf{0.0}}{0.0} & \meanstd{\textbf{100.0}}{0.0} \\
 & RE & \meanstd{200.0}{0.0} & \meanstd{149.5}{2.3} & \meanstd{143.8}{7.3} & \meanstd{194.3}{7.2} & \meanstd{0.0}{0.0} & \meanstd{0.0}{0.0} & -- & \meanstd{76.4}{0.9} & \meanstd{100.0}{0.0} & \meanstd{100.0}{0.0} & -- \\
\midrule
\multirow{6}{*}{\counterpartyrc} & Claude Opus 5 & \meanstd{352.4}{7.4} & \meanstd{341.1}{7.4} & \meanstd{-8.7}{1.4} & \meanstd{2.6}{1.0} & \meanstd{93.3}{0.9} & \meanstd{94.2}{0.8} & \meanstd{\textbf{0.0}}{0.0} & \meanstd{94.1}{0.8} & \meanstd{28.9}{3.4} & \meanstd{26.1}{3.3} & \meanstd{\textbf{100.0}}{0.0} \\
 & Gemini 3.6 Flash & \meanstd{356.5}{7.6} & \meanstd{340.9}{7.4} & \meanstd{-12.7}{1.5} & \meanstd{2.9}{1.6} & \meanstd{92.3}{0.8} & \meanstd{92.9}{0.7} & \meanstd{\textbf{0.0}}{0.0} & \meanstd{92.9}{0.7} & \meanstd{39.4}{3.7} & \meanstd{36.7}{3.6} & \meanstd{\textbf{100.0}}{0.0} \\
 & GPT 5.6 Sol & \meanstd{\textbf{373.8}}{7.8} & \meanstd{338.6}{7.5} & \meanstd{-30.0}{1.8} & \meanstd{5.1}{1.9} & \meanstd{81.1}{0.8} & \meanstd{82.6}{0.6} & \meanstd{\textbf{0.0}}{0.0} & \meanstd{82.8}{0.6} & \meanstd{90.0}{2.2} & \meanstd{87.2}{2.5} & \meanstd{\textbf{100.0}}{0.0} \\
 & RC & \meanstd{343.0}{7.4} & \meanstd{343.0}{7.4} & \meanstd{0.8}{0.4} & \meanstd{0.8}{0.4} & \meanstd{\textbf{100.0}}{0.0} & \meanstd{\textbf{100.0}}{0.0} & \meanstd{100.0}{0.0} & \meanstd{99.5}{0.2} & \meanstd{\textbf{0.0}}{0.0} & \meanstd{\textbf{0.0}}{0.0} & \meanstd{0.0}{0.0} \\
 & RCC & \meanstd{343.8}{7.3} & \meanstd{\textbf{343.8}}{7.3} & \meanstd{\textbf{0.0}}{0.0} & \meanstd{\textbf{0.0}}{0.0} & \meanstd{99.5}{0.2} & \meanstd{\textbf{100.0}}{0.0} & \meanstd{\textbf{0.0}}{0.0} & \meanstd{\textbf{100.0}}{0.0} & \meanstd{2.8}{1.2} & \meanstd{\textbf{0.0}}{0.0} & \meanstd{\textbf{100.0}}{0.0} \\
 & RE & \meanstd{265.0}{4.5} & \meanstd{161.6}{2.5} & \meanstd{78.7}{4.5} & \meanstd{182.2}{7.3} & \meanstd{0.0}{0.0} & \meanstd{0.0}{0.0} & -- & \meanstd{66.4}{1.4} & \meanstd{100.0}{0.0} & \meanstd{100.0}{0.0} & -- \\
\midrule
\multirow{6}{*}{\counterpartyrnc} & Claude Opus 5 & \meanstd{127.9}{2.2} & \meanstd{127.7}{2.3} & \meanstd{15.8}{1.6} & \meanstd{16.0}{1.6} & \meanstd{38.8}{1.4} & \meanstd{99.5}{0.5} & \meanstd{15.5}{1.4} & \meanstd{88.7}{1.0} & \meanstd{66.7}{3.5} & \meanstd{0.6}{0.6} & \meanstd{89.5}{2.7} \\
 & Gemini 3.6 Flash & \meanstd{145.8}{2.3} & \meanstd{143.6}{2.5} & \meanstd{-2.1}{0.8} & \meanstd{0.1}{0.1} & \meanstd{27.0}{1.2} & \meanstd{96.7}{1.2} & \meanstd{0.2}{0.2} & \meanstd{99.0}{0.4} & \meanstd{73.9}{3.3} & \meanstd{3.9}{1.4} & \meanstd{99.2}{0.8} \\
 & GPT 5.6 Sol & \meanstd{143.3}{2.5} & \meanstd{143.3}{2.5} & \meanstd{0.3}{0.3} & \meanstd{0.4}{0.3} & \meanstd{27.7}{1.2} & \meanstd{99.5}{0.5} & \meanstd{\textbf{0.0}}{0.0} & \meanstd{99.9}{0.1} & \meanstd{75.0}{3.2} & \meanstd{0.6}{0.6} & \meanstd{\textbf{100.0}}{0.0} \\
 & RC & \meanstd{54.6}{4.9} & \meanstd{54.6}{4.9} & \meanstd{89.1}{5.0} & \meanstd{89.1}{5.0} & \meanstd{\textbf{100.0}}{0.0} & \meanstd{\textbf{100.0}}{0.0} & \meanstd{100.0}{0.0} & \meanstd{27.8}{1.2} & \meanstd{\textbf{0.0}}{0.0} & \meanstd{\textbf{0.0}}{0.0} & \meanstd{0.0}{0.0} \\
 & RCC & \meanstd{143.7}{2.5} & \meanstd{143.7}{2.5} & \meanstd{\textbf{0.0}}{0.0} & \meanstd{\textbf{0.0}}{0.0} & \meanstd{27.8}{1.2} & \meanstd{\textbf{100.0}}{0.0} & \meanstd{\textbf{0.0}}{0.0} & \meanstd{\textbf{100.0}}{0.0} & \meanstd{74.4}{3.3} & \meanstd{\textbf{0.0}}{0.0} & \meanstd{\textbf{100.0}}{0.0} \\
 & RE & \meanstd{\textbf{200.0}}{0.0} & \meanstd{\textbf{149.5}}{2.3} & \meanstd{-56.3}{2.5} & \meanstd{-5.8}{1.1} & \meanstd{0.0}{0.0} & \meanstd{0.0}{0.0} & -- & \meanstd{76.4}{0.9} & \meanstd{100.0}{0.0} & \meanstd{100.0}{0.0} & -- \\
\bottomrule
\end{tabular}
\end{table*}
\endgroup

\begingroup
\def\meanstd#1#2{#1{\scriptsize\,$\pm$#2}}
\def\counterpartyrcc{\shortstack{\emph{Rational}\\\emph{Conditional}\\\emph{Complier}\\(RCC)}}
\def\counterpartyrc{\shortstack{\emph{Rational}\\\emph{Complier}\\(RC)}}
\def\counterpartyrnc{\shortstack{\emph{Rational}\\\emph{Exploiter}\\(RE)}}
\begin{table*}[h]
\footnotesize
\caption{\textbf{Supplier-side performance metrics against RCC, RC, and RE Customer counterparties.} The four compliance metrics pool their eligible weeks across environment--contract cases; their dispersion is likewise weighted by eligible weeks. All other entries are equal-run-weight means $\pm$ standard errors across 180 cases. Dashes denote undefined rates.}
\label{tab:caterer-performance-rational-models-micro}
\centering
\setlength{\tabcolsep}{2.5pt}
\begin{tabular}{@{}ccccccccccccc@{}}
\toprule
\multirow{2}{*}[-2pt]{\shortstack{\textbf{Counterparty}\\(\emph{Customer})}} & \multirow{2}{*}[-2pt]{\shortstack{\textbf{Agent}\\(\emph{Supplier})}} & \multicolumn{2}{c}{\textbf{Utility}} & \multicolumn{2}{c}{\textbf{(Ex-Post) Regret}} & \multicolumn{4}{c}{\textbf{Compliance}} & \multicolumn{3}{c}{\textbf{Defection}} \\
\cmidrule(lr){3-4}\cmidrule(lr){5-6}\cmidrule(lr){7-10}\cmidrule(lr){11-13}
 & & \textbf{Utility}\,$\uparrow$ & \textbf{Compl.}\,$\uparrow$ & \textbf{Regret}\,$\to 0$ & \textbf{Compl.}\,$\to 0$ & \textbf{Rate} & \textbf{Cond.}\,$\uparrow$ & \textbf{Expl.}\,$\downarrow$ & \textbf{TFT}\,$\uparrow$ & \textbf{Rate} & \textbf{Unil.}\,$\downarrow$ & \textbf{Recip.}\,$\uparrow$ \\
\midrule
\multirow{6}{*}{\counterpartyrcc} & Claude Opus 5 & \meanstd{53.9}{3.2} & \meanstd{41.1}{2.9} & \meanstd{-5.0}{2.2} & \meanstd{7.8}{2.0} & \meanstd{82.7}{1.9} & \meanstd{87.8}{1.4} & -- & \meanstd{88.3}{1.3} & \meanstd{41.7}{3.7} & \meanstd{41.7}{3.7} & -- \\
 & Gemini 3.6 Flash & \meanstd{40.7}{3.2} & \meanstd{39.5}{3.2} & \meanstd{8.3}{1.7} & \meanstd{9.5}{1.6} & \meanstd{95.2}{1.0} & \meanstd{96.5}{0.6} & -- & \meanstd{96.6}{0.6} & \meanstd{16.1}{2.7} & \meanstd{16.1}{2.7} & -- \\
 & GPT 5.6 Sol & \meanstd{54.7}{3.0} & \meanstd{44.9}{2.6} & \meanstd{-5.7}{1.8} & \meanstd{4.0}{1.7} & \meanstd{86.3}{2.0} & \meanstd{91.4}{1.3} & -- & \meanstd{91.9}{1.2} & \meanstd{28.9}{3.4} & \meanstd{28.9}{3.4} & -- \\
 & RC & \meanstd{49.0}{2.7} & \meanstd{\textbf{49.0}}{2.7} & \meanstd{\textbf{0.0}}{0.0} & \meanstd{\textbf{0.0}}{0.0} & \meanstd{\textbf{99.4}}{0.2} & \meanstd{\textbf{99.4}}{0.2} & -- & \meanstd{\textbf{99.4}}{0.2} & \meanstd{\textbf{2.8}}{1.2} & \meanstd{\textbf{2.8}}{1.2} & -- \\
 & RCC & \meanstd{49.0}{2.7} & \meanstd{\textbf{49.0}}{2.7} & \meanstd{\textbf{0.0}}{0.0} & \meanstd{\textbf{0.0}}{0.0} & \meanstd{\textbf{99.4}}{0.2} & \meanstd{\textbf{99.4}}{0.2} & -- & \meanstd{\textbf{99.4}}{0.2} & \meanstd{\textbf{2.8}}{1.2} & \meanstd{\textbf{2.8}}{1.2} & -- \\
 & RE & \meanstd{\textbf{56.3}}{2.5} & \meanstd{10.2}{1.6} & \meanstd{-7.3}{4.0} & \meanstd{38.8}{3.7} & \meanstd{4.9}{0.6} & \meanstd{11.3}{1.3} & -- & \meanstd{61.8}{2.3} & \meanstd{100.0}{0.0} & \meanstd{100.0}{0.0} & -- \\
\midrule
\multirow{6}{*}{\counterpartyrc} & Claude Opus 5 & \meanstd{63.2}{3.6} & \meanstd{41.8}{2.7} & \meanstd{-13.5}{2.9} & \meanstd{7.9}{2.1} & \meanstd{82.3}{1.9} & \meanstd{82.3}{1.9} & -- & \meanstd{82.3}{1.9} & \meanstd{43.9}{3.7} & \meanstd{43.9}{3.7} & -- \\
 & Gemini 3.6 Flash & \meanstd{43.7}{2.9} & \meanstd{40.3}{3.0} & \meanstd{6.1}{1.1} & \meanstd{9.4}{1.2} & \meanstd{96.8}{0.6} & \meanstd{96.8}{0.6} & -- & \meanstd{96.8}{0.6} & \meanstd{15.0}{2.7} & \meanstd{15.0}{2.7} & -- \\
 & GPT 5.6 Sol & \meanstd{64.1}{3.7} & \meanstd{42.7}{2.6} & \meanstd{-14.3}{2.8} & \meanstd{7.0}{1.9} & \meanstd{84.2}{2.1} & \meanstd{84.2}{2.1} & -- & \meanstd{84.2}{2.1} & \meanstd{32.2}{3.5} & \meanstd{32.2}{3.5} & -- \\
 & RC & \meanstd{49.8}{2.6} & \meanstd{\textbf{49.7}}{2.6} & \meanstd{\textbf{0.0}}{0.0} & \meanstd{\textbf{0.0}}{0.0} & \meanstd{\textbf{99.4}}{0.2} & \meanstd{\textbf{99.4}}{0.2} & -- & \meanstd{\textbf{99.4}}{0.2} & \meanstd{\textbf{2.8}}{1.2} & \meanstd{\textbf{2.8}}{1.2} & -- \\
 & RCC & \meanstd{49.8}{2.6} & \meanstd{\textbf{49.7}}{2.6} & \meanstd{\textbf{0.0}}{0.0} & \meanstd{\textbf{0.0}}{0.0} & \meanstd{\textbf{99.4}}{0.2} & \meanstd{\textbf{99.4}}{0.2} & -- & \meanstd{\textbf{99.4}}{0.2} & \meanstd{\textbf{2.8}}{1.2} & \meanstd{\textbf{2.8}}{1.2} & -- \\
 & RE & \meanstd{\textbf{145.4}}{4.9} & \meanstd{10.2}{1.6} & \meanstd{-95.6}{5.6} & \meanstd{39.5}{3.6} & \meanstd{4.9}{0.6} & \meanstd{4.9}{0.6} & -- & \meanstd{4.9}{0.6} & \meanstd{100.0}{0.0} & \meanstd{100.0}{0.0} & -- \\
\midrule
\multirow{6}{*}{\counterpartyrnc} & Claude Opus 5 & \meanstd{-19.6}{1.0} & \meanstd{-19.6}{1.0} & \meanstd{19.6}{1.0} & \meanstd{19.6}{1.0} & \meanstd{9.9}{0.8} & -- & \meanstd{9.9}{0.8} & \meanstd{90.1}{0.8} & \meanstd{\textbf{100.0}}{0.0} & -- & \meanstd{\textbf{100.0}}{0.0} \\
 & Gemini 3.6 Flash & \meanstd{-11.4}{1.0} & \meanstd{-11.4}{1.0} & \meanstd{11.4}{1.0} & \meanstd{11.4}{1.0} & \meanstd{5.8}{0.7} & -- & \meanstd{5.8}{0.7} & \meanstd{94.2}{0.7} & \meanstd{\textbf{100.0}}{0.0} & -- & \meanstd{\textbf{100.0}}{0.0} \\
 & GPT 5.6 Sol & \meanstd{\textbf{0.0}}{0.0} & \meanstd{\textbf{0.0}}{0.0} & \meanstd{\textbf{0.0}}{0.0} & \meanstd{\textbf{0.0}}{0.0} & \meanstd{0.0}{0.0} & -- & \meanstd{\textbf{0.0}}{0.0} & \meanstd{\textbf{100.0}}{0.0} & \meanstd{\textbf{100.0}}{0.0} & -- & \meanstd{\textbf{100.0}}{0.0} \\
 & RC & \meanstd{-28.4}{0.8} & \meanstd{-28.4}{0.8} & \meanstd{28.4}{0.8} & \meanstd{28.4}{0.8} & \meanstd{\textbf{11.9}}{1.0} & -- & \meanstd{11.9}{1.0} & \meanstd{88.1}{1.0} & \meanstd{\textbf{100.0}}{0.0} & -- & \meanstd{\textbf{100.0}}{0.0} \\
 & RCC & \meanstd{\textbf{0.0}}{0.0} & \meanstd{\textbf{0.0}}{0.0} & \meanstd{\textbf{0.0}}{0.0} & \meanstd{\textbf{0.0}}{0.0} & \meanstd{0.0}{0.0} & -- & \meanstd{\textbf{0.0}}{0.0} & \meanstd{\textbf{100.0}}{0.0} & \meanstd{\textbf{100.0}}{0.0} & -- & \meanstd{\textbf{100.0}}{0.0} \\
 & RE & \meanstd{\textbf{0.0}}{0.0} & \meanstd{\textbf{0.0}}{0.0} & \meanstd{\textbf{0.0}}{0.0} & \meanstd{\textbf{0.0}}{0.0} & \meanstd{0.0}{0.0} & -- & \meanstd{\textbf{0.0}}{0.0} & \meanstd{\textbf{100.0}}{0.0} & \meanstd{\textbf{100.0}}{0.0} & -- & \meanstd{\textbf{100.0}}{0.0} \\
\bottomrule
\end{tabular}
\end{table*}
\endgroup

\FloatBarrier
\clearpage
\subsection{Unilateral Defection Types against RCC}
\label{app:defection-types}

Table~\ref{tab:defection-types} classifies each LLM agent's first unilateral
defection against RCC from its breach-week reasoning trace and order, delivery,
and payment history. Labels were assigned in a single pass by Claude Opus 5.5 (high reasoning): \emph{deliberate} (knowingly
under-performs for profit), \emph{planning failure} (intends to comply but
lacks inputs or cash, or its action diverges from its stated plan), or
\emph{misread terms} (believes the action is compliant or excused under a
misreading of the contract). Defections without usable reasoning are excluded.

\begin{table*}[htbp]
\footnotesize
\caption{\textbf{Types of unilateral defection against RCC counterparties.}
Type shares are percentages of classified defections; means are unweighted
across models.}
\label{tab:defection-types}
\centering
\setlength{\tabcolsep}{4pt}
\begin{tabular}{@{}llcccccc@{}}
\toprule
\textbf{Role} & \textbf{Agent} & \shortstack{\textbf{Unilateral}\\\textbf{defection rate (\%)}} & \shortstack{\textbf{Classified}\\\textbf{defections}} & \shortstack{\textbf{Deliberate}\\\textbf{(\%)}} & \shortstack{\textbf{Planning}\\\textbf{failure (\%)}} & \shortstack{\textbf{Misread}\\\textbf{terms (\%)}} & \shortstack{\textbf{Excluded}\\\textbf{defections}} \\
\midrule
\multirow{4}{*}{Customer}
 & Claude Opus 5    & 27.2 & 48 & 89.6 & 0.0 & 10.4 & 1 \\
 & Gemini 3.6 Flash & 44.4 & 80 & 96.3 & 0.0 & 3.8  & 0 \\
 & GPT 5.6 Sol      & 90.0 & 99 & 94.9 & 0.0 & 5.1  & 63 \\
 \cmidrule(l){2-8}
 & Mean             & 53.9 & -- & 93.6 & 0.0 & 6.4  & -- \\
\midrule
\multirow{4}{*}{Supplier}
 & Claude Opus 5    & 41.7 & 75 & 89.3 & 9.3  & 1.3  & 0 \\
 & Gemini 3.6 Flash & 16.1 & 29 & 27.6 & 65.5 & 6.9  & 0 \\
 & GPT 5.6 Sol      & 28.9 & 44 & 50.0 & 31.8 & 18.2 & 8 \\
 \cmidrule(l){2-8}
 & Mean             & 28.9 & -- & 55.6 & 35.6 & 8.8  & -- \\
\bottomrule
\end{tabular}
\end{table*}

\FloatBarrier
\clearpage
\subsection{Negotiation and Performance Metrics across Multiple Domains}
\label{app:negotiation-domain-robustness}

Figure~\ref{fig:negotiation-domain-robustness} and
Table~\ref{tab:performance-domain-summary} expand the matched-setting
comparison summarized in Section~\ref{sec:prompt-guidance} of the main paper. The figure reports
the full set of negotiation metrics pooled across roles, while the table
reports role-specific payoff metrics and performance behavior. Both use the
same role and counterparty aggregation as the corresponding quantities in
Figure~\ref{fig:domain-robustness} of the main paper. We keep the underlying high-stochasticity Environment~5 fixed
while expressing the task as Catering, Hotel Cleaning, or AI Hosting.

\begin{center}
\centering
\includegraphics[width=\textwidth]{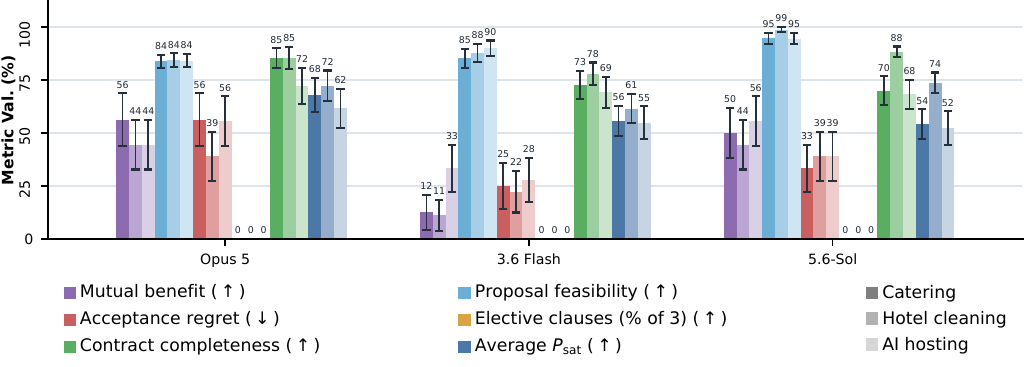}
\captionof{figure}{\textbf{Negotiation robustness across matched supplier settings.}
Bars show mean $\pm$ SE contract quality pooled across Customer and Supplier roles in matched Catering,
Hotel Cleaning, and AI Hosting settings, keeping the underlying environment
fixed (high stochasticity).}
\label{fig:negotiation-domain-robustness}
\end{center}

\begingroup

\begin{center}
\captionof{table}{\textbf{Performance robustness across matched supplier settings.} Entries are means across RCC, RC, and RE counterparties. Utility and regret are reported separately for Customer and Supplier because their scales differ; Supplier utility is profit. Behavioral rates (\%) pool both roles and all counterparties: rates are computed within a case and macro-averaged across eligible cases.}
\label{tab:performance-domain-summary}
\centering
\footnotesize
\setlength{\tabcolsep}{1.3pt}
\begin{tabular}{@{}ccccccccccccccccc@{}}
\toprule
\multirow{3}{*}[-2pt]{\textbf{Setting}} & \multirow{3}{*}[-2pt]{\textbf{Agent}} & \multicolumn{4}{c}{\textbf{Customer}} & \multicolumn{4}{c}{\textbf{Supplier}} & \multicolumn{7}{c}{\textbf{Pooled behavior}} \\
\cmidrule(lr){3-6}\cmidrule(lr){7-10}\cmidrule(lr){11-17}
 & & \multicolumn{2}{c}{\textbf{Utility}} & \multicolumn{2}{c}{\textbf{(Ex-Post) Regret}} & \multicolumn{2}{c}{\textbf{Utility}} & \multicolumn{2}{c}{\textbf{(Ex-Post) Regret}} & \multicolumn{4}{c}{\textbf{Compliance}} & \multicolumn{3}{c}{\textbf{Defection}} \\
\cmidrule(lr){3-4}\cmidrule(lr){5-6}\cmidrule(lr){7-8}\cmidrule(lr){9-10}\cmidrule(lr){11-14}\cmidrule(lr){15-17}
 & & \textbf{Utility}\,$\uparrow$ & \textbf{Compl.}\,$\uparrow$ & \textbf{Regret}\,$\to 0$ & \textbf{Compl.}\,$\to 0$ & \textbf{Utility}\,$\uparrow$ & \textbf{Compl.}\,$\uparrow$ & \textbf{Regret}\,$\to 0$ & \textbf{Compl.}\,$\to 0$ & \textbf{Rate} & \textbf{Cond.}\,$\uparrow$ & \textbf{Expl.}\,$\downarrow$ & \textbf{TFT}\,$\uparrow$ & \textbf{Rate} & \textbf{Unil.}\,$\downarrow$ & \textbf{Recip.}\,$\uparrow$ \\
\midrule
\multirow{3}{*}{\emph{Catering}} & Claude Opus 5 & 191.6 & 188.9 & \textbf{-0.8} & 2.0 & 23.6 & 14.9 & \textbf{-7.8} & 0.9 & 67.6 & 91.3 & 4.3 & 92.2 & 55.9 & 32.0 & 96.8 \\
  & Gemini 3.6 Flash & 196.3 & \textbf{190.8} & -5.5 & \textbf{0.0} & 6.9 & 5.8 & 8.9 & 10.0 & \textbf{70.1} & \textbf{94.3} & \textbf{0.0} & \textbf{96.0} & \textbf{48.9} & \textbf{24.2} & \textbf{100.0} \\
  & GPT 5.6 Sol & \textbf{200.5} & \textbf{190.8} & -9.7 & \textbf{0.0} & \textbf{24.2} & \textbf{15.9} & -8.4 & \textbf{-0.1} & 62.7 & 86.6 & \textbf{0.0} & 89.4 & 79.1 & 58.0 & \textbf{100.0} \\
\midrule
\multirow{3}{*}{\shortstack{\emph{Hotel}\\\emph{Cleaning}}} & Claude Opus 5 & 196.3 & 188.5 & -5.5 & 2.3 & 25.7 & 15.8 & -9.9 & \textbf{0.1} & 65.2 & 88.7 & 2.8 & 90.8 & 65.2 & 43.8 & 97.3 \\
  & Gemini 3.6 Flash & 194.5 & \textbf{190.8} & \textbf{-3.7} & \textbf{0.0} & 7.9 & 6.9 & \textbf{8.0} & 8.9 & \textbf{70.6} & \textbf{95.1} & \textbf{0.0} & \textbf{96.4} & \textbf{48.1} & \textbf{22.7} & \textbf{100.0} \\
  & GPT 5.6 Sol & \textbf{200.4} & 190.7 & -9.6 & 0.1 & \textbf{29.8} & \textbf{17.6} & -13.9 & -1.8 & 61.3 & 84.8 & \textbf{0.0} & 88.0 & 83.3 & 63.1 & \textbf{100.0} \\
\midrule
\multirow{3}{*}{\shortstack{\emph{AI}\\\emph{Hosting}}} & Claude Opus 5 & 197.3 & 190.1 & -6.5 & 0.7 & 32.1 & 18.4 & -16.3 & \textbf{-2.6} & 62.3 & 85.1 & 2.5 & 87.8 & 76.1 & 56.7 & 97.8 \\
  & Gemini 3.6 Flash & 194.8 & \textbf{190.8} & \textbf{-4.0} & \textbf{0.0} & 10.6 & 8.9 & \textbf{5.3} & 6.9 & \textbf{69.8} & \textbf{94.6} & \textbf{0.0} & \textbf{95.7} & \textbf{52.4} & \textbf{26.7} & \textbf{100.0} \\
  & GPT 5.6 Sol & \textbf{200.7} & \textbf{190.8} & -9.9 & \textbf{0.0} & \textbf{46.8} & \textbf{23.0} & -31.0 & -7.2 & 55.8 & 78.9 & \textbf{0.0} & 83.7 & 89.3 & 70.2 & \textbf{100.0} \\
\bottomrule
\end{tabular}
\end{center}

\endgroup

\FloatBarrier

\end{document}